\documentclass{article}

\PassOptionsToPackage{numbers,compress}{natbib} 
\usepackage[main,final]{neurips_2026}
\usepackage{wrapfig}
\usepackage[utf8]{inputenc}
\usepackage[T1]{fontenc}
\usepackage{hyperref}
\usepackage{url}
\usepackage{booktabs}
\usepackage{amsfonts} 
\usepackage{amsmath,amssymb,amsthm,mathtools,bm}
\usepackage{nicefrac}
\usepackage{microtype}
\usepackage[dvipsnames]{xcolor}
\usepackage{enumitem}
\usepackage{graphicx}
\usepackage{subcaption}
\usepackage{multirow}
\usepackage{tikz-cd}

\usepackage{placeins}

\usepackage{titletoc}
\hypersetup{colorlinks=true,linkcolor=blue,citecolor=blue,urlcolor=blue}

\newcommand*\circledblue[1]{%
\tikz[baseline=(char.base)]{
  \node[shape=circle, draw=NavyBlue!60, fill=NavyBlue!10, thick, inner sep=1pt] (char) {\scriptsize\textsf{#1}};
}}

\theoremstyle{plain}
\newtheorem{theorem}{Theorem}
\newtheorem{lemma}{Lemma}
\newtheorem{assumption}{Assumption}
\newtheorem{proposition}{Proposition}
\newtheorem{corollary}{Corollary}

\theoremstyle{definition}
\newtheorem{remark}{Remark}

\theoremstyle{plain}
\newcommand{\E}{\mathbb{E}}

\newcommand{\R}{\mathbb{R}}
\newcommand{\argmax}{\mathop{\mathrm{argmax}}}
\newcommand{\argmin}{\mathop{\mathrm{argmin}}}

\newcommand{\cA}{\mathcal{A}}
\newcommand{\cX}{\mathcal{X}}
\newcommand{\cS}{\mathcal{S}}
\newcommand{\cH}{\mathcal{H}}

\newcommand{\cV}{\mathcal{V}}
\newcommand{\Reg}{\mathrm{Reg}}

\usepackage{needspace}

\usepackage{xcolor}
\usepackage{tikz}

\definecolor{badgeblue}{RGB}{186,210,232}

\newcommand{\affmark}[1]{\textsuperscript{#1}}
\newcommand{\cofirstmark}{\textsuperscript{*}}
\newcommand{\correspondingmark}{\ensuremath{^{\dagger}}}

\newcommand{\affiliationnote}[2]{%
  \begingroup
  \renewcommand{\thefootnote}{\arabic{footnote}}%
  \footnotetext[#1]{#2}%
  \endgroup
}

\newcommand{\cofirstnote}{%
  \begingroup
  \renewcommand{\thefootnote}{\fnsymbol{footnote}}%
  \setcounter{footnote}{1}%
  \footnotetext{Co-first authors.}%
  \endgroup
}

\newcommand{\correspondingnote}[1]{%
  \begingroup
  \renewcommand{\thefootnote}{\fnsymbol{footnote}}%
  \setcounter{footnote}{2}%
  \footnotetext{Corresponding author: \texttt{#1}.}%
  \endgroup
}
\newcommand{\presentmark}{\ensuremath{^{\ddagger}}}

\title{Decision-Aware Proximal Bridge Learning for Optimal Treatment Selection}
\author{%
  Tomàs Garriga \affmark{1,2,3}\cofirstmark\correspondingmark
  \And
  Alejandro Almodóvar\affmark{4}\cofirstmark
  \And
  Axel Brando\affmark{2}\presentmark
  \And
  Gerard Sanz\affmark{1}
  \And
  Eduard Serrahima de Cambra\affmark{1}
  \And
  Juan Parras\affmark{4}
}

\begin{document}

\maketitle

\affiliationnote{1}{Novartis}
\affiliationnote{2}{Barcelona Supercomputing Center}
\affiliationnote{3}{Universitat Politècnica de Catalunya}
\affiliationnote{4}{Information Processing and Telecommunications Center, ETSI de Telecomunicación, Universidad Politécnica de Madrid}
\begingroup
\renewcommand{\thefootnote}{\fnsymbol{footnote}}
\setcounter{footnote}{3}
\footnotetext{Present affiliation: Trustworthy and Reflective Artificial Intelligence Lab (TRAILAB).}
\endgroup
\cofirstnote
\correspondingnote{tomas.garriga\_dicuzzo@novartis.com}
\setcounter{footnote}{0}

\begin{abstract}
Individualized treatment selection with continuous actions requires accurate causal response estimation in decision-relevant regions, rather than uniformly over the entire action space. Estimating a global causal response surface and then choosing the treatment that maximizes it can therefore be suboptimal, since standard estimation objectives allocate modeling effort according to the observed treatment distribution rather than the regions that determine the optimal decision. While decision-aware approaches have been studied in unconfounded settings, this problem remains underexplored in proximal causal inference, where proxy variables and bridge functions enable identification under suitable assumptions even in the presence of hidden confounding. Despite recent progress, proximal methods have primarily focused on treatment-effect and potential-outcome estimation rather than treatment selection and optimal decision-making. To bridge this gap, we introduce a policy-targeted weighted bridge loss that emphasizes decision-relevant treatment regions while retaining global stabilization. We prove a regret bound showing that the proposed weighted bridge loss controls
treatment-selection regret through a weighted ill-posedness constant. We instantiate the framework in decision-aware variants of several proximal bridge solvers, yielding practical algorithms that alternate between weighted bridge estimation, response-surface projection, policy update, and weight refinement. Empirically, we find that decision-aware weighting reduces regret across several bridge solvers, suggesting improved treatment selection in proximal settings.
\end{abstract}

\section{Introduction}
\label{sec:introduction}

Individualized treatment selection from observational data is a central goal in machine learning and causal inference. In applications such as medicine, marketing, education, and public policy, the aim is not only to estimate treatment effects, but to select an intervention with high utility for each individual. This distinction is especially important for continuous treatments, where the learner must recover enough of the response surface to recommend an optimal dose. However, most counterfactual prediction and causal-effect estimation methods optimize global accuracy over the treatment space. As emphasized by \citet{zou2022counterfactual}, such objectives can be poorly aligned with decision quality: a model may fit the response surface well on average while being inaccurate near the treatment optimum.

The challenge is amplified in observational studies with unmeasured confounding. Standard individualized-treatment methods typically assume exchangeability after adjusting for observed covariates, but treatment assignment may depend on latent severity, physician judgment, patient frailty, unrecorded preferences, or institutional constraints. When such latent factors affect both treatment and outcome, ignoring them can bias both response-surface estimates and the resulting decisions.

\begin{figure}[t]
\begin{center}
\centerline{\includegraphics[width=\textwidth]{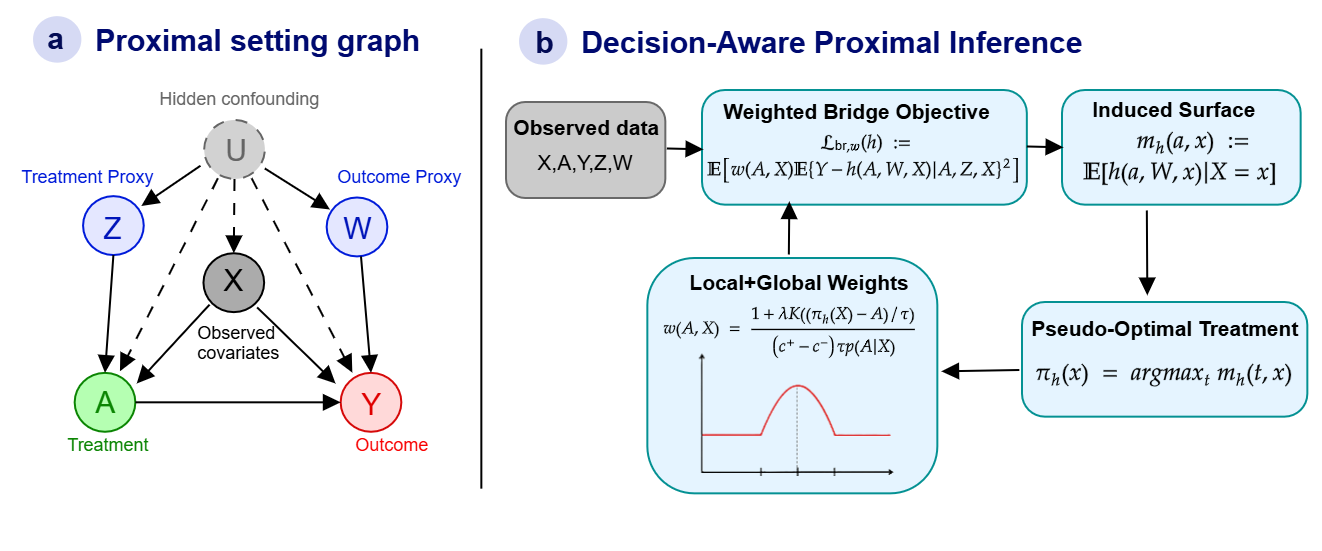}}
\end{center}
\caption{(a) Proximal causal graph with proxy variables. 
(b) Decision-aware bridge-learning pipeline with local--global weighting around the pseudo-optimal treatment.}
\label{main_figure}
\end{figure}

Proximal causal inference provides a principled approach to this setting by using proxy variables for hidden confounders and identifying causal effects through bridge functions satisfying conditional moment restrictions \citep{miao2018identifying,tchetgen2024proximal}. Recent work has developed kernel, moment-restriction, neural, doubly robust, and treatment-rule methods for proximal learning \citep{mastouri2021proximal,kompa2022deep,wu2024doubly,qi2024proximal,shen2023optimal}. Yet most of these methods remain estimation-centric: they aim to recover bridge functions or dose-response surfaces accurately in a global sense.

For treatment selection, global accuracy is not the right target. Errors in clearly suboptimal treatment regions may have little effect on policy value, whereas small errors near the optimum can change the selected dose. We therefore study \emph{policy-targeted proximal bridge learning}. Our key idea is that proximal bridge estimation is an inverse problem, and the norm used to solve it should reflect the downstream decision. We construct a weighted bridge objective that emphasizes treatment regions near a pseudo-optimal dose estimated from a pilot bridge while retaining global stabilization.

Our contributions are:\footnote{Code is available at \url{https://github.com/aalmodovares/DecisionAwareProximalLearning}.} \textbf{(1)} we formulate individualized continuous-treatment selection under hidden confounding as a regret-targeted proximal bridge-learning problem, extending Outcome Oriented Counterfactual Prediction \citep{zou2022counterfactual} to proximal inference through a policy-targeted weighted bridge objective; \textbf{(2)} we show that, under a weighted ill-posedness condition, a policy-targeted weighted bridge loss controls a weighted causal-surface loss and hence a treatment-selection regret surrogate; \textbf{(3)} we develop an iterative algorithm alternating between bridge estimation, response-surface projection, policy update, and weight refinement, instantiate it for several proximal bridge solvers, and demonstrate its effectiveness on synthetic and semi-synthetic benchmarks with hidden confounding and proxy variables.

\section{Related Work}
\label{sec:related-work}

\textbf{Decision-oriented learning and continuous treatments.}
A large literature studies treatment-effect estimation and policy learning from observational data, including generalized propensity-score methods for continuous treatments \citep{hirano2004propensity} and modern approaches to dose-response estimation and continuous-action policy optimization \citep{kallus2018policy,singh2024kernel}. A complementary line of work emphasizes that accurate causal estimation need not imply accurate treatment selection. In particular, \citet{NEURIPS2025_c2481541} propose a decision-aware learning objective for conditional average treatment effect that prioritizes errors near the decision boundary in binary treatment settings.
Individualized treatment-rule methods such as outcome-weighted learning directly target decision quality \citep{zhao2012estimating}, while outcome-oriented counterfactual prediction shows that errors near decision-relevant treatments can matter more than global surface error \citep{zou2022counterfactual}.
These ideas are especially important for continuous treatments, where only a small region of the action space may determine the recommended dose.
Most such methods, however, assume ignorability.

\textbf{Proximal causal inference and bridge learning.}
Proximal causal inference uses negative-control or proxy variables to identify causal quantities under unmeasured confounding through outcome and treatment bridge functions satisfying conditional moment restrictions \citep{miao2018identifying,tchetgen2024proximal,cui2024semiparametric}. Recent machine-learning methods estimate these bridges using kernel two-stage regression and maximum moment restrictions, including kernel proxy variables (KPV) and proxy maximum moment restriction (PMMR) \citep{mastouri2021proximal}; neural or adversarial conditional-moment objectives, including neural maximum moment restriction (NMMR) \citep{kompa2022deep,kallus2021causal,ghassami2022minimax}; deep proxy representations, including deep feature proxy variables (DFPV) \citep{xu2021deep}; and kernel methods for causal response functions \citep{singh2024kernel}. Related work develops continuous-treatment, doubly robust, and density-ratio-free proximal estimators \citep{wu2024doubly,bozkurt2025densityratio,bozkurt2025drfree}, as well as implicit bridge learning through deconfounding generative models \citep{almodovar2025decaflow} and proximal treatment-rule learning under hidden confounding \citep{qi2024proximal,shen2023optimal}. These methods advance estimation and policy learning in proximal models, but their bridge-learning objectives remain largely estimation-centric rather than decision oriented. See Appendix \ref{app:extended-related-work} for an extended discussion.

\textbf{Research gap.}
Existing outcome-oriented methods target decision-relevant counterfactual prediction but usually assume observed-confounder identification. Existing proximal methods handle latent confounding but usually optimize global bridge or effect estimation. Our work targets the intersection: decision-aware bridge learning for continuous treatment policies under unmeasured confounding.

\section{Background and Problem Setting}

\subsection{Proximal Inference}

We observe i.i.d. samples $O_i=(Y_i,A_i,Z_i,W_i,X_i)$, for $i=1,\ldots,n$,
where $A\in\cA=[c_-,c_+]$ is a continuous treatment,
$Y$ is the outcome, and $X$ denotes observed pre-treatment covariates.
The variable $U$ denotes latent confounders that may affect both treatment
assignment and the outcome. Proximal inference uses two additional observed
variables as proxies for $U$: a \emph{treatment-inducing proxy} $Z$ and an
\emph{outcome-inducing proxy} $W$. Informally, $Z$ carries information
about latent factors that affect treatment assignment, while $W$ carries
information about latent factors that affect the outcome. The left-hand side of Figure \ref{main_figure} shows the causal graph of the proximal setting. 

The proximal strategy is to replace adjustment for the unobserved $U$ by
an observable bridge equation. The bridge is not itself the causal response
surface. Rather, it is a function of the observed treatment, outcome proxy,
and covariates whose conditional expectation recovers the causal response
after averaging over the distribution of $W$. This distinction is central
for our objective: policy decisions will be made through the causal response
surface induced by a learned bridge.

We write $Y(a)$ for the potential outcome under treatment level $a$.
All expectations are with respect to the observed-data distribution unless
potential outcomes are explicitly shown. We impose the following standard
proximal inference assumptions.

\begin{assumption}[Consistency and latent exchangeability]
\label{ass:consistency-exchangeability}
For every $a\in\cA$, $Y(a)$ is well defined, $Y=Y(A)$ almost surely, and
$Y(a)\perp\!\!\!\perp A\mid U,X$.
\end{assumption}

\begin{assumption}[Negative-control proxies]
\label{ass:negative-control}
For every $a\in\cA$,
$Y(a)\perp\!\!\!\perp Z\mid A,U,X$ and
$W\perp\!\!\!\perp (A,Z)\mid U,X$.
\end{assumption}

\begin{assumption}[Completeness]
\label{ass:completeness}
For any square-integrable $\ell$, if
$\E\{\ell(U)\mid A=a,Z=z,X=x\}=0$ for all $(a,z,x)$ in the support of
$(A,Z,X)$, then $\ell(U)=0$ almost surely. For any square-integrable $g$,
if $\E\{g(Z)\mid A=a,W=w,X=x\}=0$ for all $(a,w,x)$ in the support of
$(A,W,X)$, then $g(Z)=0$ almost surely.
\end{assumption}

\begin{assumption}[Positivity]
\label{ass:positivity}
For almost every $(u,x)$, the conditional law of $A$ given $(U,X)=(u,x)$
has support containing $\cA$ and admits a density $p(a\mid u,x)>0$ on
$\cA$. Moreover, the observed density $p(a\mid x)$ is bounded away from
zero on $\cA$.
\end{assumption}

We further assume the existence of a square-integrable outcome bridge
$h_0:\cA\times\mathcal W\times\cX\to\R$ satisfying the outcome bridge restriction
    $\E\{Y-h_0(A,W,X)\mid A,Z,X\}=0$.
Sufficient regularity and solvability conditions are given in
Appendix~\ref{app:bridge-existence}. Appendix \ref{app:assumptions} further discusses the assumptions. 

\begin{proposition}[Proximal identification]
\label{prop:proximal-identification}
Under the bridge-existence conditions in
Appendix~\ref{app:bridge-existence}, for every $a\in\cA$ and almost every
$x$, $m_0(a,x)
    :=
    \E\{Y(a)\mid X=x\}
    =
    \E\{h_0(a,W,x)\mid X=x\}$.
\end{proposition}

\subsection{Causal response surfaces, policy value and regret}
\label{subsec:regret}

Proposition \ref{prop:proximal-identification} identifies the causal response surface $m_0(a,x)$ for the true bridge $h_0$. For any candidate bridge $h\in\cH$, define its induced response surface
$m_h(a,x)
    :=
    \E\{h(a,W,x)\mid X=x\}$.
A treatment policy is a measurable map $\pi:\cX\to\cA$. Its value is $V(\pi)
    :=
    \E\{Y(\pi(X))\}
    =
    \E\{m_0(\pi(X),X)\}$. For a candidate bridge $h$, define the induced policy and oracle policy by
\begin{equation}
    \pi_h(x) \in \argmax_{t\in\cA} m_h(t,x),
    \qquad
    a^\star(x) \in \argmax_{t\in\cA} m_0(t,x).
    \label{eq:induced-and-oracle-policy}
\end{equation}
We assume throughout that larger outcomes are preferred and, for notational simplicity, that maximizers are unique; fixed tie-breaking would give the same analysis. The regret of the bridge-induced policy is
\begin{equation}
    \Reg(h)
    :=
    \E\!\left[
        m_0(a^\star(X),X)-m_0(\pi_h(X),X)
    \right]
    =
    V(a^\star)-V(\pi_h).
    \label{eq:regret}
\end{equation}
This regret is decision-specific. It does not require the entire surface $m_h$ to be uniformly accurate; it requires sufficient accuracy at treatment levels that can change the optimizer.

\subsection{Why is ordinary proximal bridge learning not decision-aligned?}
\label{subsec:global-misalignment}

Standard bridge estimators minimize a global violation of the conditional moment restriction
or a kernelized maximum-moment analogue. This is natural for estimating a bridge or a full dose-response curve. It is not necessarily aligned with treatment selection. For example, a bridge may fit the moment equation well in treatment regions that are frequently observed but clearly suboptimal, while remaining inaccurate near the maximizer of $m_0(\cdot,x)$. Such localized errors can change $\pi_h(x)$ and produce large regret even when the global moment loss is small.

Our goal is therefore not to replace proximal identification. We keep the same bridge equation and the same proxy assumptions. The difference is the norm in which the bridge equation is solved. Section~\ref{sec:targeted-risk} constructs a localized-plus-global proximal risk that emphasizes treatment regions near the policy induced by a pilot bridge, while retaining a global component to stabilize the inverse problem and protect against missing the oracle region. 

Our motivation is related to outcome-oriented counterfactual prediction, which shows that global counterfactual prediction error can be poorly aligned with treatment selection \citep{zou2022counterfactual}. The challenge here is different: under unmeasured confounding, the causal surface is available only through a proximal bridge equation. Decision-targeting must therefore be imposed at the level of an inverse conditional-moment problem rather than at the level of a supervised counterfactual prediction loss.

\section{Policy-Targeted Proximal Risk}
\label{sec:targeted-risk}

This section develops the main population object of the paper: a weighted proximal bridge risk whose
geometry is aligned with optimal treatment selection and that controls a regret surrogate. All the proofs are in Appendix \ref{app:proofs}.

\subsection{A Surface Surrogate on Regret}
\label{subsec:surface-regret}

Define the squared surface error $G_h(t,X):=\bigl(m_0(t,X)-m_h(t,X)\bigr)^2$.

\begin{lemma}
\label{lemma:regret-decomp-surface}
The regret satisfies
    $\Reg(h)
    \le
    \sqrt{\E\!\left[G_h\!\left(\pi_h(X),X\right)\right]}
    +
    \sqrt{\E\!\left[G_h\!\left(a^\star(X),X\right)\right]}.
    \label{eq:regret-decomp}$
\end{lemma}

Lemma~\ref{lemma:regret-decomp-surface} makes the target explicit. To control regret, it is sufficient to control the surface error near the learned dose $\pi_h(X)$ and the oracle dose $a^\star(X)$. The first term can be targeted using the current or pilot policy. The second term is not directly observable and motivates retaining a global component.

Assume overlap, so that $p(t\mid x)>0$ for all $(x,t)\in\cX\times\cA$. Let $K$ be a nonnegative symmetric kernel supported on $[-1,1]$ with $\int K(u)\,du=1$, $\int uK(u)\,du=0$, and finite second moment $\mu_2(K)=\int u^2K(u)\,du<\infty$. Define $K_\tau(u)=\tau^{-1}K(u/\tau)$.

The localized surface loss around the pseudo-optimal treatment is
\begin{equation}
    A_\tau(h)
    :=
    \E_{X,A}\!\left[
    \frac{K\!\left((\pi_h(X)-A)/\tau\right)}{\tau\,p(A\mid X)}
    \,G_h(A,X)
    \right].
    \label{eq:Atau-surface}
\end{equation}
This approximates $G_h(\pi_h(X),X)$,  the surface error at $\pi_h(X)$; the proof and a precise kernel-bias statement is given in Appendix~\ref{app:proofs}. 
The interiority condition used for this approximation avoids kernel truncation at the treatment boundaries; a support-normalized boundary extension is described in Appendix~\ref{app:boundary-adaptive-kernel}.

Define the global surface loss
\begin{equation}
    B(h)
    :=
    \E_X\!\left[
    \frac{1}{c_+-c_-}\int_{c_-}^{c_+} G_h(t,X)\,dt
    \right]
    =
    \E_{X,A}\!\left[
    \frac{G_h(A,X)}{(c_+-c_-)p(A\mid X)}
    \right].
    \label{eq:B-surface}
\end{equation}

In appendix \ref{app:proofs} we show that $B(h)$ controls the oracle dose term $G_h(a^\star(X),X)$  up to a Lipschitz remainder. 
The bound is
deliberately coarse; its role is not to sharply approximate the unknown oracle region, but to retain global stabilization.

The combined surface surrogate is $\gamma A_\tau(h)+B(h)$, where $\gamma>0$ controls the degree of localization. Larger $\gamma$ prioritizes errors near the pseudo-optimal treatment, while smaller $\gamma$ emphasizes the global response surface.

\begin{proposition}[Weighted surface surrogate]
\label{prop:weighted-surrogate-surface}
Let $c_\gamma:=\max\{1,\gamma^{-1}\}$. Under the local kernel-approximation and global Lipschitz conditions stated in Propositions~\ref{prop:kernel-bias-surface} and~\ref{prop:global-surface-bound} (Appendix \ref{app:proofs}),
\begin{equation}
    \Reg(h)
    \le
    \sqrt{2c_\gamma}\,
    \sqrt{\gamma A_\tau(h)+B(h)}
    +
    \sqrt{2\!\left(C_h\tau^2+\frac{L}{2}(c_+-c_-)\right)}.
    \label{eq:weighted-surrogate-bound}
\end{equation}
\end{proposition}

Using the definitions of \(A_\tau(h)\) and \(B(h)\), the combined surface
objective \(\gamma A_\tau(h)+B(h)\) can be written as the policy-targeted
surface risk:
\begin{equation}
    \mathcal L^{\mathrm{surf}}_{\tau,\lambda}(h)
    :=
    \E_{X,A}\!\left[
    \frac{1+\lambda K\!\left((\pi_h(X)-A)/\tau\right)}
    {(c_+-c_-)p(A\mid X)}
    \,\bigl(m_0(A,X)-m_h(A,X)\bigr)^2
    \right],
    \label{eq:eq10-analog-surface}
\end{equation}
where $\lambda=(c_+-c_-)\gamma/\tau$.

\subsection{From weighted surface risk to weighted bridge risk}
\label{subsec:bridge-risk}

The surface risk in~\eqref{eq:eq10-analog-surface} is not directly observable because it depends on $m_0$. Proximal inference instead learns through the bridge moment equation. This raises two key questions:
\circledblue{1}\ can the policy-targeted weighting be justified at the bridge level, so that the weighted bridge loss controls the decision-relevant surface risk and hence regret? \circledblue{2}\ does introducing such weights change the population bridge target?

For any strictly positive measurable weight $\omega:\cA\times\cX\to(0,\infty)$, define
\begin{align}
    \mathcal L_{\mathrm{surf},\omega}(h)
    &:=
    \E\!\left[
    \omega(A,X)\bigl(m_h(A,X)-m_0(A,X)\bigr)^2
    \right],
    \label{eq:weighted-surface-risk} \\
    \mathcal L_{\mathrm{br},\omega}(h)
    &:=
    \E\!\left[
    \omega(A,X)\,
    \E\{Y-h(A,W,X)\mid A,Z,X\}^2
    \right].
    \label{eq:weighted-bridge-risk}
\end{align}

Let $\Delta_h(a,w,x):=h(a,w,x)-h_0(a,w,x)$ and define the outcome-bridge operator
\begin{equation}
    (T\Delta)(a,z,x):=
    \E\{\Delta(a,W,x)\mid A=a,Z=z,X=x\}.
    \label{eq:bridge-operator}
\end{equation}

\begin{proposition}[Weighted bridge loss controls weighted surface loss]
\label{prop:weighted_bridge_to_surface}
Assume that $h_0$ satisfies the outcome bridge restriction. Let $\nu:=P_{A,W,X}$ and $\mu:=P_{W\mid X}\otimes P_{A\mid X}\otimes P_X$. Suppose there exists a finite weighted ill-posedness constant
\begin{equation}
    \chi_{\omega}
    :=
    \sup_{g\in\cH}
    \frac{\|\Delta_g\|_{L_2(\omega\,d\nu)}}
         {\|T\Delta_g\|_{L_2(\omega\,dP_{A,Z,X})}}
    < \infty,
    \label{eq:weighted-illposedness}
\end{equation}
and suppose the density ratio $d\mu/d\nu$ is essentially bounded by $C_\rho$. Then, for every $h\in\cH$,
\begin{equation}
    \mathcal L_{\mathrm{surf},\omega}(h)
    \le
    C_\rho\chi_{\omega}^2\,
    \mathcal L_{\mathrm{br},\omega}(h).
    \label{eq:bridge-to-surface}
\end{equation}
\end{proposition}

Proof in Appendix \ref{app:bridge-to-surface-proof}. The condition $\chi_\omega<\infty$ is the weighted analogue of the usual
restricted ill-posedness condition for proximal inverse problems, used in works like \cite{kallus2021causal, dikkala2020minimax, chen2012nonparametric}. It is also
implied by this standard unweighted condition under bounded weights; see Appendix~\ref{app:weighted-illposedness-sufficient}.
Appendix~\ref{app:weighted-illposedness-sufficient} also provides a finite-dimensional characterization and a simple example illustrating when the weighted ill-posedness constant is finite.

The result shows that weighting the bridge loss changes the population norm in which bridge moment violations are
controlled, and this norm directly controls the weighted surface target that
appears in the regret surrogate.

For each $h\in\cH$, define the self-weighted policy-targeted bridge weight
\begin{equation}
    \omega_h(a,x)
    :=
    \frac{1+\lambda K((\pi_h(x)-a)/\tau)}{(c_+-c_-)p(a\mid x)},
    \qquad
    \lambda:=\frac{(c_+-c_-)\gamma}{\tau}.
    \label{eq:self-weight}
\end{equation}

\begin{theorem}[Policy-targeted bridge loss controls regret]
\label{thm:self_weighted_bridge_to_regret}
Under the conditions of Propositions~\ref{prop:weighted-surrogate-surface} and~\ref{prop:weighted_bridge_to_surface}, for every $h\in\cH$,
\begin{equation}
    \Reg(h)
    \le
    \sqrt{2c_\gamma C_\rho}\,\chi_{\omega_h}\,
    \sqrt{\mathcal L_{\mathrm{br},\omega_h}(h)}
    +
    \sqrt{2\left(C_h\tau^2+\frac{L}{2}(c_+-c_-)\right)}.
    \label{eq:main-regret-bound}
\end{equation}
\end{theorem}

\begin{remark}
Theorem~\ref{thm:self_weighted_bridge_to_regret} highlights the main conceptual distinction from outcome-oriented reweighting in unconfounded settings. Under unmeasured confounding, the learner cannot directly regress on counterfactual outcomes or on the causal response surface. It must instead solve a proximal inverse problem through the bridge moment equation. The theorem shows that solving this inverse problem in a policy-targeted weighted norm controls a regret surrogate, with sensitivity to the localized inverse problem captured by the weighted ill-posedness constant $\chi_{\omega_h}$.
\end{remark}

In practice, we approximate this self-weighted objective by a lagged fixed-point procedure: a bridge fitted under the current weights induces a response surface and pseudo-optimal policy, which are then used to define the weights for the next bridge fit. Section \ref{sec:algorithms} describes this practical approximation.

The preceding results explain why the weighted norm matters for regret control. It remains to check that this weighting does not change the population bridge target in the well-specified problem. The next proposition shows that, under strictly positive weights, it does not.

\begin{proposition}[Population target preservation]
\label{prop:population-target-preservation}
Suppose $h_0\in\cH$ satisfies the outcome bridge restriction. Let $\omega:\cA\times\cX\to(0,\infty)$ be measurable and satisfy $\omega(A,X)>0$ almost surely. Then $h_0$ is a population minimizer of $\mathcal L_{\mathrm{br},\omega}$. Moreover, for any $h\in\cH$,
\begin{equation}
    \mathcal L_{\mathrm{br},\omega}(h)=0
    \quad\Longleftrightarrow\quad
    \E\{Y-h(A,W,X)\mid A,Z,X\}=0
    \quad\text{a.s.}
    \label{eq:weighted-zero-risk-equivalence}
\end{equation}
Consequently, the weighted and unweighted population bridge risks have the same population minimizer. \end{proposition}

Proof in Appendix \ref{app:proof-population-target-preservation}. For the policy-targeted weight in~\eqref{eq:self-weight}, the global component in the numerator, together with positivity of $p(a\mid x)$, ensures that $\omega_h(A,X)>0$ almost surely, so the preservation result applies.

Thus, the role of the weight is not to define a different population bridge target, but to define the geometry in which deviations from the bridge equation are penalized. This distinction matters in \textbf{finite samples}, under \textbf{regularization}, and under \textbf{misspecification}, where not all bridge-moment violations can be driven to zero. 

\section{Algorithms}
\label{sec:algorithms}

Let $\{O_i=(Y_i,A_i,Z_i,W_i,X_i)\}_{i=1}^n$ denote the observed sample, with
$A_i\in\mathcal A=[c_-,c_+]$. Write
$S_i=(A_i,W_i,X_i)$, $V_i=(A_i,Z_i,X_i)$ and  $r_i(h)=Y_i-h(S_i)$. The population results above suggest solving the proximal bridge equation in a
policy-targeted weighted norm. Algorithmically, this leads to an iterative
procedure that alternates between
\begin{equation}
    \widehat h^{(s)}
    \;\longrightarrow\;
    \widehat m^{(s)}
    \;\longrightarrow\;
    \widehat\pi^{(s)}
    \;\longrightarrow\;
    \widehat\omega^{(s+1)} .
\end{equation}
Thus the treatment rule is not read directly from the fitted bridge. Each
iteration first estimates an outcome bridge, then projects the bridge into an
induced causal response surface, then optimizes the treatment over this surface,
and finally uses the resulting pseudo-optimal treatment to construct the next
policy-targeted bridge weights.

\textbf{Policy-targeted bridge fitting.}
Let $\widehat p(a\mid x)$ be an estimate of the generalized propensity score (in our experiments we estimate this conditional density using a normalizing-flow
model).
We initialize with the global inverse-propensity weight
    $\widehat\omega_i^{(0)}
    =
    \frac{1}{(c_+-c_-)\widehat p(A_i\mid X_i)} .
$
This initialization makes the first bridge fit globally treatment-uniform:
rather than emphasizing treatment values in proportion to how often they are
observed, the inverse-density factor targets the bridge loss averaged uniformly
over $\mathcal A$. Later iterations add the policy-localized component.

For iteration $s=0,\ldots,S$, given weights
$\widehat\omega^{(s)}
=(\widehat\omega_1^{(s)},\ldots,\widehat\omega_n^{(s)})$,
we estimate a bridge by solving $\widehat h^{(s)}
    \in
    \arg\min_{h\in\mathcal H}
    \left\{
        \widehat{\mathcal R}_{n,\widehat\omega^{(s)}}(h)
        +
        \eta_h\Omega_h(h)
    \right\}$,
where $\widehat{\mathcal R}_{n,\widehat\omega}(h)$
denotes the solver-specific empirical fitting criterion
and $\Omega_h$ is the corresponding regularizer.
Different proximal solvers enforce the weighted bridge
moment restriction through different empirical objectives,
which need not be direct empirical estimates of
$\mathcal L_{\mathrm{br},\omega}$ in~\eqref{eq:weighted-bridge-risk}.
Their statistical properties and relationship to the ideal
bridge risk are discussed in Appendix~\ref{app:adaptive-weight-statistics}.

\textbf{Projection to a response surface.}
The fitted bridge is then projected into an induced causal response surface,
$m_{\widehat h^{(s)}}(a,x)
=
\mathbb E\{\widehat h^{(s)}(a,W,x)\mid X=x\}$. In practice, we approximate this
projection by evaluating the bridge on a treatment grid
$\mathcal G=\{a_1,\ldots,a_M\}\subset\mathcal A$ and forming bridge
pseudo-outcomes over this grid. These pseudo-outcomes are formed by
cross-fitting: for each fold $I_k$ in a partition of $\{1,\ldots,n\}$, the bridge $\widehat h_{-k}^{(s)}$ is
trained on the observations outside $I_k$ and evaluated on the held-out
observations, giving
$\widetilde H_{im}^{(s)}=\widehat h_{-k}^{(s)}(a_m,W_i,X_i)$ for
$i\in I_k$ and $m=1,\ldots,M$. We then pool these out-of-fold pseudo-outcomes
across folds and regress them on $(a_m,X_i)$ to obtain a single surface
estimator $\widehat m^{(s)}(a,x)$. The observation-specific pseudo-optimal
treatment used for weighting is
$\widehat\pi_i^{(s)}
\in
\arg\max_{a\in\mathcal G}
\widehat m^{(s)}(a,X_i)$.

\textbf{Weight update.}
For $s<S$, we update the empirical weights by evaluating
\eqref{eq:self-weight} at $(a,x)=(A_i,X_i)$, replacing
$p$ by $\widehat p$, and using the lagged pseudo-policy
$\widehat\pi_i^{(s)}$. Thus the new weights combine a global inverse-density
component with a localized component around the current pseudo-optimal
treatment. The global component prevents the bridge fit from collapsing
entirely onto the current pseudo-optimal region and stabilizes the proximal
inverse problem.

\textbf{Statistical guarantees.} 
Appendix~\ref{app:fixed-weight-statistics}
makes the statistical implications of Proposition \ref{prop:weighted_bridge_to_surface}
explicit, showing that, under a fixed weighting scheme
and a bridge-risk oracle inequality, the induced
causal surface risk is controlled by the bridge
approximation and estimation errors. Appendix~\ref{app:adaptive-weight-statistics}
establishes a finite-round oracle inequality for
the solver-specific population criteria
$\mathcal R_\omega$ under adaptive weighting
and a uniform deviation condition.
We further provide sufficient calibration
conditions under which these guarantees
yield control of the ideal bridge risk and,
consequently, the causal surface risk.

\subsection{Solver adapters}
\label{app:solver_adapter_table}

The policy-targeted weighting scheme is a general wrapper for proximal bridge
learning. In this work, we instantiate it in four representative proximal bridge
solvers: \textbf{PMMR}, \textbf{NMMR}, \textbf{KPV}, and \textbf{DFPV}. The construction does not change the target outcome bridge: it still targets an
outcome bridge $h_0$ satisfying the bridge equation. Instead, it changes the
empirical norm in which violations of that equation are penalized.

\begin{table}[h]
\centering
\scriptsize
\setlength{\tabcolsep}{4pt}
\renewcommand{\arraystretch}{0.9}
\caption{Solver-specific adapters for the policy-targeted bridge objective.}
\label{tab:solver_weight_adapters_}
\begin{tabular}{lll}
\toprule
Solver & Baseline objective component & Policy-targeted adapter \\
\midrule
PMMR
&
$r^\top K_V r$
&
$r^\top D_{\widehat\omega}^{1/2}K_VD_{\widehat\omega}^{1/2}r$
\\
NMMR
&
$r(\theta)^\top K_V r(\theta)$
&
$r(\theta)^\top D_{\widehat\omega}^{1/2}K_VD_{\widehat\omega}^{1/2}r(\theta)$
\\
KPV (DFPV)
&
second-stage kernel ridge
&
weight second-stage bridge loss by $\widehat\omega_i$
\\
\bottomrule
\end{tabular}
\end{table}

The same weights $\widehat\omega_i$ can be used across different
proximal bridge solvers. The key solver-specific step is to insert the weights at the level where each
method represents the bridge moment. In PMMR and NMMR, the bridge moment is
represented directly by an MMR quadratic in residual pairs. In KPV and DFPV, the
bridge equation is enforced through a second-stage bridge regression after a
first-stage nuisance regression. Table~\ref{tab:solver_weight_adapters_}
summarizes the resulting adapters. $K_V$ is the moment-kernel Gram matrix on $V=(A,Z,X)$ and
$D_{\widehat\omega}=\operatorname{diag}(\widehat\omega_1,\ldots,\widehat\omega_n)$.
For PMMR and NMMR, the adapter is equivalent to replacing each residual by
$\sqrt{\widehat\omega_i}r_i$ inside the MMR loss. For KPV and DFPV, the weights
enter as ordinary loss weights in the second-stage bridge regression; the
first-stage nuisance regression is left unchanged. In the experiments we report two NMMR variants: NMMR-V uses the full V-statistic MMR objective, while NMMR-U uses the corresponding diagonal-deleted U-statistic objective. Appendix~\ref{app:proximal-solvers} presents the general implementation template, Appendices~\ref{app:pmmr} and~\ref{app:kpv} detail the MMR-based and two-stage solver variants, respectively, and Appendix~\ref{app:implementation_details} provides additional implementation details.

\section{Evaluation}
\label{sec:experiments}

\begin{wrapfigure}{r}{0.42\textwidth}
    \centering
    \includegraphics[width=\linewidth]{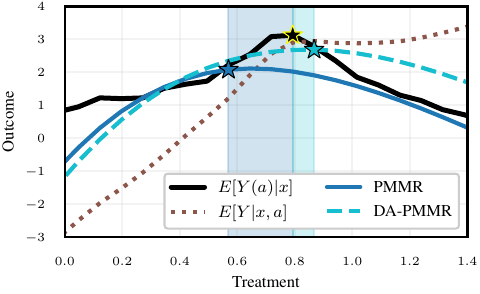}
    \caption{Single-individual example from the synthetic benchmark. DA-PMMR improves the response estimate near the oracle optimum, yielding a treatment recommendation closer to the optimal dose even though it does not uniformly improve the full response surface.}
    \label{fig:motivation}
\end{wrapfigure}

\textbf{Datasets and baselines.}
We evaluate on two proximal continuous-treatment benchmarks with hidden
confounding and proxy variables. Each observation is
$O=(X,Z,W,A,Y)$, where $X$ are observed covariates, $Z$ and $W$ are
treatment- and outcome-inducing proxies, $A$ is a continuous treatment, and
$Y$ is the outcome. Both benchmarks have non-monotone causal response surfaces
with interior, covariate-dependent optima, making regret sensitive to errors
near the optimal dose. The first benchmark is a fully synthetic proximal DGP
adapted from \citet{wu2024doubly}; the second is a semi-synthetic TCGA
benchmark using gene-expression features from The Cancer Genome Atlas
\citep{weinstein2013cancer}. See Appendix \ref{app:dataset-details}. We compare PMMR, NMMR, KPV, and DFPV with their
decision-aware (DA) variants. We additionally include
DR-KPV and DR-PMMR, based on the doubly robust proximal
estimators of \citet{bozkurt2025drfree}, adapted to
conditional dose-response estimation. We also include
non-proximal S-learner and weighted S-learner baselines
to assess the effect of ignoring proxy-based adjustment. Base learners are tuned by held-out factual RMSE, and the selected
hyperparameters are reused for their DA variants. The weighting parameters
$(\tau,\lambda,n_{\mathrm{rounds}})$ are calibrated once on a synthetic
validation setting and then fixed across methods and datasets; this prevents the
decision-aware variants from receiving model-specific tuning based on oracle
regret, which would not be available in observational applications. See
Appendix~\ref{app:hparam-selection} for details.

\textbf{Results.}
Figure~\ref{fig:regret_results} reports policy regret on the synthetic and
semi-synthetic benchmarks. Across both datasets, \textit{decision-aware weighting
generally reduces regret for proximal bridge learners}. The gains are especially
visible for the solvers whose baseline response surfaces place substantial error
near the treatment optimum, such as NMMR-U, KPV, and DFPV. In the synthetic
benchmark, the best-performing method is DA-KPV, while in the semi-synthetic benchmark the lowest regret is obtained by
DA-DFPV; Wilcoxon tests show statistical differences between DA and baseline versions for most models. Moreover, the best-performing DA variants achieve lower
policy regret than both doubly robust proximal baselines
across the two benchmarks. Also, we see that the non-proximal S-learner baselines are consistently worse.
This supports the main hypothesis that, under hidden confounding, proxy-based
bridge learning is necessary, and that reweighting the bridge loss toward
decision-relevant treatment regions can further improve treatment selection.

Figure~\ref{fig:counterfactual_results} reports counterfactual-estimation RMSE over the full treatment region sampled uniformly. The results are mixed: overall, we do not observe substantial differences between the decision-aware (DA) variants and their baseline counterparts. This is expected, since the proposed objective is not designed to reduce global counterfactual error uniformly over the treatment space, but rather to improve the learned response surface near treatment values that affect the induced policy. The individual-level example in Figure~\ref{fig:motivation} illustrates this distinction. Importantly, the DA variants \textit{do not appear to degrade general counterfactual estimation}. The doubly robust baselines also exhibit counterfactual
RMSE comparable to several of the other proximal estimators. In Appendix~\ref{app:additional-results}, we provide the corresponding counterfactual-RMSE tables, where the behavior of NMMR-U can be more clearly seen. We also report factual prediction metrics, which lead to a similar conclusion.

\begin{figure}[t]
\begin{subfigure}
{0.49\linewidth}
        \centering
    \includegraphics[width=\linewidth]{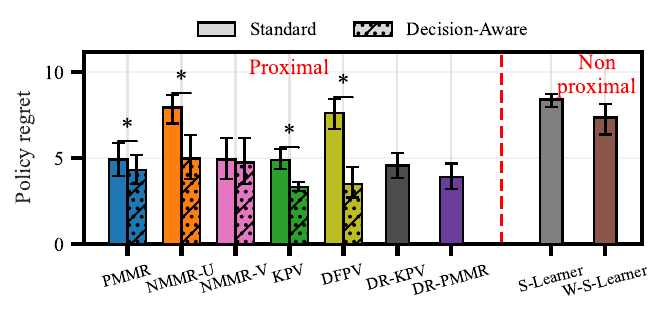}
    \caption{Synthetic experiment}
    \label{fig:semi-synthetic_regret}
\end{subfigure}
\begin{subfigure}{0.49\linewidth}
        \centering
    \includegraphics[width=\linewidth]{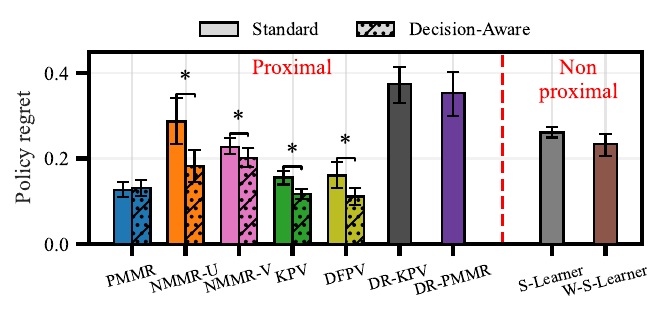}
    \caption{semi-synthetic experiment}
    \label{fig:semi-synthetic_regret}
\end{subfigure}
        \caption{Regret metrics in the synthetic and semi-synthetic datasets. Mean and 95$\%$ confidence intervals over 10 seeds are shown. * denotes statistical difference between baselines and DA models (paired Wilcoxon signed-rank test, \(p < 0.05\)).}
    \label{fig:regret_results}
\end{figure}

\begin{figure}[t]
\begin{subfigure}{0.49\linewidth}
        \centering
    \includegraphics[width=\linewidth]{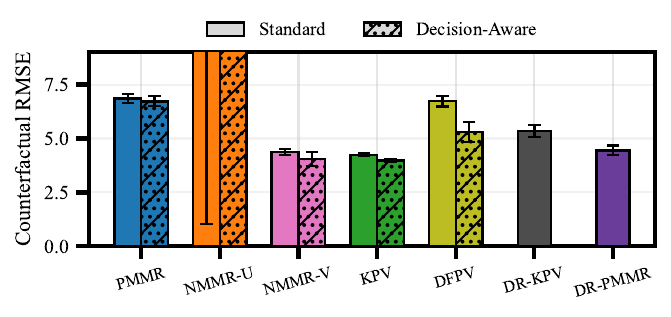}
    \caption{Synthetic experiment}
    \label{fig:synthetic_counterfactual}
\end{subfigure}
\begin{subfigure}{0.49\linewidth}
        \centering
    \includegraphics[width=\linewidth]{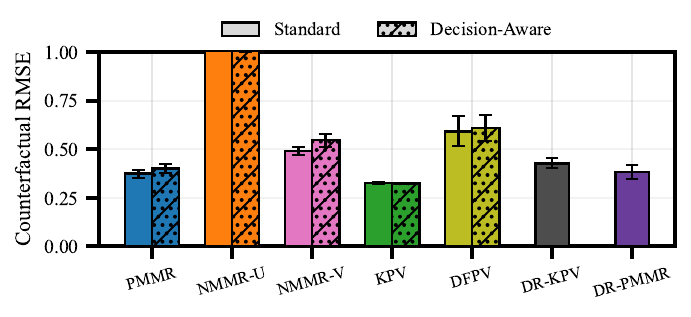}
    \caption{semi-synthetic experiment}
    \label{fig:semi-synthetic_counterfactual}
\end{subfigure}
        \caption{Counterfactual RMSE in the synthetic and semi-synthetic datasets. Mean and 95$\%$ confidence intervals over 10 seeds are shown. The behavior of NMMR-U can be more clearly appreciated in Appendix~\ref{app:additional-results}, where we provide the corresponding counterfactual-RMSE tables.}
    \label{fig:counterfactual_results}
\end{figure}

\begin{figure}[b]
\begin{subfigure}{0.32\linewidth}
        \centering
    \includegraphics[width=\linewidth]{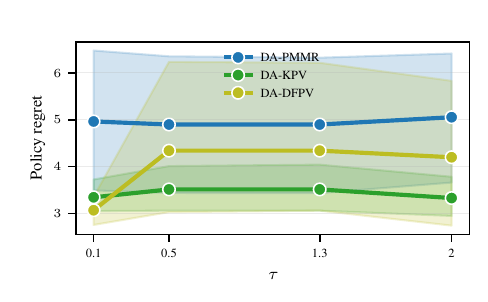}
    \caption{Varying $\tau$}
    \label{fig:tau}
\end{subfigure}
\begin{subfigure}{0.32\linewidth}
        \centering
    \includegraphics[width=\linewidth]{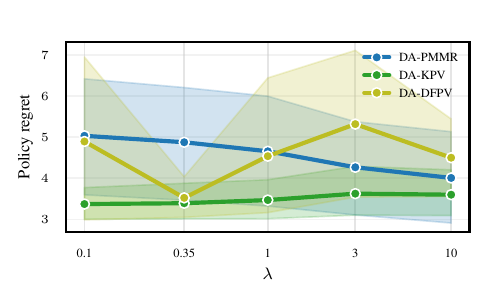}
    \caption{Varying $\lambda$}
    \label{fig:lambda}
\end{subfigure}
\begin{subfigure}{0.32\linewidth}
    \centering
    \includegraphics[width=\linewidth]{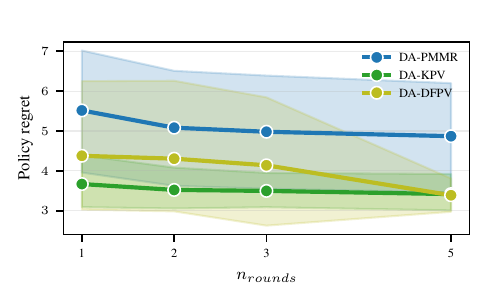}
    \caption{Varying $n_{rounds}$}
    \label{fig:nrounds}
\end{subfigure}
\caption{Sensitivity of weighting hyperparameters, holding the other two fixed. Mean and $95\%$ confidence intervals over the regret in 5 seeds reported.}
\label{fig:ablation}
\end{figure}

\textbf{Sensitivity to weighting hyperparameters.}
We vary the weighting bandwidth $\tau$, localization strength
$\lambda$, and number of reweighting rounds $n_{\mathrm{rounds}}$
one at a time. The
results in Figure~\ref{fig:ablation} show
that performance is relatively stable over a nontrivial range of bandwidths
$\tau$. DA-PMMR and DA-KPV change only mildly as $\tau$ varies, while DA-DFPV is
somewhat more sensitive, with stronger performance for narrower localization in
this experiment. Varying $\lambda$ shows the expected trade-off between
global stabilization and local policy targeting. Larger values tend to improve
DA-PMMR by concentrating more weight around the pseudo-optimal treatment, whereas
DA-KPV remains comparatively stable and DA-DFPV exhibits a more non-monotone
response. Finally, increasing $n_{\mathrm{rounds}}$ generally reduces regret,
consistent with the iterative interpretation of the method. The gains become
smaller after a few rounds.

\begin{wrapfigure}{r}{0.42\textwidth}
    \centering
    \includegraphics[width=\linewidth]{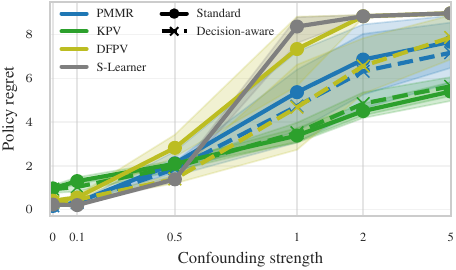}
    \caption{Regret as a function of the confounding strength for three of the DA models. Mean and 95\% confidence intervals over 5 seeds reported.}
    \label{fig:confounding-ablation-regret}
\end{wrapfigure}

\textbf{Sensitivity to hidden confounding.}
We vary the hidden-confounding strength on the synthetic
benchmark. As shown in Figure~\ref{fig:confounding-ablation-regret}, regret increases for all methods
as confounding becomes stronger. The non-proximal S-learner deteriorates
sharply, highlighting the importance of proxy-based adjustment. Among proximal
methods, decision-aware weighting often improves regret, especially for PMMR and
DFPV at moderate to high confounding levels. For this experiment, we used only 5 seeds while for the main experiments we used 10.

\textbf{Sensitivity to weakened completeness.}
We also assess sensitivity to completeness, a central proximal identification assumption, by perturbing only the proxy-generation mechanism in the synthetic benchmark. At degradation level $\alpha$, latent values in an expanding central band are removed before entering the corresponding proxies, and independent Gaussian noise with increasing variance is subsequently added; $\alpha=0$ recovers the original DGP. Across PMMR, KPV, and DFPV, performance generally deteriorates as $\alpha$ increases, consistently with the progressive loss of proxy information. Nevertheless, the decision-aware variants attain lower mean policy regret than their unweighted counterparts at every degradation level considered, while global counterfactual RMSE remains broadly comparable. Thus, decision-aware weighting can remain beneficial under moderate proxy degradation, although it cannot recover identifying information that the proxies no longer contain. Full details and results are provided in Appendix~\ref{app:completeness-sensitivity}.

\textbf{Pure-local ablation.}
Finally, we compare unweighted proximal learners, pure-local weighting, and local-plus-global decision-aware weighting on the synthetic benchmark. The local-plus-global variant attains lower mean policy regret and global counterfactual RMSE than pure-local weighting for PMMR, KPV, and DFPV (Appendix~\ref{app:pure-local-ablation}). These findings support the intended role of the global component in preventing exclusive focus on the current pseudo-optimal treatment region and providing global stabilization.

\section{Conclusion}
\label{sec:discussion}
We introduced a decision-aware framework for proximal bridge learning in individualized continuous-treatment selection. Rather than estimating the causal response surface uniformly over the treatment space, the proposed objective weights bridge estimation toward treatment regions that are most relevant for the induced policy, while retaining a global component for stability. Our theory shows that the resulting weighted bridge loss controls a regret-oriented surface surrogate through a weighted ill-posedness condition, without changing the population bridge target. We instantiated this idea as a general wrapper around several proximal bridge solvers, including PMMR, NMMR, KPV, and DFPV. Across synthetic and semi-synthetic benchmarks with hidden confounding and proxy variables, decision-aware weighting generally reduced policy regret, often without corresponding improvements in factual prediction error. These results suggest that, in proximal causal inference, aligning bridge estimation with the downstream treatment-selection objective can improve decision quality beyond globally accurate response estimation. As with other proximal methods, the main limitation of our approach is its reliance on identification assumptions involving valid proxies and bridge existence, which may be difficult to verify in practice.

\section*{Acknowledgments}
We gratefully acknowledge Novartis for sponsoring Tomàs Garriga’s industrial PhD. We also thank the Government of Catalonia’s Industrial PhDs Plan for funding part of this research. We acknowledge Horizon Europe Programme under the AI4DEBUNK Project (\url{https://www.ai4debunk.eu}), grant agreement num. 101135757. Alejandro Almodóvar and Juan Parras have received the funding from the \href{https://www.ihi-synthia.eu/}{SYNTHIA}
project, supported by the Innovative Health Initiative Joint Undertaking (IHI
JU) under grant agreement No 101172872. Funded by the European Union, the private members, and those contributing partners of the IHI JU. Views and opinions expressed are, however, those of the authors only and do not necessarily reflect
those of the aforementioned funding agencies. Neither of the aforementioned parties can be held
responsible for them.

The authors declare no competing interests beyond the affiliations and funding sources disclosed above.

\bibliographystyle{plainnat}
\bibliography{proximal_inference_refs}

@article{miao2018identifying,
  title = {Identifying Causal Effects with Proxy Variables of an Unmeasured Confounder},
  author = {Miao, Wang and Geng, Zhi and Tchetgen Tchetgen, Eric J.},
  journal = {Biometrika},
  volume = {105},
  number = {4},
  pages = {987--993},
  year = {2018},
  month = dec,
}

@article{miao2024confounding,
  title = {A Confounding Bridge Approach for Double Negative Control Inference on Causal Effects},
  author = {Miao, Wang and Shi, Xu and Li, Yilin and Tchetgen Tchetgen, Eric J.},
  journal = {Statistical Theory and Related Fields},
  volume = {8},
  number = {4},
  pages = {262--273},
  year = {2024},
  month = oct,
}

@article{tchetgen2024proximal,
  title = {An Introduction to Proximal Causal Inference},
  author = {Tchetgen Tchetgen, Eric J. and Ying, Andrew and Cui, Yifan and Shi, Xu and Miao, Wang},
  journal = {Statistical Science},
  volume = {39},
  number = {3},
  pages = {375--390},
  year = {2024},
}

@article{cui2024semiparametric,
  title = {Semiparametric Proximal Causal Inference},
  author = {Cui, Yifan and Pu, Hongming and Shi, Xu and Miao, Wang and Tchetgen Tchetgen, Eric J.},
  journal = {Journal of the American Statistical Association},
  volume = {119},
  number = {546},
  pages = {1348--1359},
  year = {2024},
}

@inproceedings{ghassami2022minimax,
  title = {Minimax Kernel Machine Learning for a Class of Doubly Robust Functionals with Application to Proximal Causal Inference},
  author = {Ghassami, Amiremad and Ying, Andrew and Shpitser, Ilya and Tchetgen Tchetgen, Eric},
  booktitle = {Proceedings of The 25th International Conference on Artificial Intelligence and Statistics},
  pages = {7210--7239},
  year = {2022},
  volume = {151},
  series = {Proceedings of Machine Learning Research},
  month = {28--30 Mar},
  publisher = {PMLR},
}

@article{miao2023identifying,
  title = {Identifying Effects of Multiple Treatments in the Presence of Unmeasured Confounding},
  author = {Miao, Wang and Hu, Wenjie and Ogburn, Elizabeth L. and Zhou, Xiao-Hua},
  journal = {Journal of the American Statistical Association},
  volume = {118},
  number = {543},
  pages = {1953--1967},
  year = {2023},
}

@article{ying2023proximal,
  title = {Proximal Causal Inference for Complex Longitudinal Studies},
  author = {Ying, Andrew and Miao, Wang and Shi, Xu and Tchetgen Tchetgen, Eric J.},
  journal = {Journal of the Royal Statistical Society Series B: Statistical Methodology},
  volume = {85},
  number = {3},
  pages = {684--704},
  year = {2023},
}

@article{zhang2023proximal,
  title = {Proximal Causal Inference without Uniqueness Assumptions},
  author = {Zhang, Jeffrey and Li, Wei and Miao, Wang and Tchetgen Tchetgen, Eric J.},
  journal = {Statistics \& Probability Letters},
  volume = {198},
  pages = {109836},
  year = {2023},
}

@inproceedings{mastouri2021proximal,
  title = {Proximal Causal Learning with Kernels: Two-Stage Estimation and Moment Restriction},
  author = {Mastouri, Afsaneh and Zhu, Yuchen and Gultchin, Limor and Korba, Anna and Silva, Ricardo and Kusner, Matt and Gretton, Arthur and Muandet, Krikamol},
  booktitle = {Proceedings of the 38th International Conference on Machine Learning},
  pages = {7512--7523},
  year = {2021},
  editor = {Meila, Marina and Zhang, Tong},
  volume = {139},
  series = {Proceedings of Machine Learning Research},
  month = {18--24 Jul},
  publisher = {PMLR},
}

@misc{kallus2021causal,
  title = {Causal Inference Under Unmeasured Confounding with Negative Controls: A Minimax Learning Approach},
  author = {Kallus, Nathan and Mao, Xiaojie and Uehara, Masatoshi},
  year = {2021},
  eprint = {2103.14029},
  archivePrefix = {arXiv},
  primaryClass = {stat.ML},
  url = {https://arxiv.org/abs/2103.14029}
}

@inproceedings{kompa2022deep,
  title = {Deep Learning Methods for Proximal Inference via Maximum Moment Restriction},
  author = {Kompa, Benjamin and Bellamy, David and Kolokotrones, Tom and Robins, James M. and Beam, Andrew},
  booktitle = {Advances in Neural Information Processing Systems},
  volume = {35},
  year = {2022},
}

@inproceedings{wu2024doubly,
  title = {Doubly Robust Proximal Causal Learning for Continuous Treatments},
  author = {Wu, Yong and Fu, Yanwei and Wang, Shouyan and Sun, Xinwei},
  booktitle = {International Conference on Learning Representations},
  year = {2024},
}

@article{qi2024proximal,
  title = {Proximal Learning for Individualized Treatment Regimes Under Unmeasured Confounding},
  author = {Qi, Zhengling and Miao, Rui and Zhang, Xiaoke},
  journal = {Journal of the American Statistical Association},
  volume = {119},
  number = {546},
  pages = {915--928},
  year = {2024},
}

@inproceedings{shen2023optimal,
  title = {Optimal Treatment Regimes for Proximal Causal Learning},
  author = {Shen, Tao and Cui, Yifan},
  booktitle = {Advances in Neural Information Processing Systems},
  volume = {36},
  pages = {47735--47748},
  year = {2023},
}

@inproceedings{zou2022counterfactual,
  title = {Counterfactual Prediction for Outcome-Oriented Treatments},
  author = {Zou, Hao and Li, Bo and Han, Jiangang and Chen, Shuiping and Ding, Xuetao and Cui, Peng},
  booktitle = {Proceedings of the 39th International Conference on Machine Learning},
  pages = {27693--27706},
  year = {2022},
  editor = {Chaudhuri, Kamalika and Jegelka, Stefanie and Song, Le and Szepesvari, Csaba and Niu, Gang and Sabato, Sivan},
  volume = {162},
  series = {Proceedings of Machine Learning Research},
  month = {17--23 Jul},
  publisher = {PMLR},
}

@article{zhao2012estimating,
  title = {Estimating Individualized Treatment Rules Using Outcome Weighted Learning},
  author = {Zhao, Yingqi and Zeng, Donglin and Rush, A. John and Kosorok, Michael R.},
  journal = {Journal of the American Statistical Association},
  volume = {107},
  number = {499},
  pages = {1106--1118},
  year = {2012},
}

@article{fernandez2022causal,
  title = {Causal Decision Making and Causal Effect Estimation Are Not the Same... and Why It Matters},
  author = {Fern{\'a}ndez-Lor{\'i}a, Carlos and Provost, Foster},
  journal = {INFORMS Journal on Data Science},
  volume = {1},
  number = {1},
  pages = {4--16},
  year = {2022},
}

@incollection{hirano2004propensity,
  title = {The Propensity Score with Continuous Treatments},
  author = {Hirano, Keisuke and Imbens, Guido W.},
  booktitle = {Applied Bayesian Modeling and Causal Inference from Incomplete-Data Perspectives},
  editor = {Gelman, Andrew and Meng, Xiao-Li},
  pages = {73--84},
  publisher = {John Wiley \& Sons},
  address = {Hoboken, NJ},
  year = {2004},
}

@inproceedings{kallus2018policy,
  title = {Policy Evaluation and Optimization with Continuous Treatments},
  author = {Kallus, Nathan and Zhou, Angela},
  booktitle = {Proceedings of the Twenty-First International Conference on Artificial Intelligence and Statistics},
  pages = {1243--1251},
  year = {2018},
  editor = {Storkey, Amos and Perez-Cruz, Fernando},
  volume = {84},
  series = {Proceedings of Machine Learning Research},
  publisher = {PMLR},
}

@article{elmachtoub2022smart,
  title = {Smart ``Predict, then Optimize''},
  author = {Elmachtoub, Adam N. and Grigas, Paul},
  journal = {Management Science},
  volume = {68},
  number = {1},
  pages = {9--26},
  year = {2022},
}

@article{mandi2024decision,
  title = {Decision-Focused Learning: Foundations, State of the Art, Benchmark and Future Opportunities},
  author = {Mandi, Jayanta and Kotary, James and Berden, Senne and Mulamba, Maxime and Bucarey, Victor and Guns, Tias and Fioretto, Ferdinando},
  journal = {Journal of Artificial Intelligence Research},
  volume = {80},
  year = {2024},
}

@article{singh2024kernel,
  title = {Kernel Methods for Causal Functions: Dose, Heterogeneous and Incremental Response Curves},
  author = {Singh, Rahul and Xu, Liyuan and Gretton, Arthur},
  journal = {Biometrika},
  volume = {111},
  number = {2},
  pages = {497--516},
  year = {2024},
}

@inproceedings{xu2021deep,
  title = {Deep Proxy Causal Learning and its Application to Confounded Bandit Policy Evaluation},
  author = {Xu, Liyuan and Kanagawa, Heishiro and Gretton, Arthur},
  booktitle = {Advances in Neural Information Processing Systems},
  volume = {34},
  pages = {26264--26275},
  year = {2021},
}

@inproceedings{bozkurt2025densityratio,
  title = {Density Ratio-based Proxy Causal Learning Without Density Ratios},
  author = {Bozkurt, Bariscan and Deaner, Ben and Meunier, Dimitri and Xu, Liyuan and Gretton, Arthur},
  booktitle = {Proceedings of The 28th International Conference on Artificial Intelligence and Statistics},
  pages = {5095--5103},
  year = {2025},
  volume = {258},
  series = {Proceedings of Machine Learning Research},
  publisher = {PMLR},
}

@inproceedings{bozkurt2025drfree,
  title = {Density Ratio-Free Doubly Robust Proxy Causal Learning},
  author = {Bozkurt, Bariscan and Zenati, Houssam and Meunier, Dimitri and Xu, Liyuan and Gretton, Arthur},
  booktitle = {Advances in Neural Information Processing Systems},
  year = {2025},
}

@inproceedings{dikkala2020minimax,
  title        = {Minimax Estimation of Conditional Moment Models},
  author       = {Dikkala, Nishanth and Lewis, Greg and Mackey, Lester W. and Syrgkanis, Vasilis},
  booktitle    = {Advances in Neural Information Processing Systems},
  volume       = {33},
  pages        = {12248--12262},
  year         = {2020},
}

@article{chen2012nonparametric,
  title        = {Estimation of Nonparametric Conditional Moment Models With Possibly Nonsmooth Generalized Residuals},
  author       = {Chen, Xiaohong and Pouzo, Demian},
  journal      = {Econometrica},
  volume       = {80},
  number       = {1},
  pages        = {277--321},
  year         = {2012},
}

@article{bica2020scigan,
  title={Estimating the Effects of Continuous-valued Interventions using Generative Adversarial Networks},
  author={Bica, Ioana and Jordon, James and van der Schaar, Mihaela},
  journal={Advances in neural information processing systems},
  year={2020}
}

@article{weinstein2013cancer,
  title={The Cancer Genome Atlas Pan-Cancer analysis project},
  author={Weinstein, John N. and Collisson, Eric A. and Mills, Gordon B. and Mills Shaw, Kenna R. and Ozenberger, Brad A. and Ellrott, Kyle and Shmulevich, Ilya and Sander, Chris and Stuart, Joshua M. and Cancer Genome Atlas Research Network},
  journal={Nature Genetics},
  volume={45},
  number={10},
  pages={1113–1120},
  year={2013}
}

@inproceedings{bergstra2011algorithms,
  title={Algorithms for Hyper-Parameter Optimization},
  author={Bergstra, James and Bardenet, R{\'e}mi and Bengio, Yoshua and K{\'e}gl, Bal{\'a}zs},
  booktitle={Advances in Neural Information Processing Systems},
  year={2011}
}

@inproceedings{akiba2019optuna,
  title={Optuna: A Next-generation Hyperparameter Optimization Framework},
  author={Akiba, Takuya and Sano, Shotaro and Yanase, Toshihiko and Ohta, Takeru and Koyama, Masanori},
  booktitle={Proceedings of the 25th ACM SIGKDD International Conference on Knowledge Discovery and Data Mining},
  year={2019}
}

@article{rubin1974estimating,
  title={Estimating Causal Effects of Treatments in Randomized and Nonrandomized Studies},
  author={Rubin, Donald B.},
  journal={Journal of Educational Psychology},
  volume={66},
  number={5},
  pages={688--701},
  year={1974}
}

@book{pearl2009causality,
  title={Causality: Models, Reasoning, and Inference},
  author={Pearl, Judea},
  edition={2},
  publisher={Cambridge University Press},
  year={2009}
}

@book{imbens2015causal,
  title={Causal Inference for Statistics, Social, and Biomedical Sciences: An Introduction},
  author={Imbens, Guido W. and Rubin, Donald B.},
  publisher={Cambridge University Press},
  year={2015}
}

@book{hernan2020causal,
  title={Causal Inference: What If},
  author={Hern{\'a}n, Miguel A. and Robins, James M.},
  publisher={Chapman \& Hall/CRC},
  year={2020}
}

@inproceedings{johansson2016learning,
  title={Learning Representations for Counterfactual Inference},
  author={Johansson, Fredrik D. and Shalit, Uri and Sontag, David},
  booktitle={Proceedings of the 33rd International Conference on Machine Learning},
  pages={3020--3029},
  year={2016}
}

@inproceedings{shalit2017estimating,
  title={Estimating Individual Treatment Effect: Generalization Bounds and Algorithms},
  author={Shalit, Uri and Johansson, Fredrik D. and Sontag, David},
  booktitle={Proceedings of the 34th International Conference on Machine Learning},
  pages={3076--3085},
  year={2017}
}

@article{wager2018estimation,
  title={Estimation and Inference of Heterogeneous Treatment Effects using Random Forests},
  author={Wager, Stefan and Athey, Susan},
  journal={Journal of the American Statistical Association},
  volume={113},
  number={523},
  pages={1228--1242},
  year={2018}
}

@inproceedings{bica2020estimating,
  title={Estimating Counterfactual Treatment Outcomes over Time Through Adversarially Balanced Representations},
  author={Bica, Ioana and Alaa, Ahmed M. and Jordon, James and van der Schaar, Mihaela},
  booktitle={Proceedings of the 8th International Conference on Learning Representations},
  year={2020}
}

@inproceedings{melnychuk2022causal,
  title={Causal Transformer for Estimating Counterfactual Outcomes},
  author={Melnychuk, Valentyn and Frauen, Dennis and Feuerriegel, Stefan},
  booktitle={Proceedings of the 39th International Conference on Machine Learning},
  pages={15293--15329},
  year={2022}
}

@article{
garriga2026cepae,
title={{CEPAE}: Conditional Entropy-Penalized Autoencoders for Time Series Counterfactuals},
author={Tomas Garriga and Gerard Sanz and Eduard Serrahima de Cambra and Axel Brando},
journal={Transactions on Machine Learning Research},
issn={2835-8856},
year={2026}
}

@inproceedings{almodovar2025decaflow,
 author = {Almod\'{o}var, Alejandro and Javaloy, Adri\'{a}n and Parras, Juan and Zazo, Santiago and Valera, Isabel},
 booktitle = {Advances in Neural Information Processing Systems},
 pages = {136612--136666},
 title = {{DeCaFlow}: A deconfounding causal generative model},
 volume = {38, Main Conference},
 year = {2025},
 doi={10.52202/085713-4558}
}

@ARTICLE{9815055,
  author={Parafita, Álvaro and Vitrià, Jordi},
  journal={IEEE Access}, 
  title={Estimand-Agnostic Causal Query Estimation With Deep Causal Graphs}, 
  year={2022},
  volume={10},
  number={},
  pages={71370-71386},
}

@inproceedings{
hess2026igcnetconditionalaveragepotential,
title={{IGC}-Net for conditional average potential outcome estimation over time},
author={Konstantin Hess and Dennis Frauen and Valentyn Melnychuk and Stefan Feuerriegel},
booktitle={The Fourteenth International Conference on Learning Representations},
year={2026}
}

@inproceedings{li2021g,
  title={G-net: a recurrent network approach to g-computation for counterfactual prediction under a dynamic treatment regime},
  author={Li, Rui and Hu, Stephanie and Lu, Mingyu and Utsumi, Yuria and Chakraborty, Prithwish and Sow, Daby M and Madan, Piyush and Li, Jun and Ghalwash, Mohamed and Shahn, Zach and others},
  booktitle={Machine Learning for Health},
  pages={282--299},
  year={2021},
  organization={PMLR}
}

@inproceedings{NEURIPS2025_c2481541,
 author = {Frauen, Dennis and Melnychuk, Valentyn and Schweisthal, Jonas and van der Schaar, Mihaela and Feuerriegel, Stefan},
 booktitle = {Advances in Neural Information Processing Systems},
 editor = {D. Belgrave and C. Zhang and H. Lin and R. Pascanu and P. Koniusz and M. Ghassemi and N. Chen},
 pages = {133883--133918},
 publisher = {Curran Associates, Inc.},
 title = {Treatment Effect Estimation for Optimal Decision-Making},
 volume = {38, Main Conference},
 year = {2025}
}

@misc{garriga2026orthogengenerativeorthogonallearner,
      title={OrthoGen: A Generative Orthogonal Learner for Time-Varying Treatments}, 
      author={Tomàs Garriga and Valentyn Melnychuk and Konstantin Hess and Eduard Serrahima de Cambra and Axel Brando and Gerard Sanz and Stefan Feuerriegel},
      year={2026},
      eprint={2610.10210},
      archivePrefix={arXiv},
      primaryClass={cs.LG},
      url={https://arxiv.org/abs/2610.10210}, 
}

\newpage
\appendix

\startcontents[appendix]

\section*{Appendix Index}
\printcontents[appendix]{}{1}{%
  \setcounter{tocdepth}{2}%
}

\newpage

\section{Extended Related Work}
\label{app:extended-related-work}

\paragraph{Causal inference and representation learning.}
Causal inference provides formal tools for reasoning about interventions and counterfactuals from randomized or observational data, with foundations in potential outcomes, graphical models, and semiparametric estimation \citep{rubin1974estimating,pearl2009causality,imbens2015causal,hernan2020causal}. Modern machine-learning approaches build on these ideas by learning representations, nuisance functions, or policies for heterogeneous treatment-effect estimation, decision making and explainability \citep{johansson2016learning,shalit2017estimating,wager2018estimation, 9815055}. Related ideas appear in several domains, including longitudinal and time-series settings \citep{bica2020estimating,melnychuk2022causal, li2021g,garriga2026cepae, garriga2026orthogengenerativeorthogonallearner, hess2026igcnetconditionalaveragepotential}. Our setting differs from this line of work because the relevant dose-response is continuous and hidden confounding is addressed through negative-control proxies rather than by assuming that all confounders are observed.

\paragraph{Decision-oriented causal learning and continuous treatments.}
A large literature studies individualized treatment rules, treatment assignment, and policy learning from observational or logged data. Classical individualized-treatment-rule methods often reduce value maximization to a weighted classification or regression problem; for example, outcome-weighted learning estimates treatment rules by weighting classification decisions according to observed outcomes and treatment propensities \citep{zhao2012estimating}. More broadly, work on causal decision making has emphasized that optimizing a treatment rule is not the same statistical objective as estimating treatment effects or potential-outcome surfaces uniformly well \citep{fernandez2022causal}. This distinction is especially important for continuous treatments, where the learner must recommend a dose or intensity rather than choose among a small number of actions. Early work on generalized propensity scores formalized dose-response estimation for continuous exposures \citep{hirano2004propensity}, while later work developed off-policy evaluation and optimization methods for continuous actions using kernelized inverse-propensity and doubly robust estimators \citep{kallus2018policy}. These approaches motivate our continuous-treatment setting and our focus on regret, but they typically rely on measured-covariate exchangeability rather than proxy-based identification under latent confounding.

\paragraph{Outcome-oriented and decision-focused counterfactual prediction.}
Our motivation is closely related to the broader decision-focused learning principle that prediction errors should be weighted by their downstream decision consequences. In operations and machine learning, predict-then-optimize methods such as SPO train predictors using losses that measure the quality of the induced decision rather than only pointwise prediction error \citep{elmachtoub2022smart}; recent surveys describe decision-focused learning as an end-to-end approach for aligning predictive models with constrained optimization objectives \citep{mandi2024decision}. In causal inference, \citet{zou2022counterfactual} make an analogous point for counterfactual prediction: global accuracy over the treatment space may be poorly aligned with treatment-selection quality, and estimators should emphasize regions that affect the recommended action. Our work adopts this outcome-oriented perspective, but differs in the source of identification. Existing outcome-oriented counterfactual prediction methods generally assume that confounding is removed by observed covariates, whereas our setting requires proximal bridge functions and proxy variables to handle unmeasured confounding.

\paragraph{Proximal causal inference.}
Proximal causal inference addresses latent confounding by using observed proxies, often formalized as treatment-inducing proxies $Z$ and outcome-inducing proxies $W$, to solve bridge-function moment equations. Early nonparametric identification results with proxy variables were established by \citet{miao2018identifying}, while the double-negative-control bridge formulation was developed by \citet{miao2024confounding}. \citet{tchetgen2024proximal} placed these ideas in a potential-outcome framework and derived the proximal g-formula, showing how causal effects can be identified even when exchangeability with respect to measured covariates fails. Subsequent work has expanded the framework to semiparametric and doubly robust inference \citep{cui2024semiparametric,ghassami2022minimax}, multiple-treatment settings \citep{miao2023identifying}, longitudinal studies \citep{ying2023proximal}, and settings where bridge solutions may not be unique \citep{zhang2023proximal}. This literature provides the identification foundation for our method. The key difference is that most proximal estimands and estimators are designed for effect or response-surface recovery, whereas we target the downstream continuous-treatment decision induced by the learned bridge.

\paragraph{Machine learning for bridge estimation.}
A central computational challenge in proximal inference is that bridge functions solve ill-posed conditional moment or Fredholm integral equations. Kernel methods by \citet{mastouri2021proximal} formulate proximal causal learning through two-stage regression and maximum moment restriction in reproducing-kernel Hilbert spaces. Minimax approaches estimate bridge functions through adversarial conditional-moment objectives \citep{kallus2021causal,ghassami2022minimax}, and \citet{kompa2022deep} scale this idea with neural maximum moment restriction. For continuous treatments, \citet{wu2024doubly} propose a kernel-smoothed doubly robust proximal estimator, addressing the difficulty that binary-treatment proximal doubly robust estimators do not transfer directly to continuous actions because exact treatment matching has probability zero. \citet{almodovar2025decaflow} propose a deconfounding normalizing flow that, applied to proximal settings, implicitly solves the bridge equation. These works improve the statistical and computational toolkit for estimating bridges and causal response functions. Our objective is complementary: rather than changing only the bridge estimator class or the moment-restriction solver, we change the loss being optimized so that bridge fitting is concentrated in treatment regions that matter for the induced policy, while retaining a global component for stabilization.

\paragraph{Proximal treatment-rule learning.}
The closest proximal decision-making literature studies optimal treatment regimes under unmeasured confounding. \citet{qi2024proximal} develop proximal learning procedures for individualized treatment regimes by combining bridge-based identification with classification-style policy learning, and \citet{shen2023optimal} characterize optimal treatment regimes using outcome and treatment confounding bridges. These works show that proximal identification can support decision making under latent confounding. They primarily focus on identifying and optimizing treatment-regime values over specified policy classes, often in discrete-action settings. Our focus is different but complementary: we study continuous treatment selection and ask how the proximal bridge-learning objective itself should be shaped when the downstream criterion is dose-selection regret. In this sense, our method connects proximal policy learning with outcome-oriented counterfactual prediction by introducing a localized, policy-aware bridge loss.

\section{Standard proximal inference assumptions}
\label{app:assumptions}

In this appendix, we provide additional discussion of
Assumptions~\ref{ass:consistency-exchangeability}--\ref{ass:positivity},
which are stated in the main text.

\paragraph{Assumption~\ref{ass:consistency-exchangeability}: Consistency and latent exchangeability.}
Assumption~\ref{ass:consistency-exchangeability} says that treatment is
as-if randomized after conditioning on the observed covariates $X$ and the
latent confounders $U$. Since $U$ is unobserved, this condition cannot be
used directly for ordinary covariate adjustment.

\paragraph{Assumption~\ref{ass:negative-control}: Negative-control proxies.}
Assumption~\ref{ass:negative-control} requires that, for every $a\in\cA$,
$Y(a)\perp\!\!\!\perp Z\mid A,U,X$ and
$W\perp\!\!\!\perp (A,Z)\mid U,X$.

The first relation says that, after conditioning on treatment, observed
covariates, and latent confounders, the treatment-inducing proxy $Z$ has
no residual association with the potential outcome. The second says that,
after conditioning on $(U,X)$, the outcome-inducing proxy $W$ has no
residual association with treatment assignment or with $Z$. Thus, the
proxies are useful because they carry information about latent confounding
while satisfying exclusion restrictions that make proximal identification
possible.

\paragraph{Assumption~\ref{ass:completeness}: Completeness.}
Assumption~\ref{ass:completeness} is a proxy-relevance condition. It rules
out nonzero latent or proxy signals that are invisible after conditioning
on the observed proxy system. In discrete settings, it requires the proxies
to contain enough variation to distinguish the relevant latent confounding
states; in continuous settings, it is the nonparametric analogue of an
injectivity condition for conditional-expectation operators.

\paragraph{Assumption~\ref{ass:positivity}: Positivity.}
Assumption~\ref{ass:positivity} ensures that every treatment level in the
target action space is observable with positive probability, both within
latent strata and after marginalizing over the latent confounders.

\section{Sufficient conditions for outcome-bridge existence}
\label{app:bridge-existence}

This appendix records one set of sufficient conditions under which the
outcome bridge used in the main text exists. The main results of the
paper are developed conditional on a valid bridge equation; the conditions
below are included to connect this bridge equation to the standard
proximal identification framework.

Throughout this appendix, assume that
Assumptions~\ref{ass:consistency-exchangeability}--
\ref{ass:positivity} hold. For each $(a,x)$, define the latent outcome
regression
\begin{equation}
    q_0(a,u,x)
    :=
    \E\{Y\mid A=a,U=u,X=x\}.
\end{equation}
By Assumption~\ref{ass:negative-control},
\begin{equation}
    \E\{Y\mid A=a,Z=z,X=x\}
    =
    \E\{q_0(a,U,x)\mid A=a,Z=z,X=x\}.
\end{equation}

We impose the following additional regularity and solvability conditions.

\begin{assumption}[Regular conditional laws and square integrability]
\label{ass:appendix-regularity}
The variables $(Y,A,Z,W,X,U)$ admit regular conditional distributions.
For every $a\in\cA$, the regression $q_0(a,U,X)$ is square-integrable.
Moreover, the conditional laws of $W$ given $(U,X)$ and of $U$ given
$(A,Z,X)$ are such that the conditional expectations appearing below are
well defined as maps between the corresponding $L_2$ spaces.
\end{assumption}

\begin{assumption}[Latent bridge solvability]
\label{ass:latent-bridge-solvability}
For every $a\in\cA$, there exists a square-integrable function
\begin{equation}
    h_0(a,\cdot,\cdot):\mathcal W\times\cX\to\R
\end{equation}
such that, for almost every $(u,x)$,
\begin{equation}
    q_0(a,u,x)
    =
    \E\{h_0(a,W,x)\mid U=u,X=x\}.
    \label{eq:latent-bridge-equation}
\end{equation}
\end{assumption}

Assumption~\ref{ass:latent-bridge-solvability} is a range condition for
the conditional-expectation operator that maps functions of $(W,X)$ to
functions of $(U,X)$. It says that the latent outcome regression
$q_0(a,U,X)$ can be represented by averaging a function of the outcome
proxy $W$ over the conditional distribution of $W$ given $(U,X)$.

\begin{proposition}[Outcome-bridge existence]
\label{prop:appendix-outcome-bridge-existence}
Under Assumptions~\ref{ass:consistency-exchangeability}--
\ref{ass:positivity} and
\ref{ass:appendix-regularity}--\ref{ass:latent-bridge-solvability}, the
function $h_0$ in Assumption~\ref{ass:latent-bridge-solvability} satisfies
the observed outcome-bridge equation
\begin{equation}
    \E\{Y-h_0(A,W,X)\mid A,Z,X\}=0 .
\end{equation}
Equivalently, for every $(a,z,x)$ in the support of $(A,Z,X)$,
\begin{equation}
    \E\{Y\mid A=a,Z=z,X=x\}
    =
    \E\{h_0(a,W,x)\mid A=a,Z=z,X=x\}.
\end{equation}
\end{proposition}

\begin{proof}
Fix $(a,z,x)$ in the support of $(A,Z,X)$. By iterated expectation and
the first negative-control relation,
\begin{equation}
\begin{aligned}
    \E\{Y\mid A=a,Z=z,X=x\}
    &=
    \E\!\left[
        \E\{Y\mid A=a,Z=z,U,X=x\}
        \mid A=a,Z=z,X=x
    \right] \\
    &=
    \E\!\left[
        \E\{Y\mid A=a,U,X=x\}
        \mid A=a,Z=z,X=x
    \right] \\
    &=
    \E\{q_0(a,U,x)\mid A=a,Z=z,X=x\}.
\end{aligned}
\end{equation}
By Assumption~\ref{ass:latent-bridge-solvability},
\begin{equation}
    q_0(a,U,x)
    =
    \E\{h_0(a,W,x)\mid U,X=x\}.
\end{equation}
Therefore,
\begin{equation}
\begin{aligned}
    \E\{Y\mid A=a,Z=z,X=x\}
    &=
    \E\!\left[
        \E\{h_0(a,W,x)\mid U,X=x\}
        \mid A=a,Z=z,X=x
    \right].
\end{aligned}
\end{equation}
By the second negative-control relation,
$W\perp\!\!\!\perp (A,Z)\mid U,X$, so
\begin{equation}
    \E\{h_0(a,W,x)\mid U,A=a,Z=z,X=x\}
    =
    \E\{h_0(a,W,x)\mid U,X=x\}.
\end{equation}
Applying iterated expectation again gives
\begin{equation}
\begin{aligned}
    \E\{Y\mid A=a,Z=z,X=x\}
    &=
    \E\!\left[
        \E\{h_0(a,W,x)\mid U,A=a,Z=z,X=x\}
        \mid A=a,Z=z,X=x
    \right] \\
    &=
    \E\{h_0(a,W,x)\mid A=a,Z=z,X=x\}.
\end{aligned}
\end{equation}
Hence
\begin{equation}
    \E\{Y-h_0(A,W,X)\mid A=a,Z=z,X=x\}=0,
\end{equation}
and the claim follows.
\end{proof}

\begin{remark}[Observed Fredholm form]
\label{rem:observed-fredholm-form}
The conclusion of Proposition~\ref{prop:appendix-outcome-bridge-existence}
can also be written as the observed Fredholm equation
\begin{equation}
    \E\{Y\mid A=a,Z=z,X=x\}
    =
    \int h_0(a,w,x)\,dP(w\mid A=a,Z=z,X=x).
\end{equation}
Thus the outcome bridge is a solution to a first-kind integral equation.
The range condition in Assumption~\ref{ass:latent-bridge-solvability} is
one way to ensure that this integral equation has a square-integrable
solution.
\end{remark}

\section{Proofs for the Policy-Targeted Risk}
\label{app:proofs}

This appendix proves the population bounds used in
Section~\ref{sec:targeted-risk}. The proof proceeds in three steps. First,
Lemma~\ref{lemma:regret-decomp-surface} reduces policy regret to two
surface errors: one at the treatment selected by the learned bridge and one at
the oracle treatment. Second, Propositions~\ref{prop:kernel-bias-surface}
and~\ref{prop:global-surface-bound} show how these two errors are controlled
by the local-plus-global surface surrogate. Third,
Proposition~\ref{prop:weighted_bridge_to_surface} connects the resulting
weighted surface loss to the weighted proximal bridge loss.

\subsection{Population objects and notation}
\label{app:population-objects}

For completeness, we restate the objects used in the proofs. Let
$O=(Y,A,Z,W,X)$ denote a generic observation, with $A\in\mathcal A=[c_-,c_+]$.
Let $h_0$ be an outcome bridge satisfying
\begin{equation}
    \E\{Y-h_0(A,W,X)\mid A,Z,X\}=0,
\end{equation}
and let $h\in\mathcal H$ be a candidate bridge. We write
\begin{equation}
    \Delta_h(a,w,x):=h(a,w,x)-h_0(a,w,x)
\end{equation}
for its bridge error. The true and bridge-induced causal response surfaces are
\begin{equation}
    m_0(a,x):=\E\{h_0(a,W,x)\mid X=x\},
    \qquad
    m_h(a,x):=\E\{h(a,W,x)\mid X=x\},
\end{equation}
where the conditional expectation over $W$ is with respect to
$P_{W\mid X=x}$. The squared surface error is
\begin{equation}
    G_h(t,x):=\{m_0(t,x)-m_h(t,x)\}^2 .
\end{equation}
The policy induced by $h$ and the oracle policy are
\begin{equation}
    \pi_h(x)\in\argmax_{t\in\mathcal A}m_h(t,x),
    \qquad
    a^\star(x)\in\argmax_{t\in\mathcal A}m_0(t,x),
\end{equation}
with the same fixed tie-breaking convention as in the main text. The regret is
\begin{equation}
    \Reg(h)
    :=
    \E\{m_0(a^\star(X),X)-m_0(\pi_h(X),X)\}.
\end{equation}

We use two population laws for $(A,W,X)$. The first is the observed joint law
\begin{equation}
    \nu:=P_{A,W,X}.
\end{equation}
The second is the product conditional law
\begin{equation}
    \mu:=P_{W\mid X}\otimes P_{A\mid X}\otimes P_X,
\end{equation}
meaning that, under $\mu$, $X\sim P_X$ and then $A$ and $W$ are drawn
independently from their observed conditional laws $P_{A\mid X}$ and
$P_{W\mid X}$. Thus $\mu$ is the law naturally associated with the surface
quantity $\E\{h(a,W,x)\mid X=x\}$ evaluated at treatments distributed as
$P_{A\mid X}$, while $\nu$ is the observed law under which the bridge function
is learned. We assume $\mu\ll\nu$ when invoking the bridge-to-surface result
and write
\begin{equation}
    C_\rho
    :=
    \left\|\frac{d\mu}{d\nu}\right\|_\infty
    <\infty,
\end{equation}
where the essential supremum is taken with respect to $\nu$. Equivalently, for
all nonnegative measurable $q$,
\begin{equation}
    \int q\,d\mu\le C_\rho\int q\,d\nu .
\end{equation}

For a nonnegative measurable weight
$\omega:\mathcal A\times\mathcal X\to[0,\infty)$, define weighted norms by
\begin{equation}
    \|f\|_{L_2(\omega\,d\nu)}^2
    :=
    \int \omega(a,x) f(a,w,x)^2\,d\nu(a,w,x),
\end{equation}
and analogously for $L_2(\omega\,d\mu)$ and
$L_2(\omega\,dP_{A,Z,X})$. The outcome-bridge operator is
\begin{equation}
    (T\Delta)(a,z,x)
    :=
    \E\{\Delta(a,W,x)\mid A=a,Z=z,X=x\}.
\end{equation}
The weighted surface and bridge risks are
\begin{align}
    \mathcal L_{\mathrm{surf},\omega}(h)
    &:=
    \E\!\left[
    \omega(A,X)\{m_h(A,X)-m_0(A,X)\}^2
    \right], \\
    \mathcal L_{\mathrm{br},\omega}(h)
    &:=
    \E\!\left[
    \omega(A,X)
    \E\{Y-h(A,W,X)\mid A,Z,X\}^2
    \right].
\end{align}
Finally, the weighted proximal ill-posedness constant is
\begin{equation}
    \chi_\omega
    :=
    \sup_{g\in\mathcal H:\,\|T\Delta_g\|_{L_2(\omega\,dP_{A,Z,X})}>0}
    \frac{\|\Delta_g\|_{L_2(\omega\,d\nu)}}
         {\|T\Delta_g\|_{L_2(\omega\,dP_{A,Z,X})}} .
\end{equation}
The condition $\chi_\omega<\infty$ is a weighted stability condition for the
proximal inverse problem: it rules out bridge-error directions that are large in
the weighted bridge norm but nearly invisible through the weighted conditional
moment operator.

When expectations of the form $\E_{X,A}$ appear in the surface bounds, $A$ is a
generic treatment draw from $P_{A\mid X}$ conditional on $X$. This notation
is only a device for writing integrals over treatment values. With this convention,
\begin{align}
    A_\tau(h)
    &:=
    \E_{X,A}\!\left[
    \frac{K((\pi_h(X)-A)/\tau)}{\tau p(A\mid X)}G_h(A,X)
    \right], \\
    B(h)
    &:=
    \E_X\!\left[
    \frac{1}{c_+-c_-}\int_{c_-}^{c_+}G_h(t,X)\,dt
    \right].
\end{align}

\subsection{Regret decomposition}
\label{app:regret-decomposition-proof}

We first prove the basic regret decomposition. This step explains why the
subsequent local and global surface bounds focus on the two treatment values
$\pi_h(X)$ and $a^\star(X)$.

\begin{proof}[Proof of Lemma~\ref{lemma:regret-decomp-surface}]
Fix $x$ and abbreviate $\pi_h(x)$ by $\pi_h$ and $a^\star(x)$ by $a^\star$.
Since $\pi_h$ maximizes $m_h(\cdot,x)$ over $\mathcal A$,
\begin{equation}
    m_h(\pi_h,x)\ge m_h(a^\star,x).
\end{equation}
Therefore,
\begin{align}
    m_0(a^\star,x)-m_0(\pi_h,x)
    &\le
    m_0(a^\star,x)-m_h(a^\star,x)
    +m_h(\pi_h,x)-m_0(\pi_h,x) \\
    &\le
    \left|m_0(a^\star,x)-m_h(a^\star,x)\right|
    +
    \left|m_h(\pi_h,x)-m_0(\pi_h,x)\right| .
\end{align}
Taking expectation over $X$ and applying Cauchy--Schwarz to each absolute-error
term gives
\begin{equation}
    \Reg(h)
    \le
    \sqrt{\E\!\left[G_h\!\left(a^\star(X),X\right)\right]}
    +
    \sqrt{\E\!\left[G_h\!\left(\pi_h(X),X\right)\right]} .
\end{equation}
\end{proof}

\subsection{Kernel approximation and global surface bounds}
\label{app:surface-bounds}

The local term \(A_\tau(h)\) is a kernel approximation to the squared surface
error at the bridge-induced treatment \(\pi_h(X)\). For the local approximation
argument, we take \(K\) to be a nonnegative symmetric kernel supported on
\([-1,1]\), satisfying
\begin{equation}
\int_{-1}^{1} K(u)\,du = 1,
\qquad
\int_{-1}^{1} uK(u)\,du = 0,
\qquad
\mu_2(K):=\int_{-1}^{1} u^2K(u)\,du < \infty .
\end{equation}
We write
\begin{equation}
K_\tau(u):=\tau^{-1}K(u/\tau).
\end{equation}
In the implementation we use the truncated normalized Gaussian kernel
\begin{equation}
K_{\mathrm{TG}}(u)
=
\frac{\varphi(u)\mathbf 1\{|u|\le 1\}}
{\int_{-1}^{1}\varphi(v)\,dv},
\qquad
\varphi(u)
=
\frac{1}{\sqrt{2\pi}}\exp\!\left(-\frac{u^2}{2}\right).
\end{equation}
This kernel is nonnegative, symmetric, integrates to one, has zero first
moment, and has finite second moment.

The following proposition makes the kernel approximation explicit. The interior
condition \(\pi_h(X)\in[c_-+\kappa,c_+-\kappa]\) ensures that the support of
\(K_\tau(\pi_h(X)-\cdot)\) remains inside the treatment interval when
\(\tau<\kappa\). This avoids boundary bias and boundary-tail terms. Standard
boundary kernels or boundary corrections could be used instead, but we keep the
compact statement below for clarity.

\begin{proposition}[Kernel approximation bias]
\label{prop:kernel-bias-surface}
Assume there exists $\kappa>0$ such that
\begin{equation}
    \pi_h(X)\in[c_-+\kappa,c_+-\kappa]
    \qquad\text{a.s.}
\end{equation}
Assume moreover that, for almost every $x$, the map $t\mapsto G_h(t,x)$ is twice
continuously differentiable on
$[\pi_h(x)-\kappa,\pi_h(x)+\kappa]$, and that there exists an integrable envelope
$M_h(X)$ such that
\begin{equation}
    \sup_{|u|\le \kappa}
    \left|
    \partial_t^2 G_h\!\left(\pi_h(X)+u,X\right)
    \right|
    \le M_h(X)
    \qquad\text{a.s.}
\end{equation}
If $\tau<\kappa$ and $K$ is supported on $[-1,1]$, then
\begin{equation}
    \left|
    A_\tau(h)-\E\!\left[G_h\!\left(\pi_h(X),X\right)\right]
    \right|
    \le
    \frac{\mu_2(K)}{2}\,\tau^2\,\E[M_h(X)].
    \label{eq:kernel-bias-surface}
\end{equation}
Consequently,
\begin{equation}
    \E\!\left[G_h\!\left(\pi_h(X),X\right)\right]
    \le
    A_\tau(h)+C_h\tau^2,
    \qquad
    C_h:=\frac{\mu_2(K)}{2}\E[M_h(X)].
\end{equation}
\end{proposition}

\begin{proof}
By the definition of $A_\tau(h)$ and overlap,
\begin{equation}
    A_\tau(h)
    =
    \E_X\!\left[
    \int_{c_-}^{c_+}K_\tau(\pi_h(X)-t)G_h(t,X)\,dt
    \right].
\end{equation}
The support condition on $K$ and the requirement $\tau<\kappa$ imply that, for
almost every $X$, the integrand is nonzero only when
$t\in[\pi_h(X)-\tau,\pi_h(X)+\tau]\subset[c_-,c_+]$. Fix such an $x$, write
$\pi=\pi_h(x)$ and $g_x(t)=G_h(t,x)$, and change variables
$u=(\pi-t)/\tau$. Then
\begin{equation}
    \int K_\tau(\pi-t)g_x(t)\,dt
    =
    \int K(u)g_x(\pi-\tau u)\,du .
\end{equation}
For each $u$ in the support of $K$, Taylor's theorem gives
\begin{equation}
    g_x(\pi-\tau u)
    =
    g_x(\pi)
    -\tau u g_x'(\pi)
    +\frac{\tau^2u^2}{2}g_x''(\pi-\xi_{x,u}\tau u)
\end{equation}
for some $\xi_{x,u}\in(0,1)$. Since $K$ is symmetric and integrates to one,
$\int uK(u)\,du=0$, and therefore the first-order term vanishes after
integration. Hence
\begin{equation}
    \left|
    \int K(u)g_x(\pi-\tau u)\,du-g_x(\pi)
    \right|
    \le
    \frac{\tau^2}{2}\mu_2(K)
    \sup_{|v|\le\tau}|g_x''(\pi+v)|.
\end{equation}
Taking expectations and using the envelope condition yields
\eqref{eq:kernel-bias-surface}.
\end{proof}

The next proposition controls the second term in the regret decomposition, the
surface error at the oracle treatment $a^\star(X)$. Because $a^\star(X)$ is not
available to the learner, we use the global average surface loss $B(h)$ as a
fallback. The resulting bound is intentionally coarse: it is a stabilization
term rather than a sharp localization result.

\begin{proposition}[Global average surface bound]
\label{prop:global-surface-bound}
Assume that, for every $x$, the map $t\mapsto G_h(t,x)$ is $L$-Lipschitz on
$[c_-,c_+]$. Then
\begin{equation}
    \E\!\left[G_h\!\left(a^\star(X),X\right)\right]
    \le
    B(h)+\frac{L}{2}(c_+-c_-).
    \label{eq:prop44-surface}
\end{equation}
\end{proposition}

\begin{proof}
Fix $x$ and write $t^\star=a^\star(x)$. By the Lipschitz condition,
\begin{equation}
    G_h(t^\star,x)
    \le
    G_h(t,x)+L|t-t^\star|,
    \qquad t\in[c_-,c_+].
\end{equation}
Equivalently,
\begin{equation}
    G_h(t,x)
    \ge
    G_h(t^\star,x)-L|t-t^\star|.
\end{equation}
Integrating over $[c_-,c_+]$ and dividing by the interval length gives
\begin{equation}
    \frac{1}{c_+-c_-}\int_{c_-}^{c_+}G_h(t,x)\,dt
    \ge
    G_h(t^\star,x)
    -
    \frac{L}{c_+-c_-}\int_{c_-}^{c_+}|t-t^\star|\,dt .
\end{equation}
For any $t^\star\in[c_-,c_+]$,
\begin{equation}
    \frac{1}{c_+-c_-}\int_{c_-}^{c_+}|t-t^\star|\,dt
    \le
    \frac{c_+-c_-}{2}.
\end{equation}
Rearranging and taking expectation over $X$ proves the claim.
\end{proof}

\subsubsection{Boundary-adaptive kernel weighting}
\label{app:boundary-adaptive-kernel}

The interiority condition in Proposition~\ref{prop:kernel-bias-surface} avoids truncating the kernel at the treatment boundaries. A support-normalized extension allows centers at the boundary. For $\pi\in[c_-,c_+]$, define the kernel mass within the feasible interval by
\begin{equation}
    C_\tau(\pi)
    :=
    \int_{c_-}^{c_+}\frac{1}{\tau}
    K\!\left(\frac{\pi-t}{\tau}\right)\,dt.
\end{equation}
For centers satisfying $C_\tau(\pi_h(x))>0$, the boundary-adaptive weights are
\begin{equation}
    \omega_h^{\mathrm{bd}}(a,x)
    :=
    \frac{
    1+\lambda K\!\left((\pi_h(x)-a)/\tau\right)/C_\tau(\pi_h(x))
    }{(c_+-c_-)p(a\mid x)}.
\end{equation}
When the kernel support lies inside the feasible treatment interval, $C_\tau(\pi_h(x))=1$, and these weights coincide with the original scheme. Near a boundary, normalization compensates for kernel mass outside the interval. The normalized kernel integrates to one on $[c_-,c_+]$, reproducing constant surface errors exactly and removing the zeroth-order truncation bias.

If, for almost every $x$, the map $t\mapsto G_h(t,x)$ is $L_{G,h}$-Lipschitz on $[c_-,c_+]$, then, writing $\pi=\pi_h(x)$,
\begin{equation}
\begin{aligned}
    &\left|
    \frac{\int_{c_-}^{c_+}K_\tau(\pi-t)G_h(t,x)\,dt}{C_\tau(\pi)}
    -G_h(\pi,x)
    \right| \\
    &\qquad\le
    \frac{L_{G,h}}{C_\tau(\pi)}
    \int_{c_-}^{c_+}K_\tau(\pi-t)|t-\pi|\,dt
    \le L_{G,h}\tau.
\end{aligned}
\end{equation}
The last inequality uses $|t-\pi|\le\tau$ wherever the kernel is nonzero. Averaging over $X$ therefore gives an $O(\tau)$ approximation remainder. The normalized kernel is generally asymmetric near a boundary, so the first-order term need not vanish, in contrast to the $O(\tau^2)$ interior approximation under the stronger smoothness conditions of Proposition~\ref{prop:kernel-bias-surface}. The weights remain nonnegative, retain the global component, and can be incorporated into the existing solver-specific weighting implementations.

\subsection{Proof of the weighted surface surrogate}
\label{app:weighted-surface-surrogate-proof}

We now combine the regret decomposition with the local kernel approximation and
the global average surface bound.

\begin{proof}[Proof of Proposition~\ref{prop:weighted-surrogate-surface}]
Lemma~\ref{lemma:regret-decomp-surface} gives
\begin{equation}
    \Reg(h)
    \le
    \sqrt{\E\!\left[G_h\!\left(\pi_h(X),X\right)\right]}
    +
    \sqrt{\E\!\left[G_h\!\left(a^\star(X),X\right)\right]}.
\end{equation}
Applying Propositions~\ref{prop:kernel-bias-surface}
and~\ref{prop:global-surface-bound} to the two terms yields
\begin{equation}
    \Reg(h)
    \le
    \sqrt{A_\tau(h)+C_h\tau^2}
    +
    \sqrt{B(h)+\frac{L}{2}(c_+-c_-)}.
\end{equation}
Using $\sqrt{u} + \sqrt{v}\le \sqrt{2(u+v)}$ for $u,v\ge0$,
\begin{equation}
    \Reg(h)
    \le
    \sqrt{2\!\left(A_\tau(h)+B(h)+C_h\tau^2+\frac{L}{2}(c_+-c_-)\right)}.
\end{equation}
Since $c_\gamma=\max\{1,\gamma^{-1}\}$,
\begin{equation}
    A_\tau(h)+B(h)
    \le
    c_\gamma\{\gamma A_\tau(h)+B(h)\}.
\end{equation}
Thus
\begin{equation}
    \Reg(h)
    \le
    \sqrt{2c_\gamma\{\gamma A_\tau(h)+B(h)\}
    +2\!\left(C_h\tau^2+\frac{L}{2}(c_+-c_-)\right)}.
\end{equation}
Finally applying $\sqrt{u+v}\le\sqrt u+\sqrt v$ gives
\begin{equation}
    \Reg(h)
    \le
    \sqrt{2c_\gamma}\sqrt{\gamma A_\tau(h)+B(h)}
    +
    \sqrt{2\!\left(C_h\tau^2+\frac{L}{2}(c_+-c_-)\right)},
\end{equation}
which is the desired bound.
\end{proof}

\subsection{Bridge-to-surface control}
\label{app:bridge-to-surface-proof}

The preceding arguments control regret by a weighted surface loss involving
$m_h-m_0$. This final step shows that, under a weighted ill-posedness condition,
the same weighted surface loss is controlled by the corresponding weighted
proximal bridge moment loss.

\begin{proof}[Proof of Proposition~\ref{prop:weighted_bridge_to_surface}]
For every $(a,x)$, define the surface error induced by the bridge error as
\begin{equation}
    \bar\Delta_h(a,x)
    :=
    m_h(a,x)-m_0(a,x)
    =
    \int \Delta_h(a,w,x)\,dP_{W\mid X=x}(w).
\end{equation}
Thus
\begin{equation}
    \mathcal L_{\mathrm{surf},\omega}(h)
    =
    \E\!\left[
    \omega(A,X)\bar\Delta_h(A,X)^2
    \right].
\end{equation}
By Jensen's inequality, for every $(a,x)$,
\begin{equation}
    \bar\Delta_h(a,x)^2
    \le
    \int \Delta_h(a,w,x)^2\,dP_{W\mid X=x}(w).
\end{equation}
It follows that
\begin{align}
    \mathcal L_{\mathrm{surf},\omega}(h)
    &\le
    \E_X\!\left[
    \int \omega(a,X)
    \left\{
        \int \Delta_h(a,w,X)^2\,dP_{W\mid X}(w)
    \right\}
    dP_{A\mid X}(a\mid X)
    \right] \\
    &=
    \int \omega(a,x)\Delta_h(a,w,x)^2\,d\mu(w,a,x)
    =
    \|\Delta_h\|_{L_2(\omega\,d\mu)}^2 .
\end{align}
By the bounded density-ratio assumption $d\mu/d\nu\le C_\rho$,
\begin{equation}
    \|\Delta_h\|_{L_2(\omega\,d\mu)}^2
    \le
    C_\rho\|\Delta_h\|_{L_2(\omega\,d\nu)}^2 .
\end{equation}
Since $h_0$ satisfies the outcome bridge restriction,
\begin{equation}
    \E[Y-h(A,W,X)\mid A,Z,X]
    =
    -\E[h(A,W,X)-h_0(A,W,X)\mid A,Z,X]
    =
    -T\Delta_h(A,Z,X),
\end{equation}
and hence
\begin{equation}
    \mathcal L_{\mathrm{br},\omega}(h)
    =
    \|T\Delta_h\|_{L_2(\omega\,dP_{A,Z,X})}^2 .
\end{equation}
The definition of $\chi_\omega$ gives
\begin{equation}
    \|\Delta_h\|_{L_2(\omega\,d\nu)}
    \le
    \chi_\omega\|T\Delta_h\|_{L_2(\omega\,dP_{A,Z,X})}.
\end{equation}
Combining the last four displays proves
\begin{equation}
    \mathcal L_{\mathrm{surf},\omega}(h)
    \le
    C_\rho\chi_\omega^2\mathcal L_{\mathrm{br},\omega}(h).
\end{equation}
\end{proof}

\begin{proof}[Proof of Theorem~\ref{thm:self_weighted_bridge_to_regret}]
By definition of $\omega_h$,
\begin{equation}
    \mathcal L^{\mathrm{surf}}_{\tau,\lambda}(h)
    =
    \mathcal L_{\mathrm{surf},\omega_h}(h).
\end{equation}
Applying Proposition~\ref{prop:weighted_bridge_to_surface} with
$\omega=\omega_h$ yields
\begin{equation}
    \mathcal L^{\mathrm{surf}}_{\tau,\lambda}(h)
    \le
    C_\rho\chi_{\omega_h}^2\mathcal L_{\mathrm{br},\omega_h}(h).
\end{equation}
Substituting this inequality into Proposition~\ref{prop:weighted-surrogate-surface}
proves the theorem.
\end{proof}

\subsection{The weighted ill-posedness condition}
\label{app:weighted-illposedness-sufficient}

Proposition~\ref{prop:weighted_bridge_to_surface} assumes the finite weighted
ill-posedness constant
\begin{equation}
\chi_\omega
:=
\sup_{g\in\mathcal H:\,
\|T\Delta_g\|_{L_2(\omega\,dP_{A,Z,X})}>0}
\frac{
\|\Delta_g\|_{L_2(\omega\,d\nu)}
}{
\|T\Delta_g\|_{L_2(\omega\,dP_{A,Z,X})}
}
<\infty .
\end{equation}
This condition is the direct stability assumption needed to control the
weighted surface target by the weighted bridge loss. It is also implied by
the usual unweighted restricted ill-posedness condition whenever the weights
are bounded away from zero and infinity.

\begin{proposition}[Bounded weights imply weighted stability]
\label{prop:bounded_weights_weighted_stability}
Assume that the global restricted ill-posedness constant
\begin{equation}
\chi_{\mathrm{glob}}
:=
\sup_{g\in\mathcal H:\,
\|T\Delta_g\|_{L_2(dP_{A,Z,X})}>0}
\frac{
\|\Delta_g\|_{L_2(d\nu)}
}{
\|T\Delta_g\|_{L_2(dP_{A,Z,X})}
}
\end{equation}
is finite. Suppose further that, for some constants
$0<\omega_{\min}\le \omega_{\max}<\infty$,
\begin{equation}
\omega_{\min}\le \omega(a,x)\le \omega_{\max}
\end{equation}
almost surely under the relevant laws. Then
\begin{equation}
\chi_\omega
\le
\sqrt{\frac{\omega_{\max}}{\omega_{\min}}}\,
\chi_{\mathrm{glob}}.
\end{equation}
Equivalently,
\begin{equation}
\chi_\omega^2
\le
\frac{\omega_{\max}}{\omega_{\min}}\,
\chi_{\mathrm{glob}}^2 .
\end{equation}
\end{proposition}

\begin{proof}
For every $g\in\mathcal H$,
\begin{equation}
\|\Delta_g\|_{L_2(\omega\,d\nu)}
\le
\sqrt{\omega_{\max}}\,
\|\Delta_g\|_{L_2(d\nu)}.
\end{equation}
Similarly,
\begin{equation}
\|T\Delta_g\|_{L_2(\omega\,dP_{A,Z,X})}
\ge
\sqrt{\omega_{\min}}\,
\|T\Delta_g\|_{L_2(dP_{A,Z,X})}.
\end{equation}
Therefore, whenever the denominator is nonzero,
\begin{equation}
\frac{
\|\Delta_g\|_{L_2(\omega\,d\nu)}
}{
\|T\Delta_g\|_{L_2(\omega\,dP_{A,Z,X})}
}
\le
\sqrt{\frac{\omega_{\max}}{\omega_{\min}}}
\frac{
\|\Delta_g\|_{L_2(d\nu)}
}{
\|T\Delta_g\|_{L_2(dP_{A,Z,X})}
}.
\end{equation}
Taking the supremum over $g\in\mathcal H$ gives the result.
\end{proof}

Proposition~\ref{prop:bounded_weights_weighted_stability} shows that the
weighted stability condition is not a qualitatively new type of assumption
when bounded or clipped weights are used. It follows from the standard global
restricted ill-posedness condition by equivalence of weighted and unweighted
$L_2$ norms.

\paragraph{Finite-dimensional characterization.}
Let $S_{\mathrm{br}}=(A,W,X)$ and $V_{\mathrm{br}}=(A,Z,X)$. Consider a correctly specified fixed-representation linear bridge class
\begin{equation}
    h_\beta(S_{\mathrm{br}})=\beta^\top\phi(S_{\mathrm{br}}),
    \qquad \beta\in\R^d,
    \qquad h_0=h_{\beta_0}.
\end{equation}
Then $\Delta_h(S_{\mathrm{br}})=b^\top\phi(S_{\mathrm{br}})$, where $b=\beta-\beta_0$. Define
\begin{equation}
    \bar\phi(V_{\mathrm{br}})
    :=\E[\phi(S_{\mathrm{br}})\mid V_{\mathrm{br}}],
\end{equation}
and the weighted matrices
\begin{equation}
\begin{aligned}
    G_\omega
    &:=\E\!\left[\omega(A,X)\phi(S_{\mathrm{br}})\phi(S_{\mathrm{br}})^\top\right],\\
    Q_\omega
    &:=\E\!\left[\omega(A,X)\bar\phi(V_{\mathrm{br}})\bar\phi(V_{\mathrm{br}})^\top\right].
\end{aligned}
\end{equation}
Assume these weighted second moments are finite. On the effective feature span, after discarding directions with $b^\top G_\omega b=0$, $G_\omega$ is positive definite. If $Q_\omega$ is also positive definite on this span, the restricted constant satisfies
\begin{equation}
    \chi_\omega^2
    =
    \sup_{b:\,b^\top Q_\omega b>0}
    \frac{b^\top G_\omega b}{b^\top Q_\omega b}
    =
    \left[\lambda_{\min}\!\left(
    G_\omega^{-1/2}Q_\omega G_\omega^{-1/2}
    \right)\right]^{-1},
\end{equation}
where the inverses and eigenvalue are evaluated on that span. Thus, the constant is finite when conditional expectation does not annihilate any nonzero bridge-feature direction in the considered bridge-error class. This is a restricted finite-dimensional stability condition, not a claim that the unrestricted proximal inverse problem is well posed. The construction includes fixed-representation second-stage linear bridge estimation, such as DFPV conditional on its learned features.

\paragraph{Scalar Gaussian example.}
Suppress $A$ and $X$, and let $W$ and $Z$ be standardized jointly Gaussian variables with correlation $\rho\ne0$. For the feature $\phi=W$, we have $\E[W\mid Z]=\rho Z$. Under a constant positive weight $\omega$,
\begin{equation}
    G_\omega=\omega,
    \qquad Q_\omega=\omega\rho^2,
    \qquad \chi_\omega=\frac{1}{|\rho|}<\infty.
\end{equation}
The constant increases as proxy informativeness decreases, approaching infinity as $|\rho|\to0$.

\subsection{Proof of Population Target Preservation}
\label{app:proof-population-target-preservation}

Let
\begin{equation}
    S:=(A,W,X), \qquad V:=(A,Z,X), \qquad r_h:=Y-h(S).
\end{equation}
Then
\begin{equation}
    \mathcal L_{\mathrm{br},\omega}(h)
    =
    \E\!\left[\omega(A,X)\{\E[r_h\mid V]\}^2\right].
\end{equation}
Since $h_0$ satisfies the outcome bridge restriction,
\begin{equation}
    \E[r_{h_0}\mid V]=0
    \quad\text{a.s.}
\end{equation}
and hence $\mathcal L_{\mathrm{br},\omega}(h_0)=0$. Because
$\mathcal L_{\mathrm{br},\omega}(h)\ge 0$ for all $h\in\cH$, this shows that
$h_0$ is a population minimizer.

Now fix any $h\in\cH$. If $\mathcal L_{\mathrm{br},\omega}(h)=0$, then the
nonnegative random variable
\begin{equation}
    \omega(A,X)\{\E[r_h\mid V]\}^2
\end{equation}
has expectation zero, and therefore equals zero almost surely. Since
$\omega(A,X)>0$ almost surely, it follows that
\begin{equation}
    \E[r_h\mid V]=0
    \quad\text{a.s.}
\end{equation}
Equivalently,
\begin{equation}
    \E\{Y-h(A,W,X)\mid A,Z,X\}=0
    \quad\text{a.s.}
\end{equation}
The converse implication is immediate: if the conditional moment restriction
holds, then $\mathcal L_{\mathrm{br},\omega}(h)=0$. Therefore,
\begin{equation}
    \mathcal L_{\mathrm{br},\omega}(h)=0
    \quad\Longleftrightarrow\quad
    \E\{Y-h(A,W,X)\mid A,Z,X\}=0
    \quad\text{a.s.}
\end{equation}
The same equivalence holds for the unweighted bridge risk, corresponding to
$\omega\equiv 1$. Hence the weighted and unweighted population bridge risks have
the same zero-risk solution.

\subsection{Fixed-weight statistical transfer to surface risk}
\label{app:fixed-weight-statistics}

Proposition~\ref{prop:weighted_bridge_to_surface} is a population inequality,
but it also transfers statistical guarantees for a weighted proximal bridge
estimator to its induced causal response surface. The next result makes this
implication explicit without imposing a particular bridge-learning algorithm.

\begin{corollary}[Oracle transfer from bridge risk to surface risk]
\label{cor:fixed-weight-surface-risk}
Let $\omega:\cA\times\cX\to(0,\infty)$ be a fixed measurable weight
satisfying the conditions of Proposition~\ref{prop:weighted_bridge_to_surface},
with $C_\rho<\infty$ and $\chi_\omega<\infty$. Suppose an estimator
$\widehat h_n\in\cH$ satisfies, with probability at least $1-\delta$,
\begin{equation}
    \mathcal L_{\mathrm{br},\omega}(\widehat h_n)
    \le
    \inf_{h\in\cH}\mathcal L_{\mathrm{br},\omega}(h)
    +r_n(\delta),
    \label{eq:fixed-weight-bridge-oracle}
\end{equation}
where $r_n(\delta)\ge0$. Then, on the same event,
\begin{equation}
    \mathcal L_{\mathrm{surf},\omega}(\widehat h_n)
    \le
    C_\rho\chi_\omega^2
    \left\{
        \inf_{h\in\cH}\mathcal L_{\mathrm{br},\omega}(h)
        +r_n(\delta)
    \right\}.
    \label{eq:fixed-weight-surface-oracle}
\end{equation}
In particular, if $h_0\in\cH$, the approximation term vanishes. If
$\mathcal L_{\mathrm{br},\omega}(\widehat h_n)=o_p(1)$ and
$C_\rho\chi_\omega^2<\infty$ is fixed, then
$\mathcal L_{\mathrm{surf},\omega}(\widehat h_n)=o_p(1)$.
\end{corollary}

\begin{proof}
On the event in~\eqref{eq:fixed-weight-bridge-oracle}, apply
Proposition~\ref{prop:weighted_bridge_to_surface} to $\widehat h_n$:
\begin{equation}
    \mathcal L_{\mathrm{surf},\omega}(\widehat h_n)
    \le C_\rho\chi_\omega^2
        \mathcal L_{\mathrm{br},\omega}(\widehat h_n)
    \le C_\rho\chi_\omega^2
        \left\{
            \inf_{h\in\cH}\mathcal L_{\mathrm{br},\omega}(h)
            +r_n(\delta)
        \right\}.
\end{equation}
When $h_0\in\cH$, the outcome bridge restriction implies
$\mathcal L_{\mathrm{br},\omega}(h_0)=0$. The consistency statement follows
from the same inequality.
\end{proof}

The approximation term in~\eqref{eq:fixed-weight-surface-oracle} need not
vanish under bridge-class misspecification. Thus the corollary transfers a
solver-specific bridge-risk guarantee; it does not assert that every proximal
solver is consistent under arbitrary misspecification. Nor does it require
recovering the bridge itself in an unweighted norm.

For completeness, when the fixed weight is bounded as
$0<\omega_{\min}\le\omega(A,X)\le\omega_{\max}<\infty$, the nonnegativity
of the squared conditional moment gives, for every $h\in\cH$,
\begin{equation}
    \omega_{\min}\mathcal L_{\mathrm{br},1}(h)
    \le \mathcal L_{\mathrm{br},\omega}(h)
    \le \omega_{\max}\mathcal L_{\mathrm{br},1}(h),
    \label{eq:bounded-weight-bridge-risk-comparison}
\end{equation}
where $\mathcal L_{\mathrm{br},1}$ denotes the unweighted bridge risk.
Together with Proposition~\ref{prop:bounded_weights_weighted_stability},
this shows that bounded positive reweighting preserves the population
zero-risk target and finite restricted stability, up to weight-dependent
constants. For fixed bounded weights, standard empirical-process or
U/V-process arguments can also retain the sample-size order of a base
learner's bound under the corresponding envelope and complexity conditions.
Such statistical bounds remain solver-specific: the actual MMR and two-stage
empirical objectives require their own concentration and calibration
arguments, rather than being automatically identified with
$\mathcal L_{\mathrm{br},\omega}$.

\subsection{A finite-round oracle inequality for adaptive weights}
\label{app:adaptive-weight-statistics}

The preceding corollary applies to a fixed weight. To account for weights
selected from the data, we consider a deterministic class containing the
weights that the learning procedure may use. This leads to a conditional
finite-round guarantee for the solver-specific objective, without assuming
that the chosen weights are independent of the training sample.

For a given proximal solver, let
$\mathcal R_\omega(h)$ denote the population
counterpart of the empirical fitting criterion
$\widehat{\mathcal R}_{n,\omega}(h)$
introduced in Section~\ref{sec:algorithms}.
In general, $\mathcal R_\omega$ need not
coincide with the ideal conditional-moment loss
$\mathcal L_{\mathrm{br},\omega}$:
PMMR and NMMR use kernelized moment criteria,
whereas KPV and DFPV use two-stage fitting objectives.
Under suitable conditions, these criteria target the
same population bridge equation, but quantitative
control of $\mathcal L_{\mathrm{br},\omega}$
requires additional solver-specific calibration
conditions, stated below.

\begin{proposition}[Finite-round adaptive oracle inequality]
\label{prop:finite-round-adaptive-oracle}
Fix a finite number of rounds $S$ and let $\Omega_S$ be a deterministic
class of positive measurable weights containing all weights
$\widehat\omega^{(s)}$ reachable in the first $S+1$ bridge fits,
$s=0,\ldots,S$. Suppose that, with probability at least $1-\delta$,
\begin{equation}
    \sup_{\substack{h\in\cH\\\omega\in\Omega_S}}
    \left|
        \widehat{\mathcal R}_{n,\omega}(h)
        -\mathcal R_\omega(h)
    \right|
    \le \epsilon_n(\delta).
    \label{eq:adaptive-uniform-deviation}
\end{equation}
Assume also that, for each round $s=0,\ldots,S$, the fitted bridge
$\widehat h^{(s)}\in\cH$ obeys
\begin{equation}
    \widehat{\mathcal R}_{n,\widehat\omega^{(s)}}
        (\widehat h^{(s)})
    \le
    \inf_{h\in\cH}
        \widehat{\mathcal R}_{n,\widehat\omega^{(s)}}(h)
    +\xi_s,
    \label{eq:adaptive-approximate-minimization}
\end{equation}
where $\xi_s\ge0$. Then, simultaneously for all $s=0,\ldots,S$,
\begin{equation}
    \mathcal R_{\widehat\omega^{(s)}}(\widehat h^{(s)})
    \le
    \inf_{h\in\cH}\mathcal R_{\widehat\omega^{(s)}}(h)
    +2\epsilon_n(\delta)+\xi_s
    \label{eq:adaptive-population-oracle}
\end{equation}
with probability at least $1-\delta$, provided
\eqref{eq:adaptive-approximate-minimization} holds on the same event.
\end{proposition}

\begin{proof}
Work on the event~\eqref{eq:adaptive-uniform-deviation}. Since
$\widehat\omega^{(s)}\in\Omega_S$, the deviation bound holds at the
possibly data-dependent weight $\widehat\omega^{(s)}$. For any
$s=0,\ldots,S$,
\begin{align}
    \mathcal R_{\widehat\omega^{(s)}}(\widehat h^{(s)})
    &\le
    \widehat{\mathcal R}_{n,\widehat\omega^{(s)}}
        (\widehat h^{(s)})+\epsilon_n(\delta)\\
    &\le
    \inf_{h\in\cH}
        \widehat{\mathcal R}_{n,\widehat\omega^{(s)}}(h)
        +\xi_s+\epsilon_n(\delta)\\
    &\le
    \inf_{h\in\cH}\mathcal R_{\widehat\omega^{(s)}}(h)
        +2\epsilon_n(\delta)+\xi_s.
\end{align}
The uniform event applies to all weights in $\Omega_S$, so the conclusion
holds simultaneously over the finite rounds without requiring a separate
independence assumption for each selected weight.
\end{proof}

The inequality~\eqref{eq:adaptive-population-oracle} controls the
solver-specific population fitting criterion. A corresponding bridge-risk
and surface-risk guarantee follows if the solver criterion is calibrated
to the conditional moment restriction. Specifically, suppose there exist
$C_{\mathrm{cal}}<\infty$ and nonnegative terms $b_\omega$ such that,
uniformly over $h\in\cH$ and $\omega\in\Omega_S$,
\begin{equation}
    \mathcal L_{\mathrm{br},\omega}(h)
    \le
    b_\omega+C_{\mathrm{cal}}
    \left\{
        \mathcal R_\omega(h)
        -\inf_{g\in\cH}\mathcal R_\omega(g)
    \right\}.
    \label{eq:solver-criterion-calibration}
\end{equation}
Assume also that the conditions of
Proposition~\ref{prop:weighted_bridge_to_surface} hold for every
$\omega\in\Omega_S$. Combining~\eqref{eq:adaptive-population-oracle},
\eqref{eq:solver-criterion-calibration}, and
Proposition~\ref{prop:weighted_bridge_to_surface} yields, on the same event,
\begin{equation}
    \mathcal L_{\mathrm{surf},\widehat\omega^{(s)}}
        (\widehat h^{(s)})
    \le
    C_\rho\chi_{\widehat\omega^{(s)}}^2
    \left\{
        b_{\widehat\omega^{(s)}}
        +C_{\mathrm{cal}}
            \bigl(2\epsilon_n(\delta)+\xi_s\bigr)
    \right\},
    \qquad s=0,\ldots,S.
    \label{eq:adaptive-surface-risk-transfer}
\end{equation}
The terms $b_\omega$ allow for bridge-class approximation and imperfect
calibration of the solver objective. For the ideal choice
$\mathcal R_\omega=\mathcal L_{\mathrm{br},\omega}$, one may take
$C_{\mathrm{cal}}=1$ and
$b_\omega=\inf_{h\in\cH}\mathcal L_{\mathrm{br},\omega}(h)$.
Under correct specification this term is zero. If, in addition,
$\sup_{\omega\in\Omega_S}\chi_\omega\le\bar\chi<\infty$,
$\sup_{\omega\in\Omega_S}b_\omega=o(1)$, and
$\epsilon_n(\delta)+\max_{0\le s\le S}\xi_s=o(1)$ along a sequence
of high-probability events with probability tending to one, then the
weighted surface risks in~\eqref{eq:adaptive-surface-risk-transfer}
converge to zero uniformly over the finite rounds. Uniform boundedness
$0<\omega_{\min}\le\omega\le\omega_{\max}<\infty$ over $\Omega_S$
and Proposition~\ref{prop:bounded_weights_weighted_stability}
provide one sufficient route to a uniform bound on $\chi_\omega$.

The uniform deviation condition~\eqref{eq:adaptive-uniform-deviation}
accounts for the complexity of the weight class, including its dependence
on pilot policies and propensity-density estimates. It is therefore
stronger than a deviation bound for a single fixed weight. Its validity,
together with~\eqref{eq:solver-criterion-calibration}, must be checked for
the chosen solver and function classes; no universal rate or improvement
over an unweighted estimator is implied by the finite-round inequality.

\section{Practical Implementation Template for Policy-Targeted Proximal Solvers}
\label{app:proximal-solvers}
\label{app:mmr}

This appendix collects the implementation details for the proximal bridge solvers used in the experiments.
The goal is to separate the common policy-targeted weighting template from the solver-specific algebra.
Throughout this appendix, write
\begin{equation}
    S := (A,W,X),
    \qquad
    V := (A,Z,X),
\end{equation}
and, for a candidate outcome bridge $h$, define the residual
\begin{equation}
    r_h := Y-h(S).
\end{equation}
The outcome bridge equation is the conditional moment restriction
\begin{equation}
    \E[r_h\mid V]=0.
\end{equation}
All solvers below estimate this same bridge equation. They differ only in how they approximate or penalize the conditional moment restriction.

\subsection{Policy-targeted weights}
\label{app:solver_template_weights}

At iteration $s$, the bridge fit is projected into a response-surface estimate
using bridge pseudo-outcomes over a treatment grid
$\mathcal G=\{a_1,\ldots,a_M\}\subset\mathcal A$. To construct the localization
weights, these pseudo-outcomes are formed by cross-fitting. For each fold $I_k$
in a partition of $\{1,\ldots,n\}$, the bridge $\widehat h_{-k}^{(s)}$ is trained
on the observations outside $I_k$ using the current weights and evaluated on the
held-out observations, giving
\begin{equation}
    \widetilde H_{im}^{(s)}
    =
    \widehat h_{-k}^{(s)}(a_m,W_i,X_i),
    \qquad i\in I_k,\quad m=1,\ldots,M .
\end{equation}
We pool these out-of-fold pseudo-outcomes across folds and regress them on
$(a_m,X_i)$ to obtain a single response-surface estimate $\widehat m^{(s)}$.
The observation-specific pseudo-optimal treatment used for weighting is
\begin{equation}
    \widehat\pi_i^{(s)}
    \in
    \argmax_{a\in\mathcal G}\widehat m^{(s)}(a,X_i),
\end{equation}
or the corresponding maximizer over $\mathcal A$ when continuous optimization is feasible.

The next bridge fit uses the fixed weights
\begin{equation}
    \widehat\omega_i^{(s+1)}
    =
    \frac{
        1+
        \lambda\,K\!\left((\widehat\pi_i^{(s)}-A_i)/\tau\right)
    }{
        (c_+-c_-)\widehat p(A_i\mid X_i)
    },
    \label{eq:appendix_policy_weight}
\end{equation}
where $\widehat p(a\mid x)$ is an estimate of the generalized propensity score,
i.e., the conditional treatment density for the continuous treatment $A$. In our
experiments, this density is estimated using a conditional normalizing-flow
model. The function $K$ is the localization kernel, $\tau$ is the bandwidth, and
$\lambda\ge 0$ controls the strength of localization. The constant $1$ in the
numerator is the global stabilizing component.

After the final weighting iteration, the solver is refit on the full sample
using the last cross-fitted weights.

\subsection{Two ways in which the weights enter}
\label{app:solver_adapter_table}

The same policy-targeted weights are used for every solver, but they enter the numerical objective in two different ways. For maximum-moment solvers, the weights multiply pairwise residual products. For two-stage regression solvers, the weights are applied in the second-stage bridge regression.

Table~\ref{tab:solver_weight_adapters} summarizes the solver-specific implementations.

\begin{table}[h]
\centering
\caption{Solver-specific implementation of the policy-targeted bridge weights.}
\label{tab:solver_weight_adapters}
\begin{tabular}{lll}
\toprule
Solver & Baseline object & Policy-targeted replacement \\
\midrule
PMMR & $r^\top K_V r$ & $r^\top D_\omega^{1/2}K_VD_\omega^{1/2}r$ \\
NMMR & $r(\theta)^\top K_V r(\theta)$ & $r(\theta)^\top D_\omega^{1/2}K_VD_\omega^{1/2}r(\theta)$ \\
KPV & second-stage kernel ridge & weighted second-stage ridge with $D_\omega$ \\
DFPV & second-stage feature ridge & weighted second-stage ridge with $D_\omega$ \\
\bottomrule
\end{tabular}
\end{table}

Here $K_V$ is the moment-kernel Gram matrix on $V=(A,Z,X)$ and
\begin{equation}
    D_\omega := \operatorname{diag}(\widehat\omega_1,\ldots,\widehat\omega_n).
\end{equation}
For PMMR and NMMR, the algebra is equivalent to replacing each residual by $\sqrt{\widehat\omega_i}r_i$. For KPV and DFPV, the actual sample weight in the second-stage weighted regression is $\widehat\omega_i$, not $\sqrt{\widehat\omega_i}$.

\subsection{Why the weighted objective still targets the bridge equation}
\label{app:mmr_weighted_bridge_equation}

The policy-targeted objective does not change the population bridge equation when the weights are strictly positive. It changes the geometry in which violations of the bridge equation are penalized. Indeed, for any positive measurable weight $\omega(V)$,
\begin{equation}
    \omega(V)\{\E[r_h\mid V]\}^2=0
    \quad\Longleftrightarrow\quad
    \E[r_h\mid V]=0.
\end{equation}
Equivalently,
\begin{equation}
    \sqrt{\omega(V)}\,\E[r_h\mid V]
    =
    \E[\sqrt{\omega(V)}\,r_h\mid V],
\end{equation}
because $\omega(V)$ is measurable with respect to $V$. Thus, in a well-specified population problem, the weighted and unweighted conditional moment restrictions have the same bridge solutions. The weights matter in finite samples, under regularization, and under misspecification, where they determine which moment violations receive more emphasis.

\section{MMR-Based Solvers: PMMR and NMMR}
\label{app:pmmr}

PMMR and NMMR both estimate the outcome bridge by penalizing a kernel maximum moment restriction. The difference is the bridge class: PMMR uses an RKHS bridge, whereas NMMR uses a neural bridge. This section first derives the common weighted MMR criterion and then gives the PMMR and NMMR specializations.

\subsection{Weighted maximum moment restriction}
\label{app:mmr_basic}

Let $k$ be a positive definite kernel on $\cV$ and let $\mathcal H_k$ be its RKHS. The unweighted kernel maximum moment criterion is
\begin{equation}
    \mathcal R_k(h)
    :=
    \left\|
        \E[r_h k(V,\cdot)]
    \right\|_{\mathcal H_k}^2.
\end{equation}
For an independent copy $(Y',S',V')$, with $r'_h=Y'-h(S')$, this can be written as
\begin{equation}
    \mathcal R_k(h)
    =
    \E\!\big[r_h r'_h k(V,V')\big].
\end{equation}
Given observations $\{(y_i,s_i,v_i)\}_{i=1}^n$, define
\begin{equation}
    r_i(h):=y_i-h(s_i),
    \qquad
    r(h):=(r_1(h),\ldots,r_n(h))^\top,
\end{equation}
and let
\begin{equation}
    (K_V)_{ij}:=k(v_i,v_j).
\end{equation}
The empirical V-statistic is
\begin{equation}
    \widehat{\mathcal R}_{k,V,n}(h)
    =
    \frac{1}{n^2}r(h)^\top K_V r(h).
\end{equation}

For policy-targeted learning, define
\begin{equation}
    K_{V,\widehat\omega}
    :=
    D_{\widehat\omega}^{1/2}K_VD_{\widehat\omega}^{1/2},
    \qquad
    D_{\widehat\omega}^{1/2}
    :=
    \operatorname{diag}(\sqrt{\widehat\omega_1},\ldots,\sqrt{\widehat\omega_n}).
\end{equation}
The weighted empirical MMR criterion is
\begin{equation}
    \widehat{\mathcal R}_{k,V,n,\widehat\omega}(h)
    =
    \frac{1}{n^2}
    r(h)^\top K_{V,\widehat\omega}r(h).
    \label{eq:appendix_weighted_mmr}
\end{equation}
Equivalently, \eqref{eq:appendix_weighted_mmr} is the ordinary MMR criterion applied to the transformed residuals
\begin{equation}
    \widetilde r_i(h)
    :=
    \sqrt{\widehat\omega_i}\{y_i-h(s_i)\}.
\end{equation}

\subsection{Policy-targeted PMMR}
\label{app:pmmr_weighted}

Let $\ell$ be a positive definite kernel on $\cS$ with RKHS $\mathcal H_\ell$. PMMR restricts $h\in\mathcal H_\ell$ and uses the representer form
\begin{equation}
    h_\alpha(\cdot)
    =
    \sum_{j=1}^n \alpha_j\ell(s_j,\cdot).
\end{equation}
Let
\begin{equation}
    (L_S)_{ij}:=\ell(s_i,s_j),
    \qquad
    y:=(y_1,\ldots,y_n)^\top.
\end{equation}
Then the residual vector is
\begin{equation}
    r(\alpha)=y-L_S\alpha.
\end{equation}
The unweighted PMMR objective is
\begin{equation}
    J_{\mathrm{PMMR}}(\alpha)
    =
    (y-L_S\alpha)^\top K_V(y-L_S\alpha)
    +
    \eta\,\alpha^\top L_S\alpha,
\end{equation}
where the common factor $n^{-2}$ is absorbed into the regularization parameter. Its stabilized solution is
\begin{equation}
    \widehat\alpha
    =
    \left(
        L_SK_VL_S+
        \eta L_S+
        \delta I_n
    \right)^{-1}
    L_SK_Vy,
\end{equation}
with numerical jitter $\delta>0$.

The policy-targeted PMMR objective is obtained by replacing $K_V$ with the weighted moment Gram matrix $K_{V,\widehat\omega}$:
\begin{equation}
    J_{\mathrm{PMMR},\widehat\omega}(\alpha)
    =
    (y-L_S\alpha)^\top
    K_{V,\widehat\omega}
    (y-L_S\alpha)
    +
    \eta\,\alpha^\top L_S\alpha.
\end{equation}
Therefore
\begin{equation}
    \widehat\alpha_{\widehat\omega}
    =
    \left(
        L_SK_{V,\widehat\omega}L_S+
        \eta L_S+
        \delta I_n
    \right)^{-1}
    L_SK_{V,\widehat\omega}y.
    \label{eq:appendix_weighted_pmmr_solution}
\end{equation}
The fitted bridge is
\begin{equation}
    \widehat h_{\widehat\omega}(a,w,x)
    =
    \sum_{j=1}^n
    \widehat\alpha_{\widehat\omega,j}
    \ell\{s_j,(a,w,x)\}.
\end{equation}
Thus policy-targeted PMMR is implemented by the single replacement
\begin{equation}
    K_V
    \quad\leadsto\quad
    D_{\widehat\omega}^{1/2}K_VD_{\widehat\omega}^{1/2}.
\end{equation}

\subsection{Policy-targeted NMMR}
\label{app:nmmr}

NMMR uses the same conditional moment restriction, but represents the bridge by a neural network $h_\theta:\cS\to\R$. Define
\begin{equation}
    r_i(\theta):=y_i-h_\theta(s_i),
    \qquad
    r(\theta):=(r_1(\theta),\ldots,r_n(\theta))^\top.
\end{equation}
The unweighted V-statistic NMMR loss is
\begin{equation}
    \widehat{\mathcal R}_{k,V,n}(\theta)
    =
    \frac{1}{n^2}r(\theta)^\top K_Vr(\theta).
\end{equation}
The corresponding U-statistic version removes diagonal terms:
\begin{equation}
    \widehat{\mathcal R}_{k,U,n}(\theta)
    =
    \frac{1}{n(n-1)}
    r(\theta)^\top
    \{K_V-\operatorname{diag}(K_V)\}
    r(\theta).
\end{equation}

The policy-targeted NMMR loss replaces $r_i(\theta)$ by $\sqrt{\widehat\omega_i}r_i(\theta)$. The weighted V-statistic objective is
\begin{equation}
    \widehat{\mathcal R}_{k,V,n,\widehat\omega}(\theta)
    =
    \frac{1}{n^2}
    r(\theta)^\top
    D_{\widehat\omega}^{1/2}K_VD_{\widehat\omega}^{1/2}
    r(\theta),
    \label{eq:appendix_weighted_nmmr_v}
\end{equation}
and the weighted U-statistic objective is
\begin{equation}
    \widehat{\mathcal R}_{k,U,n,\widehat\omega}(\theta)
    =
    \frac{1}{n(n-1)}
    r(\theta)^\top
    D_{\widehat\omega}^{1/2}
    \{K_V-\operatorname{diag}(K_V)\}
    D_{\widehat\omega}^{1/2}
    r(\theta).
    \label{eq:appendix_weighted_nmmr_u}
\end{equation}
The estimator minimizes one of these losses plus the usual network regularization penalty:
\begin{equation}
    \widehat\theta_{\widehat\omega}
    \in
    \argmin_\theta
    \left\{
        \widehat{\mathcal R}_{k,V,n,\widehat\omega}(\theta)
        +
        \operatorname{pen}(\theta)
    \right\},
\end{equation}
or analogously with \eqref{eq:appendix_weighted_nmmr_u}.

For mini-batch training, the same formula is applied within each batch: compute the batch-level moment kernel on the batch values of $V$, form the residual vector, multiply residuals by $\sqrt{\widehat\omega_i}$, and evaluate the quadratic form. Full-batch training is preferable when feasible because the NMMR objective is pairwise.

\section{Two-Stage Solvers: KPV and DFPV}
\label{app:kpv}

KPV and DFPV estimate the bridge through two stages. The first stage estimates a nuisance conditional object, such as a conditional mean embedding or conditional feature. The second stage solves the bridge equation by regressing $Y$ on that estimated conditional object. In the policy-targeted implementation used here, the first stage is left unweighted and the policy weights are applied only in the second-stage bridge regression.

This choice matches the role of the weights in the theory. The weights define the norm over moment-view points $V=(A,Z,X)$ in which bridge residuals are penalized. The first stage estimates the nuisance conditional object itself, whereas the second stage is the step that directly enforces the bridge equation.

\subsection{Policy-targeted KPV}
\label{app:kpv_weighted}

KPV uses kernels to estimate the conditional mean embedding of $S$ given $V$ \citep{mastouri2021proximal}. Let $\ell$ be a positive definite kernel on $\cS$ with RKHS $\mathcal H_\ell$, and let $k$ be a positive definite kernel on $\cV$. The bridge equation can be written as
\begin{equation}
    \E[Y\mid V=v]
    =
    \E[h(S)\mid V=v]
    =
    \left\langle h,\mu_{S\mid v}\right\rangle_{\mathcal H_\ell},
\end{equation}
where
\begin{equation}
    \mu_{S\mid v}:=\E\{\ell(S,\cdot)\mid V=v\}
\end{equation}
is the conditional mean embedding.

Split the sample into a first-stage sample $\mathcal I_1$ of size $m_1$ and a second-stage sample $\mathcal I_2$ of size $m_2$. Define
\begin{equation}
    (K_{V,11})_{ij}:=k(v_i,v_j),\quad i,j\in\mathcal I_1,
    \qquad
    (K_{V,12})_{ij}:=k(v_i,v_j),\quad i\in\mathcal I_1,\ j\in\mathcal I_2.
\end{equation}
With first-stage regularization $\lambda_1>0$, the estimated conditional embedding at a second-stage point $v_j$, $j\in\mathcal I_2$, is
\begin{equation}
    \widehat\mu_{S\mid v_j}
    =
    \sum_{i\in\mathcal I_1}
    \widehat\Gamma_{ij}\ell(s_i,\cdot),
\end{equation}
where
\begin{equation}
    \widehat\Gamma
    :=
    \left(K_{V,11}+m_1\lambda_1 I_{m_1}\right)^{-1}K_{V,12}.
\end{equation}
Let
\begin{equation}
    (L_{S,11})_{ij}:=\ell(s_i,s_j),
    \qquad i,j\in\mathcal I_1,
\end{equation}
and define the second-stage embedding Gram matrix
\begin{equation}
    G_\mu
    :=
    \left(
        \left\langle
            \widehat\mu_{S\mid v_i},
            \widehat\mu_{S\mid v_j}
        \right\rangle_{\mathcal H_\ell}
    \right)_{i,j\in\mathcal I_2}
    =
    \widehat\Gamma^\top L_{S,11}\widehat\Gamma.
\end{equation}
Writing $y_2=(y_j:j\in\mathcal I_2)$, the unweighted second-stage KPV objective is
\begin{equation}
    J_{\mathrm{KPV}}(\beta)
    =
    (y_2-G_\mu\beta)^\top(y_2-G_\mu\beta)
    +
    m_2\lambda_2\beta^\top G_\mu\beta.
\end{equation}
The corresponding stabilized solution is
\begin{equation}
    \widehat\beta
    =
    \left(G_\mu+m_2\lambda_2 I_{m_2}+\delta I_{m_2}\right)^{-1}y_2.
\end{equation}
The fitted bridge is
\begin{equation}
    \widehat h(a,w,x)
    =
    \sum_{j\in\mathcal I_2}
    \widehat\beta_j\widehat\mu_{S\mid v_j}(a,w,x),
\end{equation}
where
\begin{equation}
    \widehat\mu_{S\mid v_j}(a,w,x)
    =
    \sum_{i\in\mathcal I_1}
    \widehat\Gamma_{ij}\ell\{s_i,(a,w,x)\}.
\end{equation}

For policy-targeted KPV, the first-stage embedding estimate is unchanged. The second-stage bridge regression becomes
\begin{equation}
    \widehat h_{\lambda_2,\widehat\omega}
    \in
    \argmin_{h\in\mathcal H_\ell}
    \left\{
        \frac{1}{m_2}
        \sum_{j\in\mathcal I_2}
        \widehat\omega_j
        \left(
            y_j-
            \left\langle h,\widehat\mu_{S\mid v_j}\right\rangle_{\mathcal H_\ell}
        \right)^2
        +
        \lambda_2\|h\|_{\mathcal H_\ell}^2
    \right\}.
\end{equation}
In finite-dimensional form,
\begin{equation}
    J_{\mathrm{KPV},\widehat\omega}(\beta)
    =
    (y_2-G_\mu\beta)^\top
    D_{\widehat\omega}
    (y_2-G_\mu\beta)
    +
    m_2\lambda_2\beta^\top G_\mu\beta,
    \label{eq:appendix_weighted_kpv_objective}
\end{equation}
where now
\begin{equation}
    D_{\widehat\omega}
    :=
    \operatorname{diag}(\widehat\omega_j:j\in\mathcal I_2).
\end{equation}
A symmetric stabilized implementation is obtained by defining
\begin{equation}
    G_{\mu,\widehat\omega}
    :=
    D_{\widehat\omega}^{1/2}G_\mu D_{\widehat\omega}^{1/2}
\end{equation}
and solving
\begin{equation}
    \widehat\eta_{\widehat\omega}
    =
    \left(
        G_{\mu,\widehat\omega}
        +
        m_2\lambda_2 I_{m_2}
        +
        \delta I_{m_2}
    \right)^{-1}
    D_{\widehat\omega}^{1/2}y_2,
\end{equation}
then setting
\begin{equation}
    \widehat\beta_{\widehat\omega}
    :=
    D_{\widehat\omega}^{1/2}\widehat\eta_{\widehat\omega}.
\end{equation}
The policy-targeted KPV bridge is
\begin{equation}
    \widehat h_{\widehat\omega}(a,w,x)
    =
    \sum_{j\in\mathcal I_2}
    \widehat\beta_{\widehat\omega,j}
    \widehat\mu_{S\mid v_j}(a,w,x).
\end{equation}
When $D_{\widehat\omega}=I_{m_2}$, this reduces to ordinary KPV.

\subsection{Policy-targeted DFPV}
\label{app:dfpv_weighted}

DFPV replaces kernel conditional embeddings by learned finite-dimensional features. Let
\begin{equation}
    \phi_\theta:\cS\to\R^d
\end{equation}
be a learned bridge-side feature map and write
\begin{equation}
    h_{\beta,\theta}(s)=\beta^\top\phi_\theta(s).
\end{equation}
The bridge equation becomes
\begin{equation}
    \E(Y\mid V)
    =
    \beta^\top \E\{\phi_\theta(S)\mid V\}.
\end{equation}
DFPV therefore first estimates the conditional feature
\begin{equation}
    \mu_\theta(v):=\E\{\phi_\theta(S)\mid V=v\},
\end{equation}
and then regresses $Y$ on the estimated conditional feature.

Let $I_1$ and $I_2$ denote the first-stage and second-stage samples, with cardinalities $m_1$ and $m_2$. The first stage solves
\begin{equation}
    \widehat g
    \in
    \argmin_{g\in\mathcal G}
    \left\{
        \frac{1}{m_1}
        \sum_{i\in I_1}
        \|\phi_\theta(s_i)-g(v_i)\|_2^2
        +
        \lambda_1\Omega_1(g)
    \right\}.
\end{equation}
For $j\in I_2$, write
\begin{equation}
    \widehat\mu_j:=\widehat g(v_j).
\end{equation}
Let $\widehat M\in\R^{m_2\times d}$ be the matrix whose $j$th row is $\widehat\mu_j^\top$, and let $y_2=(y_j:j\in I_2)$.

The unweighted second-stage ridge problem is
\begin{equation}
    \widehat\beta
    \in
    \argmin_{\beta\in\R^d}
    \left\{
        \frac{1}{m_2}
        \sum_{j\in I_2}
        (y_j-\beta^\top\widehat\mu_j)^2
        +
        \lambda_2\|\beta\|_2^2
    \right\},
\end{equation}
with stabilized solution
\begin{equation}
    \widehat\beta
    =
    \left(
        \widehat M^\top\widehat M
        +
        m_2\lambda_2 I_d
        +
        \delta I_d
    \right)^{-1}
    \widehat M^\top y_2.
\end{equation}
The fitted bridge is
\begin{equation}
    \widehat h(a,w,x)=\widehat\beta^\top\phi_\theta(a,w,x).
\end{equation}

For policy-targeted DFPV, the first-stage feature regression is left unweighted and the second-stage ridge problem becomes
\begin{equation}
    \widehat\beta_{\widehat\omega}
    \in
    \argmin_{\beta\in\R^d}
    \left\{
        \frac{1}{m_2}
        \sum_{j\in I_2}
        \widehat\omega_j
        (y_j-\beta^\top\widehat\mu_j)^2
        +
        \lambda_2\|\beta\|_2^2
    \right\}.
\end{equation}
Equivalently, with
\begin{equation}
    D_{\widehat\omega}
    :=
    \operatorname{diag}(\widehat\omega_j:j\in I_2),
\end{equation}
we minimize
\begin{equation}
    J_{\mathrm{DFPV},\widehat\omega}(\beta)
    =
    (y_2-\widehat M\beta)^\top
    D_{\widehat\omega}
    (y_2-\widehat M\beta)
    +
    m_2\lambda_2\|\beta\|_2^2.
\end{equation}
The stabilized solution is
\begin{equation}
    \widehat\beta_{\widehat\omega}
    =
    \left(
        \widehat M^\top D_{\widehat\omega}\widehat M
        +
        m_2\lambda_2 I_d
        +
        \delta I_d
    \right)^{-1}
    \widehat M^\top D_{\widehat\omega}y_2.
    \label{eq:appendix_weighted_dfpv_solution}
\end{equation}
The policy-targeted DFPV bridge is
\begin{equation}
    \widehat h_{\widehat\omega}(a,w,x)
    =
    \widehat\beta_{\widehat\omega}^\top\phi_\theta(a,w,x).
\end{equation}
When $D_{\widehat\omega}=I_{m_2}$, this reduces to ordinary DFPV.

\section{Implementation Details and Pitfalls}
\label{app:implementation_details}
\label{app:dfpv}

This section records practical details shared by the solver implementations.

\paragraph{Cross-fitting of bridge pseudo-outcomes.}
Within each policy-targeted iteration, cross-fitting is used to construct the
bridge pseudo-outcomes that train the projected response surface. For each fold
$I_k$ in a partition of $\{1,\ldots,n\}$, the bridge
$\widehat h_{-k}^{(s)}$ is trained on the observations outside $I_k$ using the
current weights and evaluated on the held-out observations over the treatment
grid. The resulting out-of-fold pseudo-outcomes
$\widetilde H_{im}^{(s)}=\widehat h_{-k}^{(s)}(a_m,W_i,X_i)$, for
$i\in I_k$, are pooled across folds and used to train a single response-surface
estimator $\widehat m^{(s)}$. The observation-specific pseudo-optimal treatments
$\widehat\pi_i^{(s)}$ obtained from this surface are then used to form the next
weights. This prevents each observation's bridge pseudo-outcome from being
computed by a bridge fit that was trained on that observation.

\paragraph{Projection from bridge to response surface.}
After fitting a bridge, the induced causal response surface is
\begin{equation}
    m_h(a,x)=\E\{h(a,W,x)\mid X=x\}.
\end{equation}
In practice, this expectation is estimated by evaluating the fitted bridge on a
treatment grid and then regressing the bridge pseudo-outcomes on $(a,X)$. During
weight construction, these pseudo-outcomes are the out-of-fold values described
above. The pseudo-optimal treatment is obtained by maximizing the resulting
surface estimate over the treatment grid.

\paragraph{Computational resources} All experiments were run on CPU. The implementation supports parallel execution across independent random seeds, models, and experimental configurations, which makes the evaluation pipeline naturally scalable on multi-core machines. The computational overhead introduced by the proposed decision-aware weighting strategy is linear in the number of weighting rounds: if the base proximal learner has training cost $T_{\mathrm{base}}$, then running $n_{rounds}$ weighting rounds requires approximately $\mathcal{O}(n_{rounds}\,T_{\mathrm{base}})$ time, up to lower-order costs for computing and normalizing the weights.

\section{Dataset Details}
\label{app:dataset-details}

This appendix describes the synthetic and semi-synthetic benchmarks used in our experiments. Both datasets are designed to evaluate proximal learning for individualized treatment selection with continuous treatments, hidden confounding, and proxy variables. In each case, the observed data consist of
\begin{equation}
O=(X,Z,W,A,Y),
\end{equation}
where $X$ denotes observed covariates, $Z$ is a treatment-inducing proxy, $W$ is an outcome-inducing proxy, $A\in\mathcal{A}$ is a continuous treatment, and $Y$ is the outcome. The latent confounders are unobserved by the learner but affect both treatment assignment and the outcome.

The key property of both benchmarks is that the conditional causal response
\begin{equation}
m_0(a,x)=\mathbb{E}\{Y(a)\mid X=x\}
\end{equation}
is non-monotone in $a$, with an interior and covariate-dependent optimal treatment. This makes the datasets suitable for evaluating decision-aware bridge learning: accurate estimation near the individualized optimum is more important for policy regret than uniform accuracy over clearly suboptimal treatment regions. The following subsections describe the fully synthetic DGP, the TCGA semi-synthetic DGP, and the Monte Carlo procedure used to compute ground-truth causal curves and policy values.

\subsection{Synthetic proximal continuous-treatment benchmark}
\label{app:synthetic-dgp}

We use a synthetic benchmark adapted from the proximal continuous-treatment graph of
\citet{wu2024doubly}. The original Wu et al. simulator is designed for proximal causal
effect estimation with continuous treatments; we retain its proxy-confounding structure,
but modify the outcome surface so that the conditional causal response has an interior, covariate-dependent maximizer.

The corresponding causal graph is shown in Figure~\ref{fig:synthetic-proximal-graph}.

\begin{figure}[h]
\centering
\tikzset{
  observed/.style={draw, rounded corners, fill=gray!20, inner sep=2pt},
  hidden/.style={draw, dashed, rounded corners, fill=white, inner sep=2pt}
}
\begin{tikzcd}[
    column sep=small,
    row sep=small
]
|[hidden]| \epsilon_Z \arrow[d, dashed] &
|[hidden]| \epsilon_1 \arrow[d, dashed] &
|[hidden]| \epsilon_2 \arrow[dr, dashed] &
|[hidden]| \epsilon_3 \arrow[dll, dashed] \arrow[d, dashed] &
|[hidden]| \epsilon_W \arrow[d, dashed] \\
|[observed]| Z \arrow[ddr] &
|[hidden]| U_Z \arrow[l, dashed] \arrow[ddrr, bend left, dashed] &&
|[hidden]| U_W \arrow[r, dashed] \arrow[ddll, bend right, dashed] &
|[observed]| W \arrow[ddl] \\
&& |[observed]| X \arrow[dl] \arrow[dr] && \\
& |[observed]| A \arrow[rr] && |[observed]| Y &
\end{tikzcd}
\caption{Synthetic proximal graph. Observed variables are shown in gray with solid nodes,
while unobserved variables are shown in white with dashed nodes. The treatment-inducing
proxy is $Z$, the outcome-inducing proxy is $W$, $X$ is observed, and
$U_Z,U_W$ are latent confounding components.}
\label{fig:synthetic-proximal-graph}
\end{figure}
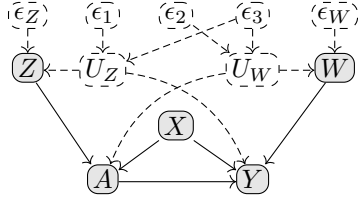

Let $d_X,d_Z,d_W$ denote the dimensions of $X,Z,W$, and define
\begin{equation}
\theta_d=(1,1/2^2,\ldots,1/d^2)^\top .
\end{equation}
We sample
\begin{equation}
X\sim \mathcal{N}(0,\Sigma_X),\qquad
(\Sigma_X)_{jk}=
\begin{cases}
1, & j=k,\\
\rho_X, & |j-k|=1,\\
0, & \text{otherwise},
\end{cases}
\end{equation}
and independent noises
\begin{equation}
\epsilon_1,\epsilon_2,\epsilon_3\sim \mathcal{N}(0,1),\qquad
V_Z\sim \mathrm{Unif}([-1,1]^{d_Z}),\qquad
V_W\sim \mathrm{Unif}([-1,1]^{d_W}).
\end{equation}
The latent variables and proxies are
\begin{equation}
U_Z=\epsilon_1+\epsilon_3,\qquad
U_W=\epsilon_2+\epsilon_3,
\end{equation}
\begin{equation}
Z=V_Z+c_P U_Z\mathbf{1}_{d_Z},\qquad
W=V_W+c_P U_W\mathbf{1}_{d_W}.
\end{equation}
The treatment assignment mean is
\begin{equation}
\mu_A(X,Z)=\Lambda\!\left(3X^\top\theta_{d_X}+3Z^\top\theta_{d_Z}\right),
\qquad
\Lambda(t)=0.1+0.8\frac{\exp(t)}{1+\exp(t)}.
\end{equation}
The factual treatment is
\begin{equation}
A=\Pi_{\mathcal{A}}\{\mu_A(X,Z)+\sigma_A U_W\},
\end{equation}
where $\Pi_{\mathcal{A}}$ denotes clipping to the compact treatment support
$\mathcal{A}=[a_{\min},a_{\max}]$. In our experiments, $\mathcal{A}$ is estimated once as
the empirical $(q_{\min},q_{\max})$-quantiles of the unclipped assignment distribution using
$n_{\mathrm{pilot}}$ pilot samples.

The outcome is
\begin{equation}
Y=g(A,X_0)+\beta_{XW}\{X^\top\theta_{d_X}+W^\top\theta_{d_W}\}
+\psi(A,X,U_Z)+\epsilon_Y,
\qquad
\epsilon_Y\sim\mathcal{N}(0,\sigma_Y^2).
\end{equation}
The structural response is
\begin{equation}
g(a,x_0)=
\alpha(x_0)\exp\!\left[-\frac{1}{2}r(a,x_0)^2\right]
\left\{1+\kappa(x_0)\tanh\!\big(1.5\,r(a,x_0)\big)\right\}
-\lambda_{\mathrm{tail}}r(a,x_0)^2,
\end{equation}
where
\begin{equation}
r(a,x_0)=\frac{a-c(x_0)}{s(x_0)},\qquad
c(x_0)=\frac{a_{\min}+a_{\max}}{2}+\delta_c\tanh(\eta_c x_0),
\end{equation}
\begin{equation}
s(x_0)=\delta_s(a_{\max}-a_{\min})\{1+0.10\tanh(0.8x_0)\},
\end{equation}
\begin{equation}
\alpha(x_0)=1+\delta_\alpha\tanh(x_0),\qquad
\kappa(x_0)=\delta_\kappa\tanh(x_0).
\end{equation}
The latent outcome-confounding term is
\begin{equation}
\psi(a,x,U_Z)
=
\left[
\gamma
\left\{
1+\beta_A
\left(
\frac{a-\frac{a_{\min}+a_{\max}}{2}}
{\delta_\psi(a_{\max}-a_{\min})}
\right)^2
\right\}
+\gamma_0
\right]U_Z .
\end{equation}
Thus the latent contribution to the outcome is treatment-dependent, so failing to account
for hidden confounding can distort the shape of the estimated dose-response curve.

Under an intervention $A=a$,
\begin{equation}
Y(a)=g(a,X_0)+\beta_{XW}\{X^\top\theta_{d_X}+W^\top\theta_{d_W}\}
+\psi(a,X,U_Z)+\epsilon_Y.
\end{equation}
Since $X$ is independent of $(W,U_Z,\epsilon_Y)$, and these variables are mean-zero,
the conditional causal response is
\begin{equation}
m_0(a,x)=\mathbb{E}\{Y(a)\mid X=x\}
=
g(a,x_0)+\beta_{XW}x^\top\theta_{d_X}.
\end{equation}
Consequently, the individualized optimal treatment
\begin{equation}
a^\star(x)\in\arg\max_{a\in\mathcal{A}}m_0(a,x)
\end{equation}
is interior and varies with $x_0$. This property makes the benchmark appropriate for
evaluating decision-aware bridge learning: errors near $a^\star(x)$ affect policy regret,
whereas errors in clearly suboptimal regions are less consequential.

Figure~\ref{fig:synthetic-reference-curves} illustrates the effect of hidden confounding in the
synthetic benchmark. For both an individual covariate value and the population average, the
associational curves differ substantially from the causal curves, and their maximizers can occur at
different treatment values. This confirms that selecting treatments using $\mathbb{E}[Y\mid A=a,X=x]$
or $\mathbb{E}[Y\mid A=a]$ is not equivalent to optimizing the interventional responses
$\mathbb{E}\{Y(a)\mid X=x\}$ or $\mathbb{E}\{Y(a)\}$.

\begin{figure}[t]
\centering
\begin{subfigure}{0.48\linewidth}
    \centering
    \includegraphics[width=\linewidth]{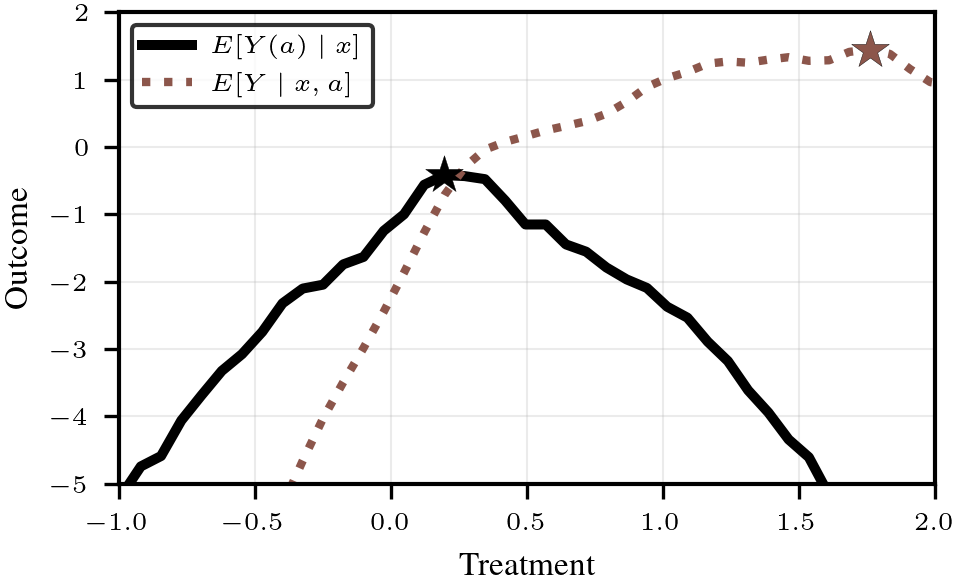}
    \caption{Individual-level curves for one randomly selected test point.}
    \label{fig:synthetic-individual-reference-curves}
\end{subfigure}
\hfill
\begin{subfigure}{0.48\linewidth}
    \centering
    \includegraphics[width=\linewidth]{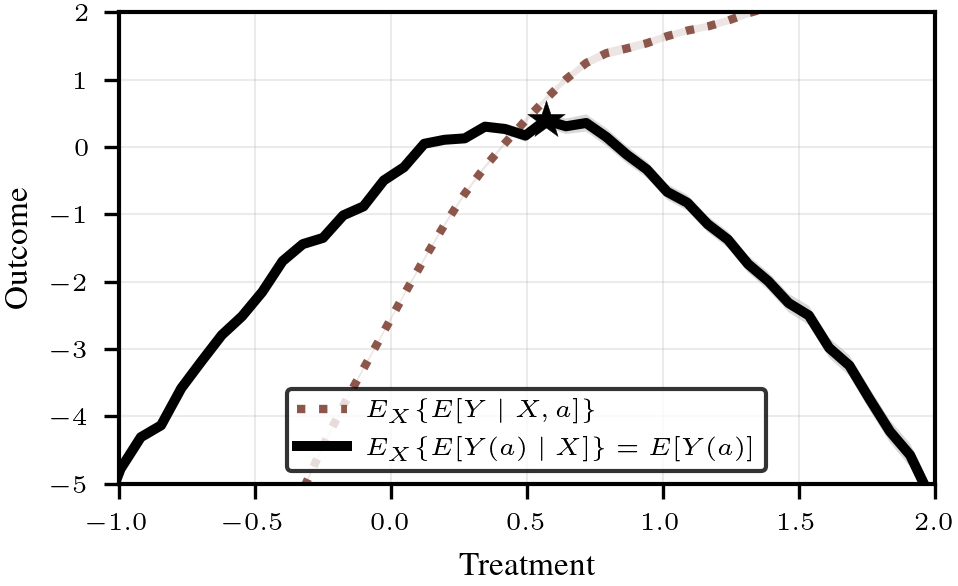}
    \caption{Population-level curves.}
    \label{fig:synthetic-population-reference-curves}
\end{subfigure}
\caption{
Reference causal and associational curves in the synthetic benchmark.
Left: for a fixed covariate value $x$, the solid curve shows the interventional response
$\mathbb{E}\{Y(a)\mid X=x\}$, while the dotted curve shows the associational response
$\mathbb{E}\{Y\mid A=a,X=x\}$. Right: the corresponding population quantities
$\mathbb{E}\{Y(a)\}$ and $\mathbb{E}\{Y\mid A=a\}$ are obtained by averaging over $X$.
Stars mark the maximizers of each curve. The separation between causal and associational
curves illustrates the effect of hidden confounding; in particular, optimizing the
associational curve can lead to a different treatment recommendation than optimizing the
causal curve.
}
\label{fig:synthetic-reference-curves}
\end{figure}

For reproducibility, the values used in all synthetic experiments are
\begin{equation}
\begin{gathered}
d_X=5,\quad d_Z=d_W=2,\quad \rho_X=0.5,\quad
c_P=0.35,\quad \sigma_A=0.35,\quad \sigma_Y=0,\\
q_{\min}=0.001,\quad q_{\max}=0.999,\quad
n_{\mathrm{pilot}}=50{,}000,\quad \text{support seed}=2718,\\
\beta_{XW}=1.2,\quad
\delta_c=0.4\times 0.5=0.2,\quad
\eta_c=2.7,\quad
\delta_s=0.1\times 0.45=0.045,\\
\delta_\alpha=0.25+0.1=0.35,\quad
\delta_\kappa=0.7,\quad
\lambda_{\mathrm{tail}}=0.06,\\
\gamma=0.7,\quad
\beta_A=0.25,\quad
\delta_\psi=0.18,\quad
\gamma_0=1.2,\\
n_{\mathrm{train}}=3000,\quad n_{\mathrm{test}}=1000,\quad
|\mathcal{G}_A|=51.
\end{gathered}
\end{equation}
Here $|\mathcal{G}_A|$ denotes the number of equally spaced grid points used for policy
optimization over $\mathcal{A}$.

\subsection{Semi-synthetic TCGA proximal benchmark}
\label{app:tcga-dgp}

We also evaluate on a semi-synthetic benchmark based on TCGA gene-expression
covariates. The raw gene-expression matrix is taken from the TCGA benchmark
used in SCIGAN \citep{bica2020scigan}, which is derived from the Cancer Genome
Atlas pan-cancer project \citep{weinstein2013cancer}. We use the real
gene-expression features to construct observed covariates and proxy-relevant
features, while the treatment, hidden confounder, proxies, and outcome are
generated from a controlled proximal structural model. This gives a benchmark
with realistic covariate structure and known counterfactual outcomes.

The corresponding proximal causal graph is shown in Figure~\ref{fig:tcga-proximal-graph}.

\begin{figure}[h]
\centering
\tikzset{
  observed/.style={draw, rounded corners, fill=gray!20, inner sep=2pt},
  hidden/.style={draw, dashed, rounded corners, fill=white, inner sep=2pt}
}
\begin{tikzcd}[
    column sep=large,
    row sep=large
]
& |[hidden]| U \arrow[dl, dashed] \arrow[dr, dashed]
    \arrow[ddr, dashed] \arrow[ddl, dashed]& \\
|[observed]| Z \arrow[d] &
|[observed]| X \arrow[dr] \arrow[dl] &
|[observed]| W \arrow[d] \\
|[observed]| A \arrow[rr] & & |[observed]| Y
\end{tikzcd}
\caption{Semi-synthetic TCGA proximal graph. Observed variables are shown in gray,
while the latent confounder is shown in white with a dashed border. The proxy $Z$
is treatment-inducing and affects treatment assignment but not the outcome directly.
The proxy $W$ is outcome-inducing and affects the outcome but not treatment
assignment directly.}
\label{fig:tcga-proximal-graph}
\end{figure}
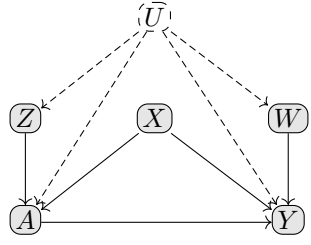

Let $G^{\mathrm{raw}}\in\mathbb{R}^{n\times p}$ denote the raw TCGA
gene-expression matrix. We first standardize each gene,
\begin{equation}
G_{ij}
=
\frac{G^{\mathrm{raw}}_{ij}-\bar G_j}{s_j+10^{-8}},
\end{equation}
and select the $p_0=180$ genes with largest empirical variance. The selected
genes are split into disjoint blocks used to construct observed covariates and
proxy-specific features. We apply PCA to obtain
\begin{equation}
X_i=\mathrm{PCA}_5(G_i^X)\in\mathbb{R}^5,\qquad
\widetilde Z_i=\mathrm{PCA}_3(G_i^Z)\in\mathbb{R}^3,\qquad
\widetilde W_i=\mathrm{PCA}_3(G_i^W)\in\mathbb{R}^3.
\end{equation}
The variables observed by the learner are
\begin{equation}
O_i=(X_i,Z_i,W_i,A_i,Y_i),
\end{equation}
where $Z_i\in\mathbb{R}^3$, $W_i\in\mathbb{R}^3$, and $A_i\in(0,1)$.

The latent confounder is two-dimensional,
\begin{equation}
U_i=(U_{i1},U_{i2})^\top,\qquad
U_{i1}\sim \mathcal{N}(0,1),\qquad
U_{i2}\sim \mathrm{Exp}(1)-1.
\end{equation}
The treatment and outcome proxies are generated as noisy nonlinear measurements
of $U_i$, $X_i$, and the TCGA-derived proxy features:
\begin{equation}
Z_i=
\tanh(B_ZU_i+C_ZX_i+R_Z\widetilde Z_i)+\epsilon_i^Z,
\qquad
\epsilon_i^Z\sim\mathcal{N}(0,\sigma_Z^2 I_3),
\end{equation}
\begin{equation}
W_i=
\tanh(B_WU_i+C_WX_i+R_W\widetilde W_i)+\epsilon_i^W,
\qquad
\epsilon_i^W\sim\mathcal{N}(0,\sigma_W^2 I_3).
\end{equation}
Here $\sigma_Z=\sigma_W=0.2$, and
\begin{equation}
B_Z
=
0.8
\begin{pmatrix}
1 & 0\\
0 & 1\\
1/\sqrt{2} & -1/\sqrt{2}
\end{pmatrix},
\qquad
B_W
=
0.8
\begin{pmatrix}
1 & 0\\
0 & 1\\
1/\sqrt{2} & 1/\sqrt{2}
\end{pmatrix},
\end{equation}
\begin{equation}
C_Z
=
0.2
\begin{pmatrix}
1&0&0&0&0\\
0&1&0&0&0\\
0&0&1&0&0
\end{pmatrix},
\qquad
C_W
=
0.2
\begin{pmatrix}
0&0&1&0&0\\
0&0&0&1&0\\
0&0&0&0&1
\end{pmatrix},
\end{equation}
\begin{equation}
R_Z=0.5I_3,\qquad R_W=0.5I_3.
\end{equation}
Since both $B_Z$ and $B_W$ have rank two, both proxies carry information
about the two-dimensional latent confounder. By construction, $Z$ affects
treatment assignment but does not enter the outcome equation directly, while
$W$ enters the outcome equation but does not enter the treatment assignment
equation.

Each unit has a latent center for the dose-response curve,
\begin{equation}
s_i^\star=\gamma_0+\gamma_X^\top X_i+\gamma_U^\top U_i,\qquad
\bar s_i^\star=\frac{s_i^\star-\mu_s}{\sigma_s},
\end{equation}
\begin{equation}
a_i^\star=0.15+0.70\,\sigma(\bar s_i^\star),
\qquad
\sigma(t)=\frac{1}{1+\exp(-t)}.
\end{equation}
The constants $\mu_s$ and $\sigma_s$ are the empirical mean and standard
deviation of $s^\star$, estimated once from a pilot sample. The coefficients are
\begin{equation}
\gamma_0=0,\qquad
\gamma_X=\frac{0.6}{\sqrt{5}}(1,-1,1,-1,1)^\top,\qquad
\gamma_U=\frac{0.8}{\sqrt{2}}(1,-1)^\top.
\end{equation}
Thus $a_i^\star\in(0.15,0.85)$, so the relevant optima are interior to the
treatment support.

Treatment assignment is generated from a beta distribution. First define
\begin{equation}
r_i=\eta_0+\eta_X^\top X_i+\eta_U^\top U_i+\eta_Z^\top Z_i,
\qquad
\bar r_i=\frac{r_i-\mu_r}{\sigma_r},
\end{equation}
where $\mu_r$ and $\sigma_r$ are pilot estimates of the mean and standard
deviation of $r$. The assignment mean is
\begin{equation}
m_i=
\operatorname{clip}\left\{
0.3a_i^\star+0.7\sigma(\bar r_i),\,0.02,\,0.98
\right\},
\end{equation}
and
\begin{equation}
A_i\sim \mathrm{Beta}\{\phi m_i,\phi(1-m_i)\},
\qquad
\phi=20.
\end{equation}
The assignment coefficients are
\begin{equation}
\eta_0=0,\qquad
\eta_X=\frac{0.4}{\sqrt{5}}(1,1,-1,1,-1)^\top,
\end{equation}
\begin{equation}
\eta_U=\frac{0.8}{\sqrt{2}}(1,1)^\top,\qquad
\eta_Z=\frac{0.6}{\sqrt{3}}(1,-1,1)^\top.
\end{equation}
Because $A_i$ depends on $U_i$, the observational treatment is confounded.

The potential outcome under treatment $a\in[0,1]$ is
\begin{equation}
Y_i(a)
=
\frac{
\mu_i+b_i+\lambda_W^\top W_i+\Gamma U_{i1}aX_{i0}
-c_i(a-a_i^\star)^2+\epsilon_i^Y
}{4},
\qquad
\epsilon_i^Y\sim\mathcal{N}(0,\sigma_Y^2),
\end{equation}
with $\sigma_Y=0.1$ and $\Gamma=1.0$. The observed outcome is $Y_i=Y_i(A_i)$.
The outcome components are
\begin{equation}
c_i=5+5\,\mathrm{softplus}
\left(\kappa_0+\kappa_X^\top X_i+\kappa_U^\top U_i\right),
\qquad
\mathrm{softplus}(t)=\log(1+\exp(t)),
\end{equation}
\begin{equation}
\mu_i=\theta_X^\top X_i+\theta_U^\top U_i,
\qquad
b_i=0.5\sin(\omega_X^\top X_i)+0.5\cos(\omega_U^\top U_i).
\end{equation}
The corresponding coefficients are
\begin{equation}
\theta_X=\frac{0.25}{\sqrt{5}}(1,-1,1,1,-1)^\top,
\qquad
\theta_U=\frac{0.5}{\sqrt{2}}(1,-1)^\top,
\end{equation}
\begin{equation}
\omega_X=\frac{0.2}{\sqrt{5}}(1,2,-1,1,-2)^\top,
\qquad
\omega_U=\frac{0.3}{\sqrt{2}}(1,1)^\top,
\end{equation}
\begin{equation}
\kappa_0=0,\qquad
\kappa_X=\frac{0.3}{\sqrt{5}}(1,-1,0.5,1,-0.5)^\top,
\qquad
\kappa_U=\frac{0.3}{\sqrt{2}}(1,-1)^\top,
\end{equation}
\begin{equation}
\lambda_W=\frac{0.5}{\sqrt{3}}(1,-1,1)^\top.
\end{equation}
The quadratic term gives each unit a single-peaked dose-response curve centered
near $a_i^\star$, while the additional interaction $(\Gamma U_{i1}aX_{i0})$
makes the latent confounding treatment-dependent. Therefore, hidden confounding
can distort not only the level of the response curve, but also its shape.

Figure~\ref{fig:tcga-reference-curves} shows the same diagnostic comparison for the
semi-synthetic TCGA benchmark. The causal curves retain an interior maximizer, while the
associational curves are shifted by the confounded treatment assignment mechanism. This visualizes
why the benchmark requires proxy-based adjustment: the observed conditional association does not
recover the causal dose-response surface used for policy evaluation.

\begin{figure}[t]
\centering
\begin{subfigure}{0.48\linewidth}
    \centering
    \includegraphics[width=\linewidth]{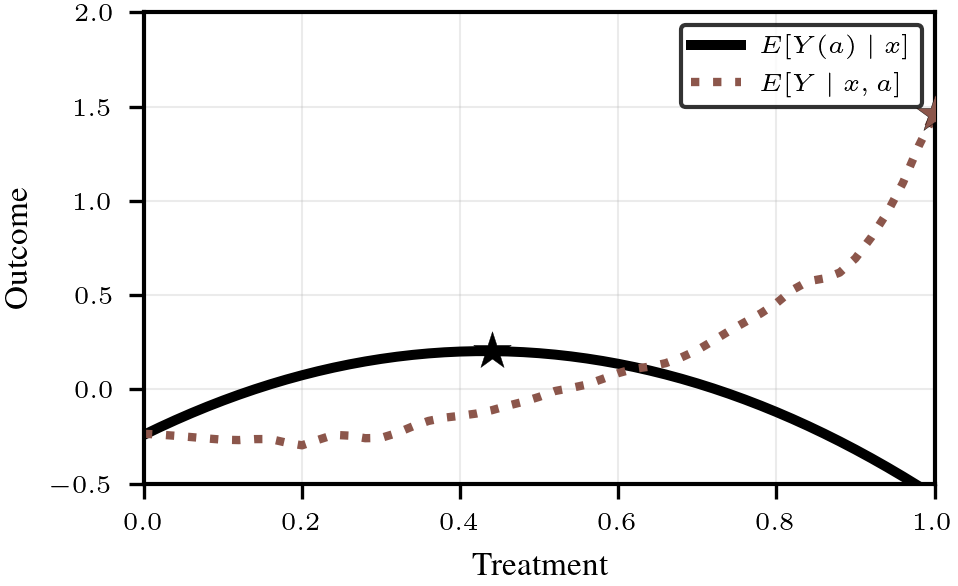}
    \caption{Individual-level curves for one randomly selected test point.}
    \label{fig:tcga-individual-reference-curves}
\end{subfigure}
\hfill
\begin{subfigure}{0.48\linewidth}
    \centering
    \includegraphics[width=\linewidth]{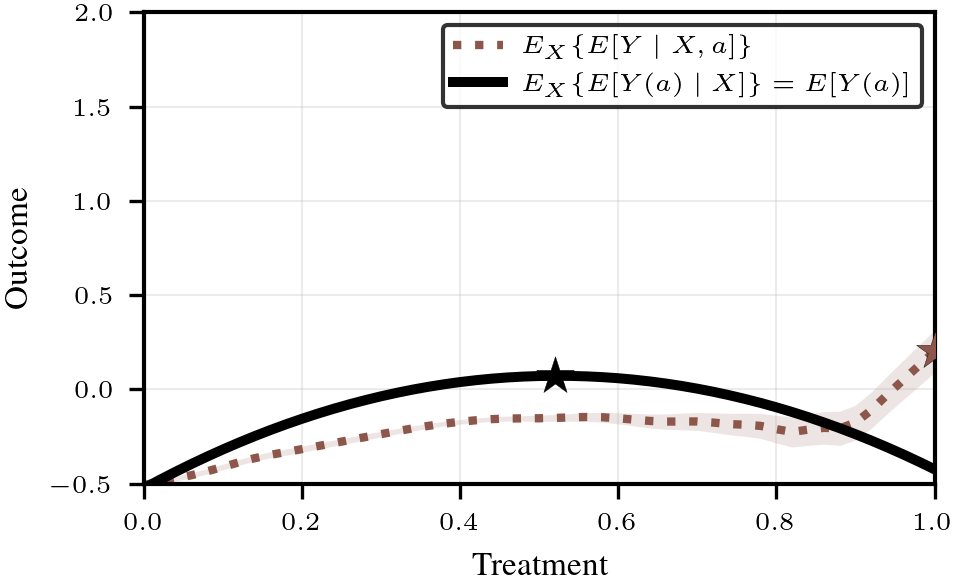}
    \caption{Population-level curves.}
    \label{fig:tcga-population-reference-curves}
\end{subfigure}
\caption{
Reference causal and associational curves in the semi-synthetic TCGA benchmark.
Left: for a fixed covariate value $x$, the solid curve shows the interventional response
$\mathbb{E}\{Y(a)\mid X=x\}$, while the dotted curve shows the associational response
$\mathbb{E}\{Y\mid A=a,X=x\}$. Right: the corresponding population quantities
$\mathbb{E}\{Y(a)\}$ and $\mathbb{E}\{Y\mid A=a\}$ are obtained by averaging over $X$.
Stars mark the maximizers of the displayed curves. The discrepancy between causal and
associational curves illustrates the impact of hidden confounding and shows that optimizing
the associational response can lead to a different treatment recommendation than optimizing
the causal response.
}
\label{fig:tcga-reference-curves}
\end{figure}

The causal estimand used for evaluation is the conditional response surface
\begin{equation}
m_0(a,x)=\mathbb{E}\{Y(a)\mid X=x\}.
\end{equation}
For this semi-synthetic DGP, $m_0(a,x)$ does not have a simple closed form
because the outcome averages over $U$, $W$, proxy noise, and TCGA-derived
features. We therefore evaluate $m_0(a,x)$, individualized optima, and policy
values by Monte Carlo simulation from the known structural equations. The
individualized policy target is
\begin{equation}
a^\star(x)\in\arg\max_{a\in\mathcal{G}_A}m_0(a,x),
\end{equation}
where $\mathcal{G}_A$ is a grid of $51$ equally spaced treatment values on
$[0,1]$.

For all semi-synthetic TCGA experiments we use
\begin{equation}
n_{\mathrm{train}}=3000,\qquad n_{\mathrm{test}}=1000,\qquad
\mathcal{A}=[0,1],\qquad |\mathcal{G}_A|=51.
\end{equation}
The TCGA preprocessing uses $180$ top-variance genes,
PCA dimensions $5,3,3$ for $X,\widetilde Z,\widetilde W$.

\subsection{Monte Carlo construction of ground-truth causal curves}
\label{app:ground-truth-curves}

For both benchmarks, the ground-truth conditional causal response
\begin{equation}
m_0(a,x)=\mathbb{E}\{Y(a)\mid X=x\}
\end{equation}
is evaluated directly from the known structural equations. This step is necessary because the
effect of the latent confounders is non-additive and depends on both the treatment and covariates.
Consequently, the conditional response cannot in general be read off from the deterministic part of
the outcome equation by simply removing an additive confounding term.

For a fixed covariate value $x$ and treatment level $a$, we use the structural intervention
\begin{equation}
do(A=a)
\end{equation}
and marginalize over the remaining stochastic components of the data-generating process. That is,
we write
\begin{equation}
m_0(a,x)
=
\mathbb{E}\{Y(a)\mid X=x\}
=
\mathbb{E}\!\left[
Y \mid do(A=a), X=x
\right],
\end{equation}
where the expectation is taken over the latent confounders, proxy variables, and outcome noise
generated under the intervention. Equivalently,
\begin{equation}
m_0(a,x)
=
\int
\mathbb{E}\!\left[
Y \mid do(A=a), X=x, U=u, W=w
\right]
\,dP(u,w\mid X=x).
\end{equation}
Since the simulator specifies the full joint law, this integral can be approximated by Monte Carlo.
For each evaluation covariate $x$ and treatment value $a$, we repeatedly sample latent variables
and outcome proxies from their conditional distribution, set the treatment to
$a$, and evaluate the interventional outcome equation.

Concretely, for $b=1,\ldots,B$, we draw
\begin{equation}
U^{(b)},W^{(b)} \sim P(U,W\mid X=x)
\end{equation}
according to the corresponding data-generating process, set $A=a$, and compute
\begin{equation}
Y^{(b)}(a,x)
=
Y\!\left(
do(A=a),X=x,U^{(b)},W^{(b)}
\right).
\end{equation}
The ground-truth curve is then approximated by
\begin{equation}
\widehat m_0(a,x)
=
\frac{1}{B}
\sum_{b=1}^B
Y^{(b)}(a,x).
\end{equation}
Once $x$ is fixed, the simulator is used to generate
the remaining latent and proxy variables under the intervention. We use $B=256$ Monte Carlo
draws for each pair $(x,a)$.

Ground-truth individualized policies are computed by evaluating this Monte Carlo curve on the
same treatment grid used by the learners:
\begin{equation}
\widehat a^\star(x)
\in
\arg\max_{a\in\mathcal{G}_A}
\widehat m_0(a,x).
\end{equation}
Policy values and regrets are then computed using these Monte Carlo estimates of the conditional
causal response rather than associational quantities such as
$\mathbb{E}[Y\mid A=a,X=x]$, which remain biased under hidden confounding.

\section{Hyperparameter Selection}
\label{app:hparam-selection}

This appendix describes the hyperparameter-selection protocol used in the experiments. The protocol has two stages. First, we select the base hyperparameters of each unweighted proximal learner using only factual validation error. Second, we calibrate the decision-aware weighting parameters using a limited shared calibration procedure and then keep them fixed across weighted methods and datasets. This separation avoids tuning model-specific weighting parameters using oracle quantities that would not be available in real observational applications.

\subsection{Stage 1: Base model selection using factual validation error}
\label{app:hparam-stage1}

For each dataset and each model class, we select the unweighted base learner by minimizing factual prediction error on held-out validation data. For a fitted bridge or response predictor $\widehat f_\lambda$, with hyperparameters $\lambda$, the validation objective is
\begin{equation}
\mathrm{RMSE}_{\mathrm{factual}}(\lambda)
=
\left[
\frac{1}{n_{\mathrm{val}}}
\sum_{i\in\mathcal{I}_{\mathrm{val}}}
\left\{
Y_i-\widehat f_\lambda(A_i,W_i,X_i)
\right\}^2
\right]^{1/2}.
\end{equation}
For methods whose factual prediction is obtained through a bridge-induced surface rather than a direct regression function, $\widehat f_\lambda$ denotes the corresponding factual prediction used by that estimator. This objective uses only observed factual outcomes and does not require counterfactual or oracle information.

The search was performed with Bayesian optimization using a tree-structured Parzen estimator sampler \citep{bergstra2011algorithms}, with a maximum budget of $100$ trials per model class and dataset, using Optuna \citep{akiba2019optuna}. This stage includes only unweighted learners; decision-aware weighting parameters are calibrated separately in Stage~\ref{app:hparam-stage2}. The search spaces were as follows.

For PMMR, we searched the ridge regularization parameter over
$\{10^{-4},5\cdot 10^{-4},10^{-3},5\cdot 10^{-3},10^{-2},5\cdot 10^{-2},10^{-1},5\cdot 10^{-1},1\}$.
The bandwidth scale for the bridge input $(A,W,X)$ was searched over
$\{0.3,0.5,0.75,1.0,1.25,1.5,1.75,2.0,2.25,2.5,3.0,3.5,4.0\}$, and the bandwidth scale for the conditioning input $(A,Z,X)$ was searched over
$\{0.3,0.5,0.75,1.0,1.25,1.5,1.75,2.0,2.25,2.5,2.75,3.0\}$.

For KPV, the first-stage regularization parameter was searched over
$\{10^{-4},3\cdot10^{-4},10^{-3},3\cdot10^{-3},10^{-2},3\cdot10^{-2},10^{-1}\}$,
and the second-stage regularization parameter over
$\{10^{-3},3\cdot10^{-3},10^{-2},3\cdot10^{-2},10^{-1},3\cdot10^{-1},1\}$.
The kernel variance scale was searched over
$\{0.5,0.75,1.0,1.25,1.5,1.75,2.0,2.5,3.0\}$.

For DFPV, the first-stage and second-stage penalties were both searched over
$\{10^{-3},3\cdot10^{-3},10^{-2},3\cdot10^{-2},10^{-1},3\cdot10^{-1},1\}$.
The hidden-layer architecture was selected from
$\{[32],[64],[32,32],[64,64],[128,128]\}$, the learned feature dimension from
$\{4,8,16\}$, the activation function from $\{\mathrm{ReLU},\mathrm{ELU},\tanh\}$, the learning rate from
$\{3\cdot10^{-4},10^{-3},3\cdot10^{-3}\}$, and the weight decay from
$\{0,10^{-5},10^{-4},10^{-3},10^{-2}\}$.

For NMMR-U, we searched the activation function over
$\{\mathrm{ReLU},\mathrm{ELU},\mathrm{SiLU}\}$, whether to use a product kernel over $\{\mathrm{false},\mathrm{true}\}$, the bandwidth scale for $(A,Z,X)$ over
$\{0.3,0.5,0.75,1.0,1.5,2.0,3.0\}$, the hidden-layer architecture over
$\{[32,32],[64,64],[64,64,64]\}$, and the learning rate over
$\{3\cdot10^{-5},10^{-4},3\cdot10^{-4},10^{-3}\}$. NMMR-V used the same search family as NMMR-U, with the only difference being the use of the V-statistic objective rather than the U-statistic objective.



\subsection{Stage 2: Shared calibration of decision-aware weighting}
\label{app:hparam-stage2}

After selecting the unweighted base hyperparameters, we calibrate the decision-aware weighting parameters. The weighted bridge objective uses a local-global weight controlled by a localization strength $\lambda$, a bandwidth $\tau$, and the number of alternating reweighting rounds $n_{rounds}$. These parameters determine how strongly the bridge loss is concentrated near the current pseudo-optimal treatment while retaining global stabilization.

In real observational settings, the target policy regret is not observable, since it depends on the unknown conditional causal response $m_0(a,x)$. A full model-specific regret-based tuning procedure would therefore use information unavailable in practice. To avoid this, we use a conservative shared calibration protocol. We tune $(\lambda,\tau,n_{rounds})$ once on a synthetic calibration setting where ground-truth causal curves are available, using a single base model and a single random seed. The selected weighting parameters are then fixed and reused for all weighted methods and datasets.

The weighting search space is
\begin{equation}
\lambda\sim \mathrm{LogUniform}(0.2,20.0),
\qquad
\tau\sim \mathrm{LogUniform}(0.02,2.0),
\qquad
n_{rounds}\in\{2,3,4,5\}.
\end{equation}
This protocol deliberately prevents each weighted model from receiving separate access to oracle regret tuning. As a result, performance differences among weighted methods are not driven by unequal model-specific optimization over non-observable objectives.

\subsection{Sensitivity Studies}
\label{app:hparam-ablation}

To assess sensitivity to the weighting parameters, we also performed ablation studies in Section \ref{sec:experiments} over the localization strength, bandwidth, number of reweighting rounds, and confounding intensity. These studies show that performance is stable over a nontrivial neighborhood of the selected weighting parameters. This supports the use of a shared weighting protocol: the decision-aware gains are not attributable to a single fragile hyperparameter choice, and the calibrated parameters transfer across model classes and datasets in our experiments.
\FloatBarrier

\section{Additional Results}

\subsection{Global RMSE}
\label{app:additional-results}
Table \ref{tab:counterfactual_rmse_combined} shows the same counterfactual RMSE results as Figure \ref{fig:counterfactual_results} for the non doubly robust methods, but NMMR-U results are better appreciated here. Figure \ref{fig:factual_results} shows factual RMSE for the baseline and DA proximal solvers and for the synthetic and semi-synthetic datasets. The conclusions are similar as with the counterfactual RMSE results.

\begin{table}[t]
\centering
\caption{Counterfactual RMSE across methods and training modes for the synthetic and semi-synthetic settings. 
Values are reported as $\mathrm{mean}_{\mathrm{std}}$ over 10 runs. Lower is better.}
\label{tab:counterfactual_rmse_combined}
\resizebox{\linewidth}{!}{
\begin{tabular}{lcccccccccc}
\toprule
& \multicolumn{5}{c}{Synthetic} & \multicolumn{5}{c}{Semi-synthetic} \\
\cmidrule(lr){2-6} \cmidrule(lr){7-11} 
& DFPV & KPV & NMMR-U & NMMR-V & PMMR
& DFPV & KPV & NMMR-U & NMMR-V & PMMR \\
\midrule

Standard
& $6.59_{0.51}$ 
& $4.25_{0.16}$ 
& $12.16_{22.52}$ 
& $4.38_{0.20}$ 
& $6.85_{0.42}$
& $0.59_{0.13}$ 
& $0.33_{0.01}$ 
& $3.17_{6.01}$ 
& $0.49_{0.04}$ 
& $0.37_{0.03}$ \\

DA
& $5.30_{0.80}$ 
& $3.97_{0.11}$ 
& $32.94_{31.28}$ 
& $4.05_{0.57}$ 
& $6.70_{0.54}$
& $0.61_{0.11}$ 
& $0.32_{0.01}$ 
& $1.01_{0.01}$ 
& $0.55_{0.06}$ 
& $0.40_{0.04}$ \\

\bottomrule
\end{tabular}
}
\end{table}

\begin{figure}[t]
\begin{subfigure}{0.49\linewidth}
        \centering
    \includegraphics[width=\linewidth]{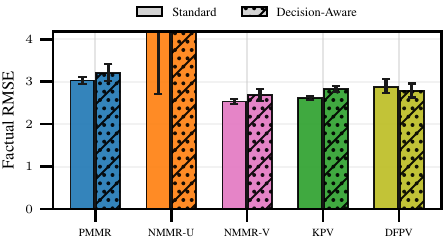}
    \caption{Synthetic experiment}
    \label{fig:semi-synthetic_factual}
\end{subfigure}
\begin{subfigure}{0.49\linewidth}
        \centering
    \includegraphics[width=\linewidth]{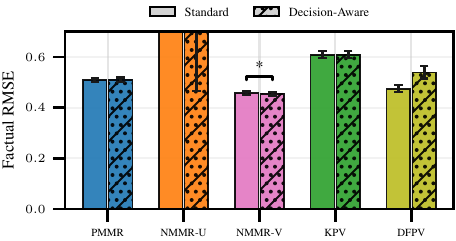}
    \caption{semi-synthetic experiment}
    \label{fig:semi-synthetic_factual}
\end{subfigure}
        \caption{Factual RMSE in the synthetic and semi-synthetic datasets. Mean and 95$\%$ confidence intervals over 10 seeds are shown. * denotes statistical difference between baselines and DA models.}
    \label{fig:factual_results}
\end{figure}

\FloatBarrier

\subsection{Sensitivity to Weakened Completeness}
\label{app:completeness-sensitivity}

\paragraph{Motivation.}
Assumption~\ref{ass:completeness} requires the proxy system to contain sufficient information to distinguish the latent signals relevant to proximal identification. This condition may become weak or fail when the observed proxies are noisy or coarsened measurements of the latent confounders. We therefore conduct a controlled sensitivity analysis that progressively degrades the information carried by both proxy channels while leaving the remaining structural equations of the synthetic DGP unchanged. The purpose is to examine how the proximal bridge estimators and their decision-aware variants deteriorate along the same degradation path, rather than to claim identification once completeness fundamentally fails.

\paragraph{Perturbed proxy mechanism.}
Recall that the proxies in the original synthetic benchmark are generated as
\begin{equation}
    Z
    =
    V_Z+c_PU_Z\mathbf{1}_{d_Z},
    \qquad
    W
    =
    V_W+c_PU_W\mathbf{1}_{d_W},
    \label{eq:completeness-original-proxies}
\end{equation}
where $U_Z$ and $U_W$ are correlated latent confounding components, $V_Z$ and $V_W$ are proxy-specific noise variables, and $c_P$ controls the strength of the latent signal in the proxies.

Let $\alpha\in[0,1]$ denote the degradation level. We define the width of the collapsed central band and the scale of the additional noise as
\begin{equation}
    b_\alpha=\alpha c_{\max},
    \qquad
    \sigma_\alpha=\alpha\sigma_{\max},
    \label{eq:completeness-degradation-scales}
\end{equation}
where $c_{\max}$ and $\sigma_{\max}$ are fixed constants controlling, respectively, the maximum collapse width and the maximum additional-noise scale. For a latent scalar $u$, define the central-band dropout transformation
\begin{equation}
    \mathcal{C}_\alpha(u)
    =
    u\mathbf{1}\{|u|>b_\alpha\}
    =
    \begin{cases}
        0, & |u|\leq b_\alpha,\\
        u, & |u|>b_\alpha.
    \end{cases}
    \label{eq:completeness-central-collapse}
\end{equation}
The perturbed proxies are then generated as
\begin{equation}
\begin{aligned}
    Z_\alpha
    &=
    V_Z
    +c_P\mathcal{C}_\alpha(U_Z)\mathbf{1}_{d_Z}
    +\varepsilon_{Z,\alpha},\\
    W_\alpha
    &=
    V_W
    +c_P\mathcal{C}_\alpha(U_W)\mathbf{1}_{d_W}
    +\varepsilon_{W,\alpha},
\end{aligned}
\label{eq:completeness-perturbed-proxies}
\end{equation}
where
\begin{equation}
    \varepsilon_{Z,\alpha}
    \sim
    \mathcal{N}(0,\sigma_\alpha^2 I_{d_Z}),
    \qquad
    \varepsilon_{W,\alpha}
    \sim
    \mathcal{N}(0,\sigma_\alpha^2 I_{d_W}),
\end{equation}
independently of one another and of the remaining variables.

At $\alpha=0$, we have $b_0=\sigma_0=0$ and $\mathcal{C}_0(u)=u$ almost surely because the latent variables are continuous. Hence,
\begin{equation}
    Z_0
    =
    V_Z+c_PU_Z\mathbf{1}_{d_Z},
    \qquad
    W_0
    =
    V_W+c_PU_W\mathbf{1}_{d_W},
\end{equation}
so the original synthetic DGP is recovered exactly.

For $\alpha>0$, every latent value in the interval $[-b_\alpha,b_\alpha]$ has the same direct contribution to its corresponding proxy. Consequently, distinct latent states within this interval become indistinguishable through the direct proxy signal. The expanding dead zone therefore introduces structured information loss, while the additional Gaussian noise prevents the coarsening transformation from acting as a benign denoising operation. The same symmetric transformation is applied to both latent proxy channels, but their added noise terms remain independent.

The generation of $X$, the joint distribution and correlation structure of $(U_Z,U_W)$, and the structural treatment and outcome equations are otherwise unchanged. In particular, the treatment-response function and the procedure used to calculate the ground-truth causal response remain fixed. Thus, the only modified structural equations are those producing $Z$ and $W$. Because these proxies subsequently enter the treatment and outcome equations, their perturbation can naturally change the resulting observed-data distribution even though the functional forms of those mechanisms are not altered.

This construction directly reduces the information supplied by the proxies and makes the associated proximal inverse problem increasingly difficult. Nevertheless, it should be interpreted as a controlled stress test that moves the DGP away from the completeness conditions. The non-injectivity of the direct latent-to-proxy mapping does not, by itself, formally characterize the relevant conditional-expectation operators for every $\alpha>0$.

\paragraph{Experimental protocol.}
We consider
\begin{equation}
    \alpha\in\{0,0.25,0.50,0.75,0.90\}.
\end{equation}
At every degradation level, we evaluate PMMR, KPV, and DFPV together with their corresponding decision-aware variants. All remaining experimental settings follow the main synthetic benchmark. We report policy regret and global counterfactual RMSE, averaged over five random seeds. The estimates at $\alpha=0$ need not exactly match the main synthetic results because the latter are averaged over ten seeds.

\begin{table*}[t]
\centering
\scriptsize
\setlength{\tabcolsep}{5pt}
\renewcommand{\arraystretch}{1.05}
\caption{Policy regret under progressive degradation of proxy informativeness. Values are means $\pm$ standard deviations over five seeds. Lower is better.}
\label{tab:completeness-sensitivity-regret}
\begin{tabular}{@{}lcccccc@{}}
\toprule
& \multicolumn{2}{c}{PMMR}
& \multicolumn{2}{c}{KPV}
& \multicolumn{2}{c}{DFPV} \\
\cmidrule(lr){2-3}
\cmidrule(lr){4-5}
\cmidrule(lr){6-7}
$\alpha$
& Standard & DA
& Standard & DA
& Standard & DA \\
\midrule
$0.00$
& $5.37\pm2.09$ & $4.73\pm1.87$
& $4.29\pm0.33$ & $3.30\pm0.44$
& $5.83\pm2.49$ & $3.25\pm1.24$ \\
$0.25$
& $5.23\pm2.16$ & $4.56\pm1.85$
& $4.33\pm0.37$ & $3.32\pm0.45$
& $7.41\pm2.07$ & $3.18\pm1.95$ \\
$0.50$
& $6.32\pm1.96$ & $5.65\pm2.04$
& $4.53\pm0.48$ & $3.79\pm0.86$
& $8.83\pm0.06$ & $4.75\pm2.99$ \\
$0.75$
& $6.79\pm1.64$ & $6.19\pm2.16$
& $5.11\pm0.66$ & $4.44\pm1.27$
& $8.26\pm0.83$ & $5.11\pm2.10$ \\
$0.90$
& $7.77\pm0.75$ & $7.46\pm0.95$
& $5.87\pm1.11$ & $4.89\pm1.65$
& $7.89\pm1.03$ & $5.73\pm2.42$ \\
\bottomrule
\end{tabular}
\end{table*}

\begin{table*}[t]
\centering
\scriptsize
\setlength{\tabcolsep}{5pt}
\renewcommand{\arraystretch}{1.05}
\caption{Global counterfactual RMSE under progressive degradation of proxy informativeness. Values are means $\pm$ standard deviations over five seeds. Lower is better.}
\label{tab:completeness-sensitivity-rmse}
\begin{tabular}{@{}lcccccc@{}}
\toprule
& \multicolumn{2}{c}{PMMR}
& \multicolumn{2}{c}{KPV}
& \multicolumn{2}{c}{DFPV} \\
\cmidrule(lr){2-3}
\cmidrule(lr){4-5}
\cmidrule(lr){6-7}
$\alpha$
& Standard & DA
& Standard & DA
& Standard & DA \\
\midrule
$0.00$
& $6.72\pm0.13$ & $6.53\pm0.10$
& $3.98\pm0.13$ & $3.94\pm0.15$
& $6.12\pm0.41$ & $5.49\pm0.37$ \\
$0.25$
& $6.63\pm0.27$ & $6.37\pm0.13$
& $4.15\pm0.10$ & $4.07\pm0.10$
& $6.50\pm0.93$ & $5.11\pm0.84$ \\
$0.50$
& $6.77\pm0.22$ & $6.44\pm0.12$
& $4.28\pm0.17$ & $4.32\pm0.17$
& $6.65\pm0.86$ & $6.43\pm0.44$ \\
$0.75$
& $7.20\pm0.76$ & $6.92\pm0.75$
& $4.43\pm0.07$ & $4.44\pm0.16$
& $6.79\pm1.06$ & $6.26\pm1.06$ \\
$0.90$
& $7.29\pm0.47$ & $7.02\pm0.54$
& $4.50\pm0.15$ & $4.57\pm0.20$
& $6.48\pm0.86$ & $6.06\pm0.47$ \\
\bottomrule
\end{tabular}
\end{table*}

\begin{figure}[ht]
    \centering
    \includegraphics[width=0.85\linewidth]{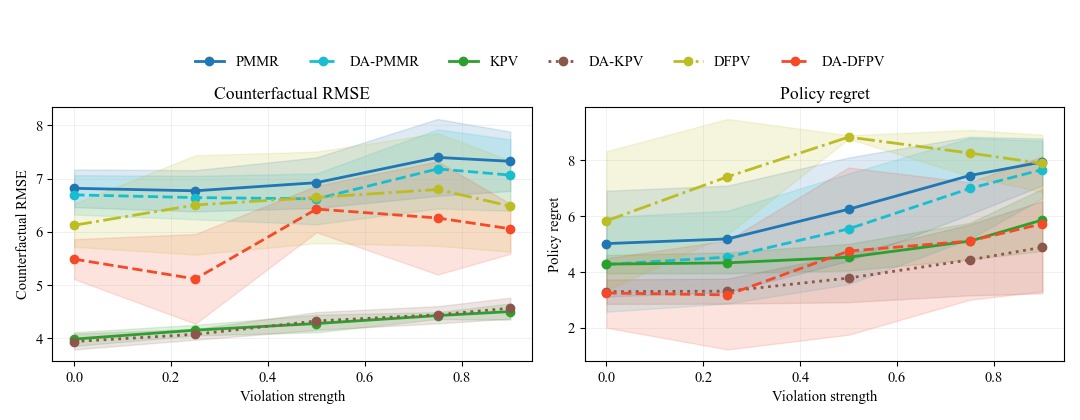}
    \caption{Counterfactual RMSE and policy regret varying the violation strenght factor, $\alpha$. 95 $\%$ CIs over 10 seeds.}
    \label{fig:weakness_sensitivity}
\end{figure}

\paragraph{Results.}
Tables~\ref{tab:completeness-sensitivity-regret} and
\ref{tab:completeness-sensitivity-rmse} show an overall deterioration as the proxy perturbation becomes stronger, although finite-sample variability makes the trajectories non-monotone for some solvers. Trends are better seen in Figure~\ref{fig:weakness_sensitivity}. Comparing the undegraded setting with $\alpha=0.90$, mean regret increases from $5.37$ to $7.77$ for PMMR, from $4.29$ to $5.87$ for KPV, and from $5.83$ to $7.89$ for DFPV. The corresponding decision-aware variants also deteriorate, showing that policy-targeted weighting does not remove the consequences of losing proxy information.

Nevertheless, the decision-aware variant achieves lower mean regret than its unweighted counterpart for all three solvers at every degradation level considered. Under the strongest perturbation, mean regret decreases from $7.77$ to $7.46$ for PMMR, from $5.87$ to $4.89$ for KPV, and from $7.89$ to $5.73$ for DFPV. The advantage becomes relatively small for PMMR under severe degradation, whereas clearer improvements remain for KPV and DFPV.

Global counterfactual RMSE remains broadly comparable between the weighted and unweighted estimators. DA-PMMR and DA-DFPV attain lower mean RMSE at every degradation level. For KPV, the differences are small: DA-KPV performs slightly better at $\alpha\in\{0,0.25\}$ and slightly worse at the three stronger degradation levels. The policy-regret improvements therefore do not appear to be obtained through a large systematic sacrifice in accuracy over the full treatment region.

Overall, these results show that decision-aware weighting can remain useful when the proxies are moderately degraded and that this behavior is not specific to a single proximal bridge solver. At the same time, the deterioration of both the standard and decision-aware estimators emphasizes the limitation of this robustness: decision-aware weighting can prioritize the information that remains in decision-relevant regions, but it cannot reconstruct latent information removed from the proxy system or restore identification after a fundamental failure of completeness.

\FloatBarrier

\begin{figure}
    \centering
    \includegraphics[width=0.85\linewidth]{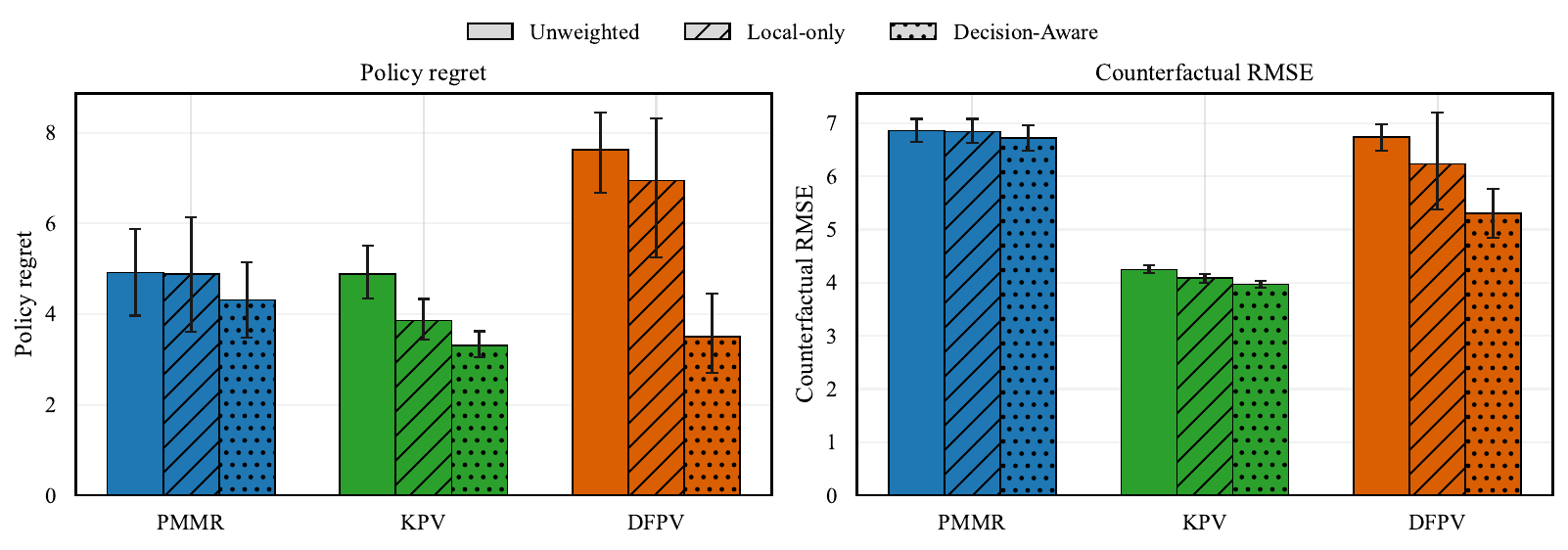}
    \caption{Policy regret and counterfactual RMSE comparing unweighted, pure-local and decision-aware weighting. 95 $\%$ confidence intervals over 5 seeds. Lower is better.}
    \label{fig:local_vs_da}
\end{figure}

\subsection{Pure-Local versus Local--Global Weighting}
\label{app:pure-local-ablation}

\paragraph{Experimental protocol.}
On the synthetic benchmark, we compare unweighted proximal learners, pure-local weighting, and the proposed local-plus-global decision-aware weighting for PMMR, KPV, and DFPV over ten random seeds. The pure-local variant uses the same hyperparameters as the corresponding decision-aware variant, except that the global component, the constant $1$ in the weight numerator, is removed. Tables~\ref{tab:pure-local-regret} and~\ref{tab:pure-local-rmse} and Figure~\ref{fig:local_vs_da} report means with uncertainty intervals in brackets.

\begin{table}[t]
\centering
\scriptsize
\setlength{\tabcolsep}{5pt}
\renewcommand{\arraystretch}{1.05}
\caption{Policy regret for unweighted, pure-local, and local-plus-global decision-aware weighting on the synthetic benchmark over ten seeds. Entries give means with uncertainty intervals in brackets. Lower is better.}
\label{tab:pure-local-regret}
\begin{tabular}{@{}lccc@{}}
\toprule
Learner & Unweighted & Pure-local & Decision-Aware \\
\midrule
PMMR & $4.944\;[3.867, 6.025]$ & $4.855\;[3.339, 6.360]$ & $4.377\;[3.445, 5.288]$ \\
KPV & $4.867\;[4.240, 5.493]$ & $3.845\;[3.292, 4.398]$ & $3.302\;[2.947, 3.657]$ \\
DFPV & $7.524\;[6.475, 8.572]$ & $6.894\;[5.000, 8.787]$ & $3.493\;[2.415, 4.571]$ \\
\bottomrule
\end{tabular}
\end{table}

\begin{table}[t]
\centering
\scriptsize
\setlength{\tabcolsep}{5pt}
\renewcommand{\arraystretch}{1.05}
\caption{Global counterfactual RMSE for unweighted, pure-local, and local-plus-global decision-aware weighting on the synthetic benchmark over ten seeds. Entries give means with uncertainty intervals in brackets. Lower is better.}
\label{tab:pure-local-rmse}
\begin{tabular}{@{}lccc@{}}
\toprule
Learner & Unweighted & Pure-local & Decision-Aware \\
\midrule
PMMR & $6.862\;[6.615, 7.109]$ & $6.842\;[6.561, 7.123]$ & $6.719\;[6.442, 6.997]$ \\
KPV & $4.244\;[4.164, 4.323]$ & $4.093\;[3.993, 4.194]$ & $3.960\;[3.890, 4.048]$ \\
DFPV & $6.738\;[6.455, 7.021]$ & $6.236\;[5.119, 7.353]$ & $5.301\;[4.732, 5.870]$ \\
\bottomrule
\end{tabular}
\end{table}

\paragraph{Results.}
Compared with pure-local weighting, local-plus-global weighting reduces mean regret by approximately $9.8\%$, $14.1\%$, and $49.3\%$ for PMMR, KPV, and DFPV, respectively, and has lower mean global counterfactual RMSE for all three solvers. Pure-local weighting may assign zero weight away from the kernel neighborhood, whereas the global component maintains positive weighting throughout the treatment space. These findings are consistent with the intended global stabilization role.

\FloatBarrier


\end{document}